%% file: main.tex
\pdfoutput=1
\documentclass[11pt]{article} 
 
\usepackage[top=1in, bottom=1in, left=1in, right=1in]{geometry}

\usepackage{times} 
\usepackage{bm} 
\usepackage{color} 
\usepackage{tikz}
\usetikzlibrary{positioning}
\usepackage{adjustbox}
\usepackage{mathtools}
\usepackage{thmtools}
\usepackage{thm-restate}
\mathtoolsset{showonlyrefs}
\usepackage[numbers,sort&compress]{natbib}

\input{old_macros}

\usepackage[scr=esstix]{mathalpha}

\hypersetup{
  colorlinks,
  linkcolor={red!50!black},
  citecolor={blue!50!black},
}

\usepackage{graphicx} 
\usepackage{amsmath,amsfonts,amssymb,amsthm,bbm}
\usepackage{xcolor}
\usepackage[shortlabels]{enumitem}

\numberwithin{equation}{section}

\newcommand{\on}{\operatorname}
\newcommand{\md}{\mathcal{D}}
\newcommand{\me}{\mathcal{E}}
\newcommand{\mq}{\mathcal{Q}}
\newcommand{\cf}{\mathcal{F}}
\newcommand{\cy}{\mathcal{Y}}
\newcommand{\cg}{\mathcal{G}}

\newcommand{\sd}{\mathbb{S}^{d-1}}

\newcommand{\eps}{\epsilon}
\global\long\def\mg#1{\textcolor{purple}{\textbf{[MG comments:} #1\textbf{]}}}
\global\long\def\mgn#1{\textcolor{green}{\textbf{[Note:} #1\textbf{]}}}
\global\long\def\jb#1{\textcolor{teal}{\textbf{[Joan comments:} #1\textbf{]}}}
\global\long\def\dw#1{\textcolor{blue}{\textbf{[DW comments:} #1\textbf{]}}}

\newcommand{\rtmf}{\rho_t^{\textsc{MF}}}
\newcommand{\rsmf}{\rho_s^{\textsc{MF}}}
\newcommand{\rimf}{\rho_{\infty}^{\textsc{MF}}}
\newcommand{\brt}{\bar{\rho}^m_t}

\newcommand{\brz}{\bar{\rho}^m_0}
\newcommand{\rtm}{\hat{\rho}^m_t}
\newcommand{\rsm}{\hat{\rho}^m_s}
\newcommand{\rzm}{\hat{\rho}^m_0}

\newcommand{\bwti}{\bar{w}_t(i)}
\newcommand{\bwtj}{\bar{w}_t(j)}
\newcommand{\bwzi}{\bar{w}_0(i)}

\newcommand{\dit}{\Delta_t(i)}

\newcommand{\djt}{\Delta_t(j)}
\newcommand{\dis}{\Delta_s(i)}
\newcommand{\djs}{\Delta_s(j)}
\newcommand{\diz}{\Delta_0(i)}

\newcommand{\hek}{\text{He}_k}
\newcommand{\hej}{\text{He}_j}
\newcommand{\he}{\text{He}}
\newcommand{\jmax}{J_{\text{max}}}
\newcommand{\javg}{J_{\text{avg}}}

\newcommand{\hpt}{H_t^{\perp}}
\newcommand{\hpi}{H_{\infty}^{\perp}}
\newcommand{\hpit}{H_{\infty}^{\perp}}
\newcommand{\dpt}{D_t^{\perp}}
\newcommand{\cbal}{C_{\text{b}}}

\newcommand{\csim}{C_{\text{SIM}}}
\newcommand{\cstarr}{C_{\rho^*}}

\newcommand{\creg}
{C_{\text{reg}}}
\newcommand{\cdecd}{C_{\ref{lemma:dec2body}}}
\newcommand{\chdec}{C_{\ref{lemma:dec1body}}}
\newcommand{\wed}{\text{BSD}\:}
\newcommand{\cri}{\text{CRI}\:}

\newcommand{\tw}{\tilde{w}}
\newcommand{\wpe}{w_{\perp}}
\newcommand{\er}{\delta}

\newcommand{\xii}{\xi^{\infty}}
\newcommand{\hxiti}{\hat{\xi}_t(w_i)}

\newcommand{\epsmj}{\eps_m^{\ref{lemma:concJ}}}
\newcommand{\epsmh}{\eps_m^{\ref{lemma:conctau}}}
\newcommand{\epsmv}{\eps_m^{\ref{lemma:concentration}}}

\newcommand{\clsc}{C_{\textsc{lsc}}}
\newcommand{\tlsc}{\tau}

\renewcommand{\bwti}{\xi_t(w_i)}
\renewcommand{\bwtj}{\xi_t(w_j)}
\renewcommand{\bwzi}{w_i}

\newcommand{\xorfour}{$\mathsf{XOR}_4$}
\newcommand{\hoponethree}{$\mathsf{Staircase}$}
\newcommand{\mantwo}{$\mathsf{Circle}$}
\newcommand{\mis}{$\mathsf{Misspecified}$}
\newcommand{\hefour}{$\He_4$}
\newcommand{\random}[2]{$\mathsf{Random}_{#1, #2}$}
\newcommand{\hdit}{\hat{\Delta}_t(i)}

\newcommand{\rTmf}{\rho_T^{\textsc{MF}}}

\newtheorem{assumptionalt}{Assumption}[assumption]

\newenvironment{assumptionp}[1]{
  \renewcommand\theassumptionalt{#1}
  \assumptionalt
}{\endassumptionalt}

\renewcommand{\mgn}[1]{}
\renewcommand{\mg}[1]{}
\renewcommand{\jb}[1]{}
\renewcommand{\dw}[1]{}

\title{Propagation of Chaos in One-hidden-layer Neural Networks beyond Logarithmic Time}

\author{
{Margalit Glasgow}$^{1}$\!,\,~ 
{Denny Wu}$^{2,3}$\!,\,~ 
{Joan Bruna}$^{2,3}$\!
\vspace{1.5mm}
\\
\normalsize{$^{1}$Massachusetts Institute of Technology},\,\,\, 
\normalsize{$^{2}$New York University},\,\,\, 
\normalsize{$^{3}$Flatiron Institute}
}

\begin{document}

\maketitle

\input{Body/abstract}

\tableofcontents

\input{Body/intro}

\input{Body/setting}

\input{Body/main_result}

\input{Body/proof_overview}

\input{Body/sims}
\input{Body/conclusion}

\bigskip 

\subsection*{Acknowledgments} 
The authors would like to thank Yuqing Wang, Jamie Simon, Taiji Suzuki, Jason Lee, and Andrea Montanari and anonymous referees for useful discussions during the completion of this work. JB acknowledges funding support from NSF DMS-MoDL 2134216 and NSF CAREER CIF 1845360. This material is based upon MG's work supported by the NSF under award 2402314. This work was done in part while MG and DW were visiting the Simons Institute for the Theory of Computing.

\bigskip

{

\fontsize{11}{11}\selectfont   
\bibliographystyle{alpha}
\bibliography{ref}

}

\newpage

\appendix

\input{Appendices/app_related}
\input{Appendices/apx_dynamics} 
\input{Appendices/apx_conc}
\input{Appendices/apx_pot}

\input{Appendices/examples_proofs}

\input{Appendices/apx_sims}

\end{document}

%% file: old_macros.tex
\usepackage{comment,url,algorithmic,graphicx,subcaption,relsize}
\usepackage{amssymb,amsfonts,amsmath,amsthm,amscd,dsfont,mathrsfs,mathtools,nicefrac}
\usepackage{float,psfrag,epsfig,color,xcolor,url,hyperref}
\usepackage{bbm,epstopdf}
\usepackage[toc,page]{appendix}

\newcommand{\dee}{\mathop{\mathrm{d}\!}}

\newcommand{\dw}{\,\dee w}

\def\ddefloop#1{\ifx\ddefloop#1\else\ddef{#1}\expandafter\ddefloop\fi}

\def\ddef#1{\expandafter\def\csname bb#1\endcsname{\ensuremath{\mathbb{#1}}}}
\ddefloop ABCDEFGHIJKLMNOPQRSTUVWXYZ\ddefloop

\def\ddef#1{\expandafter\def\csname c#1\endcsname{\ensuremath{\mathcal{#1}}}}
\ddefloop ABCDEFGHIJKLMNOPQRSTUVWXYZ\ddefloop

\def\ddef#1{\expandafter\def\csname s#1\endcsname{\ensuremath{\mathscr{#1}}}}
\ddefloop ABCDEFGHIJKLMNOPQRSTUVWXYZ\ddefloop

\def\ddef#1{\expandafter\def\csname v#1\endcsname{\ensuremath{\boldsymbol{#1}}}}
\ddefloop ABCDEFGHIJKLMNOPQRSTUVWXYZabcdefghijklmnopqrstuvwxyz\ddefloop

\def\ddef#1{\expandafter\def\csname v#1\endcsname{\ensuremath{\boldsymbol{\csname #1\endcsname}}}}
\ddefloop {alpha}{beta}{gamma}{delta}{epsilon}{varepsilon}{zeta}{eta}{theta}{vartheta}{iota}{kappa}{lambda}{mu}{nu}{xi}{pi}{varpi}{rho}{varrho}{sigma}{varsigma}{tau}{upsilon}{phi}{varphi}{chi}{psi}{omega}{Gamma}{Delta}{Theta}{Lambda}{Xi}{Pi}{Sigma}{varSigma}{Upsilon}{Phi}{Psi}{Omega}{ell}\ddefloop

\def\balign#1\ealign{\begin{align}#1\end{align}}
\def\baligns#1\ealigns{\begin{align*}#1\end{align*}}
\def\balignat#1\ealign{\begin{alignat}#1\end{alignat}}
\def\balignats#1\ealigns{\begin{alignat*}#1\end{alignat*}}
\def\bitemize#1\eitemize{\begin{itemize}#1\end{itemize}}
\def\benumerate#1\eenumerate{\begin{enumerate}#1\end{enumerate}}

\newenvironment{talign*}
 {\csname align*\endcsname}
 {\endalign}
\newenvironment{talign}
 {\csname align\endcsname}
 {\endalign}

\def\balignst#1\ealignst{\begin{talign*}#1\end{talign*}}
\def\balignt#1\ealignt{\begin{talign}#1\end{talign}}

\let\originalleft\left
\let\originalright\right
\renewcommand{\left}{\mathopen{}\mathclose\bgroup\originalleft}
\renewcommand{\right}{\aftergroup\egroup\originalright}

\def\tinycitep*#1{{\tiny\citep*{#1}}}
\def\tinycitealt*#1{{\tiny\citealt*{#1}}}
\def\tinycite*#1{{\tiny\cite*{#1}}}
\def\smallcitep*#1{{\scriptsize\citep*{#1}}}
\def\smallcitealt*#1{{\scriptsize\citealt*{#1}}}
\def\smallcite*#1{{\scriptsize\cite*{#1}}}

\theoremstyle{plain}  
\newtheorem{remark}{\textbf{Remark}}
\newtheorem*{remark*}{\textbf{Remark}}

\def\R{\mathbb{R}}

\def\<{\left\langle} 
\def\>{\right\rangle}

\DeclareSymbolFont{rsfs}{U}{rsfs}{m}{n}
\DeclareSymbolFontAlphabet{\mathscrsfs}{rsfs}

\newcommand{\He}{\mathsf{He}}

\providecommand{\argmin}{\mathop\mathrm{arg min}}

\def\supp#1{\mathrm{supp}({#1})}

\ifdefined\nonewproofenvironments\else
\ifdefined\ispres\else
\newtheorem{theo}{Theorem}

\newtheorem{fact}[theo]{Fact}

\newtheorem{theorem}{Theorem}
\newtheorem{lemma}[theo]{Lemma}
\newtheorem{corollary}[theo]{Corollary}
\newtheorem{definition}[theo]{Definition}
\newtheorem{proposition}[theo]{Proposition}

\newtheorem*{theorem*}{\textbf{Theorem}}

\renewenvironment{proof}{\noindent\textbf{Proof.}\hspace*{.3em}}{\qed\\}
\newenvironment{proof-sketch}{\noindent\textbf{Proof Sketch}
  \hspace*{1em}}{\qed\bigskip\\}
\newenvironment{proof-idea}{\noindent\textbf{Proof Idea}
  \hspace*{1em}}{\qed\bigskip\\}
\newenvironment{proof-of-lemma}[1][{}]{\noindent\textbf{Proof of Lemma {#1}}
  \hspace*{1em}}{\qed\\}
\newenvironment{proof-of-theorem}[1][{}]{\noindent\textbf{Proof of Theorem {#1}}
  \hspace*{1em}}{\qed\\}
\newenvironment{proof-attempt}{\noindent\textbf{Proof Attempt}
  \hspace*{1em}}{\qed\bigskip\\}
\fi
\newtheorem{observation}[theo]{Observation}
\newtheorem{claim}[theo]{Claim}
\newtheorem{assumption}{Assumption}

\fi
\makeatletter
\@addtoreset{equation}{section}
\makeatother

\hypersetup{
  colorlinks,
  linkcolor={blue!50!black},
  citecolor={blue!50!black},
  urlcolor={blue!80!black}
}



%% file: Body/abstract.tex
\begin{abstract}
We study the approximation gap between the dynamics of a polynomial-width neural network and its infinite-width counterpart, both trained using projected gradient descent in the mean-field scaling regime. We demonstrate how to tightly bound this approximation gap through a differential equation governed by the mean-field dynamics. A key factor influencing the growth of this ODE is the \textit{local Hessian} of each particle, defined as the derivative of the particle's velocity in the mean-field dynamics with respect to its position. We apply our results to the canonical feature learning problem of estimating a well-specified single-index model; we permit the information exponent to be arbitrarily large, leading to convergence times that grow polynomially in the ambient dimension $d$. We show that, due to a certain ``self-concordance'' property in these problems — where the local Hessian of a particle is bounded by a constant times the particle's velocity — polynomially many neurons are sufficient to closely approximate the mean-field dynamics throughout training.
\end{abstract}

%% file: Body/intro.tex
\section{Introduction}

\paragraph{The Mean-field Regime.} We consider the training of the following one-hidden-layer neural network with $m$ neurons via gradient-based optimization:
\begin{align}
    f(x) = \frac{1}{m}\sum_{i=1}^m \sigma(\langle x,w_i\rangle), \quad w_1,w_2,...,w_m\in\sd,
    \label{eq:two-layer-nn}
\end{align}
where $\sigma:\R\to\R$ is the nonlinear activation function (e.g., ReLU), and $\{w_i\}_{i=1}^m$ are trainable parameters, constrained to the sphere. Due to the nonlinearity of the activation function, the optimization landscape is generally non-convex. In this context, two approaches have been developed to ``convexify" the problem through overparameterization (i.e., increasing the network width $m$) and to establish global optimization guarantees: the \textit{neural tangent kernel} (NTK) \cite{jacot2018neural,du2018gradient,allen-zhu2019convergence,zou2020gradient} and the \textit{mean-field} analysis \cite{nitanda2017stochastic,chizat2018global,mei2018mean,rotskoff2018neural,sirignano2020mean}.
The NTK approach linearizes the training dynamics around initialization under appropriate scalings, ensuring that the trainable parameters remain close to their random initialization \cite{chizat2019lazy}. However, this condition prevents feature learning and often leads to suboptimal statistical rates, as it fails to capture the adaptivity of neural networks \cite{ghorbani2019limitations,chizat2020implicit,yang2020feature,ba2022high}.

The mean-field analysis, on the other hand, lifts \eqref{eq:two-layer-nn} into the (infinite-dimensional) space of measures by considering the empirical distribution of neurons $\hat{\rho}^m = m^{-1}\sum_{i=1}^m \delta_{w_i}$. Under certain regularity conditions, one can establish weak convergence of the empirical distribution to the limiting mean-field measure as the number of neurons tends to infinity: $\hat{\rho}^m\overset{m\to\infty}{\to}\rho^{\text{MF}}$, and the trajectory of the limiting parameter distribution is characterized by a partial differential equation (PDE). This (McKean-Vlasov type) PDE description can capture the nonlinear evolution of the neural network beyond the kernel (lazy) regime.

Studying the mean-field dynamics has several advantages, particularly with regard to learning sparse or low-dimensional target functions such as multi-index models. First, in contrast to the NTK regime, the mean-field dynamics describes feature learning which often leads to improved statistical efficiency (see e.g., \cite{bach2017breaking, chizat2020implicit, abbe2022merged, mahankali2023beyond}). Further, overparameterized neural networks are useful for fitting functions that are not well-specified, for instance a multi-index function with an unknown link function. In such instances, prior correlation loss analyses \cite{abbe2023sgd,lee2024neural} that ignore the interaction between neurons cannot establish learnability\footnote{In Section~\ref{sec:assm}, we give several concrete examples of this, along with simulations.}. Second, from a purely analytical perspective, the infinite-width limit allows us to exploit certain problem symmetries that simplify the mean-field PDE into low-dimensional descriptions as done in \citep{abbe2022merged,hajjar2022symmetries,arnaboldi2023high,chen2024mean, montanari2025dynamical}.

\paragraph{Propagation of Chaos.}
Since training infinite-width networks is computationally infeasible, the practical significance of the above theoretical benefits hinges on having a quantitative connection between finite-width networks and their associated mean-field limit. This is precisely the goal we embark upon in this work. The dynamics of polynomial-width neural networks can be viewed as a finite (interacting) particle discretization of the limiting mean-field PDE. Therefore, one of the main challenges in transferring learning guarantees of the infinite-width limit to the finite-width system lies in the non-asymptotic control of particle discretization error --- known as the \textit{propagation of chaos} \cite{sznitman1991topics, chaintron2022propagation}.

In the context of neural network theory, existing propagation of chaos results typically fall short of delivering this non-asymptotic control. On one hand, \textit{exponential-in-time Grönwall-type} estimates leverage the regularity of the dynamics to propagate the Monte-Carlo error at initialization (at scale $O(1/m)$) to obtain an estimate of the form $\sup_{t\in [0,T]} (f_{\rtmf}(x) - f_{\rtm}(x))^2 \lesssim \exp(T)\cdot (m^{-1} \wedge \eta)$ where $\eta>0$ is the learning rate \cite{mei2018mean,mei2019mean,de2020quantitative}. Hence, this type of discretization error analysis is only quantitative when the time horizon is short, such as $T=O_d(1)$ for learning low ``leap'' functions \cite{abbe2022merged,berthier2023learning, montanari2025dynamical} and $T=O_d(\log d)$ for learning certain quartic polynomials \cite{mahankali2023beyond}. 
On the other hand, for the \textit{mean-field Langevin dynamics} (MFLD) \cite{hu2019mean,nitanda2022convex,chizat2022mean}, which introduces additive Gaussian noise to the gradient updates, exponential dependency on time can be removed under a uniform logarithmic Sobolev inequality (LSI), leading to \textit{uniform-in-time propagation of chaos} \cite{chen2022uniform,suzuki2023convergence,kook2024sampling,nitanda2024particle}. However, the LSI assumption ultimately transfers the exponential dependency to the runtime \cite{suzuki2023feature,wang2024mean,mousavi2024learning,takakura2024mean}.
Finally, \cite{chen2020dynamical,pham2021limiting,chizat2022sparse} proved uniform-in-time fluctuations around the mean-field limit, but in the asymptotic width limit. To our knowledge, the only work that coupled a poly-width network with the infinite-width limit for $\text{poly}(d)$ time is \cite{ren2023depth}, which considered a specific bottleneck architecture for learning a symmetric target function.

Consequently, despite the feature learning advantage, the function class that can be learned by two-layer neural networks trained via gradient descent in the mean-field regime with \textit{polynomial compute} is largely unknown, except for target functions reachable within finite (or at most $\log d$) time horizon. It is likely that for many interesting problems, this $T=O_d(\log d)$ horizon is not sufficient for the mean-field dynamics to converge to a low-loss solution. For instance, when the target function is low-dimensional, prior works have shown that gradient-based feature learning often requires $T\gtrsim d^{\Theta(k^*)}$ runtime, where $k^*$ is the \textit{information/leap exponent} (IE) of the link function, which may be arbitrarily large \cite{benarous2021online,abbe2023sgd,bietti2023learning}. 
The goal of this work is to identify sufficient and verifiable conditions under which the mean-field limit is well-approximated by $m=\text{poly}(d)$ neurons up to $T=\text{poly}(d)$ time horizon.

\subsection{Our Contributions}  

In this work, we study a teacher-student setting where the target function is parameterized by finitely many ``teacher'' neurons. Let $\rtmf$ denote the distribution at time $t$ of the infinite-width mean-field dynamics trained with projected (spherical) gradient flow on infinite data, and $\rtm$ the $m$-particle mean-field discretization of this dynamics, trained with $n$ samples. We establish a set of conditions under which $\rtm$ is well approximated by $\rtmf$ up to the time required to learn the teacher model.  
The crux of these conditions is twofold:
\begin{enumerate}[leftmargin=*,itemsep=0.5mm] 
    \item The mean-field dynamics satisfy a certain \em local strong convexity \em (Assumption~\ref{def:lsc}),  which states that when a neuron is close to a teacher neuron, the local landscape is strongly convex.
    \item A certain average \emph{stability} parameter $\javg$  (Assumption~\ref{assm:growth}) is at most $O(1/T)$, where $T$ is the convergence time. Loosely speaking, $\javg$ is a measure of the average sensitivity of the neurons with respect to a small perturbation in any one neuron.
\end{enumerate} 

Denote $f_\rho(x) := \mathbb{E}_{\rho}[\sigma(\langle x, w \rangle)]$,  $ \| f \|:= \mathbb{E}_{x} [|f(x)|^2]^{1/2}$ and $\mathcal{E}( \rho, \tilde{\rho}) = \| f_\rho - f_{\tilde{\rho}} \|$. 
We show in Theorem~\ref{thm:main} that if the above conditions hold (along with several other regularity and technical conditions), then for $t \leq T$, with high probability one has 
\begin{align}
    \mathcal{E}(\rtmf, \rtm) \lessapprox \frac{\on{poly}(d, t)}{\min\left(\sqrt{m}, \sqrt{n}\right)}.
\end{align}
This means that $\text{poly}(d,T)$ neurons suffice to approximate the mean-field limit up to the time of convergence. This result also gives a non-asymptotic rate of convergence of $\rtm$ to an appropriate empirical measure of $\rtmf$ with time dependence that goes beyond the pessimistic Grönwall estimate.



In Theorem~\ref{theorem:SIM}, we apply our result to a setting of learning a single-index model (SIM) with high information exponent $k^* \geq 4$, for which gradient flow converges in time $T = \Theta(d^{k^*/2})$. First, we prove that in this setting, the limiting mean-field network, trained on the population loss, can learn the target function at time $T$. Then we use Theorem~\ref{thm:main} to deduce that with $m, n = d^{\Theta(k^*)}$, at time $T$ the distance $\mathcal{E}(\rtmf, \rtm)$ is small, and thus the finite-width model $\rtm$ also achieves small population loss. 

\begin{remark*}
To our knowledge, our work is the first to prove propagation of chaos (i.e., the above bound on $\mathcal{E}(\rtmf, \rtm)$) with polynomially many neurons at timescales longer than $\log(d)$. We remark that we do not believe all the conditions we impose to be necessary -- we discuss this in detail in Section~\ref{sec:assm}. 
Existing techniques (see \cite{chaintron2022propagation} for review) primarily leverage either (a) convexity in the system, (b) Grönwall's method, or (c) a large diffusion term. Our techniques go beyond these approaches, and as such they could be useful to establish quantitative propagation of chaos in interacting particle systems with little or no noise.
\end{remark*}



\paragraph{Outline.} 
In Section~\ref{sec:setting}, we provide preliminaries on the setting and explain the basic objects we will analyze. In Section~\ref{sec:result}, we state our main results, as outlined in the contributions. In Section~\ref{sec:pfoverview}, we give an overview of the proofs. In Sections~\ref{sec:assm} and \ref{sec:sims}, we discuss the assumptions of our settings, comment on their necessity, and provide simulations. We conclude in Section~\ref{sec:conclusion}. Full proofs are given in the Appendix.

\paragraph{Notations.} $\mathcal{P}(\Omega)$ denotes the space of probability distributions over $\Omega$. $W_1(\rho, \rho')$ denotes the 1-Wasserstein distance between distributions $\rho$ and $\rho'$.
We will use lower-case letters ($f, g, h$) to denote functions defined on $\sd$, Greek letters ($\Delta$, $\xi$, etc) to denote vector-valued functions $\sd \rightarrow \mathbb{R}^d$, and upper-case letters to denote matrix-valued functions $\sd \rightarrow \mathbb{R}^{d \times d}$ or $\sd \times \sd \rightarrow \mathbb{R}^{d \times d}$. 
When $\hat{\mu}$ is an empirical measure of the form $\hat{\mu}=\frac1m \sum_{i} \delta_{w_i}$, we will use the shorthand $f(i) = f(w_i)$, and denote $\mathbb{E}_i f(i) := \frac1m \sum_i f(w_i)$. We write $P^{\perp}_w := (I - ww^\top )$. For $H \in L^2(\sd \times \sd, \mu^2, \mathbb{R}^{d \times d})$, $D \in L^2(\sd, \mu, \mathbb{R}^{d \times d})$ and $\Lambda, \in L^2(\sd, \mu, \mathbb{R}^{d})$, we use $H\Lambda (w) := \mathbb{E}_{w' \sim \mu} H(w, w') \Lambda(w')$, and $D \odot \Lambda (w) = D(w)\Lambda(w)$. For $f \in L^2( \R^d, \nu)$, we write $\| f \|_\nu^2:= \mathbb{E}_x |f(x)|^2$, and omit the subscript when the context is clear. 

Throughout this paper, we use the asymptotic notation $O_{C}(X)$ to denote $X$ times some constant that depends arbitrarily on $C$. Whenever a term of the form $C$ (usually with some subscript) appears, this term is referring to a constant, meaning that its value does not depend on $m, n, d$ (which we will take to infinity). We write ``with high probability'' when the probability approaches $1$ as $m$ or $n$ goes to infinity. This probability is taken over the neural network initialization $\{w_i\}_{i \in [m]}$ and the random sample of $n$ data points.

%% file: Body/setting.tex
\section{Setting and Preliminaries}\label{sec:setting}
\subsection{Projected Gradient Dynamics on Neural Networks}
Consider a neural network to be parameterized by some distribution $\rho \in \mathcal{P}(\sd)$, such that 
\begin{align*}
    f_{\rho}(x) := \mathbb{E}_{w \sim \rho} \sigma(w^\top x),
\end{align*}
for a link function (activation) $\sigma$. We require $\sigma$ to satisfy the regularity conditions in Assumption~\ref{assm:reg}.

A supervised regression problem is parameterized by an initial distribution for the network weights, $\rho_0$, and a distribution $\md$ over points $(x, y) \in \mathbb{R}^{d} \times \mathbb{R}$. Given $(\rho_0, \md)$, we define $f^*(x) = \mathbb{E}_{\md}[y | x]$.
We will train the neural network to minimize the squared loss 
\begin{align}
    L_{\md}(\rho) := \mathbb{E}_{(x, y) \sim \md} (f_{\rho}(x) - y)^2~.
\end{align}



We study the projected gradient flow dynamics of $\rho$ induced by moving each particle $w \sim \rho$ in the direction of the gradient of the loss $L_{\md}(\rho)$, and then projecting the particle back on the sphere:
\begin{align}\label{eq:Vexpansion}\textstyle
        \frac{d}{dt} w = \nu_{\md}(w, \rho) := -(I - ww^\top )\nabla_w \mathscr{f}_{\md}(w) + (I - ww^\top )\nabla_w \mathbb{E}_{w' \sim \rho} \mathscr{k}_{\md}(w, w')
\end{align}
where 
\begin{align}\label{eq:FKdefbody}
    \mathscr{f}_{\md}(w) := \mathbb{E}_{(x, y)\sim \md} y\sigma(w^\top x) \qquad \text{and} \qquad \mathscr{k}_{\md}(w, w') := \mathbb{E}_{(x, y) \sim \md} \sigma(w'^\top x)\sigma(w^\top x). 
\end{align}
In the case where we train on infinite data, the relevant problem parameters are $(f^*, \rho_0, \md_{x})$, where $\md_x$ is the $x$-marginal of $\md$. In such setting, and when $\md_x$ is clear from context, we will use $\nu(w, \rho)$ (without any distribution subscripted) to denote the case where $x \sim \md_x$, $y = f^*(x)$ deterministically. Whenever an expectation over $x$ appears in this paper without explicit distribution, it should be interpreted as over $x \sim \md_x$. In this paper, we will primarily be interested in a teacher-student setting with a ground truth measure $\rho^*$, such that $f^*(x) = \mathbb{E}_{w^* \sim \rho^*}\sigma(x^\top w^*)$. Thus we will sometimes describe a problem by $(\rho^*, \rho_0, \md_x)$. 

\subsection{Coupling between Mean Field and Finite-Neuron Dynamics}

We will study the evolution of two different learning dynamics in this paper.
\paragraph{Infinite-width, infinite-data \em mean-field \em dynamics.}
We denote the mean-field distribution at time $t$ by $\rtmf \in \mathcal{P}(\sd)$, where we initialize $\rho_0^{\textsc{MF}} = \rho_0$. Each particle $w \in \sd$ in the mean-field dynamics evolves according to the infinite-data velocity $\nu(w, \rtmf) \in T_w \sd$. $\xi_t(w) \in \sd$ denotes the \emph{characteristic} of a particle initialized at $w$ and evolved under the mean-field dynamics:
\begin{align*}\textstyle
    \frac{d}{dt} \xi_t(w) = \nu(\xi_t(w), \rtmf) \qquad \xi_0(w_i) = w_i~.
\end{align*}
This dynamics can also be expressed though the \em continuity equation:\em
$
    \frac{d}{dt} \rtmf = \nabla \cdot (\nu(w, \rtmf)\rtmf).
$

\paragraph{Finite-width, finite-data dynamics.}
Let $\rtm$ denote the empirical measure defined by \em $m$ neurons \em under the projected gradient flow induced by the \em empirical loss \em from $n$ training samples. Let $\hat{\md}$ denote the empirical distribution of the $n$ training samples. We initialize $\rzm = \frac{1}{m}\sum_{i = 1}^m \delta_{w_i}$, where $w_i \sim \rho_0$ i.i.d. for each $i \in [m]$. Each particle $w \in \sd$ in the finite dynamics evolves according to the empirical velocity $\nu_{\hat{\md}}(w, \rtm)$. This defines an ODE in $(\sd)^{\otimes m}$, whose characteristics are now denoted by $\hxiti$, and solve  
\begin{align*}\textstyle
    \frac{d}{dt} \hxiti = \nu_{\hat{\md}}(\hxiti, \rtm) \qquad \hat{\xi}_0(w_i) = w_i~,~i\in [m]~.
\end{align*}
We will study the setting where the training data are drawn i.i.d.~from a sub-Gaussian distribution with sub-Gaussian label noise (See Assumption~\ref{assm:reg}, \ref{assm:data}). 

\paragraph{Coupling the dynamics.}
Let $\brt$ be the distribution initialized at $\rzm$, but that evolves according to the dynamics $\nu(\cdot, \rtmf)$. That is, $\brt = \frac1m \sum_{i = 1}^m \delta_{\xi_t(w_i)}$. Note that $\brt$ is equivalent in distribution to a random sample of $m$ particles drawn iid from $\rtmf$. Define the coupling error at neuron $w_i$ as
 \begin{align}
    \dit :=  \hxiti - \bwti \in \mathbb{R}^d~,~i\in [m]~,
\end{align}
such that $\diz = 0$ for all $i$. Now by definition, $W_1(\rtm, \brt) \leq \mathbb{E}_i \|\dit\|$; thus it is easy to show that $\mathbb{E}_i \|\dit\|$ gives a good bound on the function-error distance between $\rtmf$ and $\rtm$:
\begin{restatable}{lemma}{factw}\label{fact} Suppose Assumption~\ref{assm:reg} holds. With high probability over the draw $\rho_0^m$, we have
        $$\| f_{\rtmf} - f_{\rtm}\|^2 \leq O_{\creg} \left(\mathbb{E}_i \|\dit\|\right)^2 + \frac{\log(m)}{m}.$$
\end{restatable}


\subsection{Description of the Dynamics of $\Delta$}\label{sec:dyn_desc}
The main result of this section is Lemma~\ref{lemma:errdynamicsbody}, which gives a first-order approximation of the dynamics of $\dit$. The quantities $\{\dit\}_i$ evolve via their own particle interaction system, governed by two main terms: a self-interaction term, and an interaction term. 
The self-interaction term is described by what we call the \em local Hessian, \em the derivative of a particle's velocity with respect to that particle's position.
\begin{definition}[Local Hessian]
The \em local Hessian \em $D^{\perp}_t: \sd \to \R^{d \times d}$ of neuron $w$ at time $t$ is
    \begin{align}
    D^{\perp}_t(w) := \left(\nabla_{\xi_t(w)} \nu(\xi_t(w), \rtmf)\right)(I - \xi_t(w)\xi_t(w)^\top).
\end{align}
We will also use the abbreviated notation $ D^{\perp}_t(i) :=  D^{\perp}_t(\bwzi)$.
\end{definition}
\begin{remark}\label{rem:ll}
We call this the local Hessian because it equals the \em negative \em Hessian of the landscape of the map $\bwti \rightarrow  U_t(\bwti) := U(\bwti; \rtmf)$, where $U = \frac{\delta L}{\delta \rho}$ is the first-variation of the loss, so that $V = \nabla U$, 
and $\bwti$ is restricted to the manifold $\sd$. Thus if the local landscape $U_t(\bwti)$ is convex on $\sd$, then $\dpt(i)$ is negative semi-definite. 
\end{remark}

The part of the dynamics driven by the other $\djt$ is described by what we term the \em interaction Hessian, \em the (rescaled) derivative of a particle's velocity with respect to the other particles' position. 
\begin{definition}[Interaction Hessian]
Define the \em interaction Hessian \em $H^{\perp}_t: \mathbb{S}^{d-1} \times \mathbb{S}^{d-1}  \rightarrow \mathbb{R}^{d \times d}$ by
\begin{align}
    H^{\perp}_t(w, w') := \left(I - \xi_t(w)\xi_t(w)^\top\right)\nabla_{\xi_t(w')} \nabla_{\xi_t(w)} \mathscr{k}(\xi_t(w), 
    \xi_t(w'))\left(I - \xi_t(w')\xi_t(w')^\top\right),
\end{align}
We will also use the abbreviated notation $ H^{\perp}_t(i, j) :=  H^{\perp}_t(\bwzi, w_j)$.
\end{definition}
\begin{fact}\label{fact:PSDbody}
For any $w, w'$, $H^{\perp}_t(w,w')$ is a positive semi-definite kernel.
\end{fact}
\begin{proof}
By definition of $\mathscr{k}_{\md}$ in Equation~\ref{eq:FKdefbody}, one can check that $H^{\perp}_t(w, w') = \mathbb{E}_x\phi_x(w)\phi_x(w')^\top$, where 
we define the feature map $\phi_x(w) := (I - \xi_t(w)\xi_t(w)^\top)\sigma'(\xi_t(w)^\top x)x$
\end{proof}

We make the following basic regularity assumptions on the activation function and the data.
\begin{assumptionp}{Regularity}[Regularity Assumptions]\label{assm:reg} ~
\begin{enumerate}[{\bfseries{R\arabic{enumi}}}]
    \item\label{assm:sigma} For a constant $\creg$, the activation $\sigma$ satisfies that for $j = 0, 1, 2, 3$ and any subGaussian variable $X$, we have $\mathbb{E}_{X}|\sigma^{(j)}(X)|^5 \leq (\creg/11)^5$, where $\sigma^{(j)}$ denotes the $j$th derivative of $\sigma$.
\item \label{assm:data} The distribution $\md_x$ on the covariates is $\creg$-subGaussian, and the noise has covariance at most $1$, that is $\mathbb{E}_{y \sim \md | x}(y - f^*(x))^2 \leq 1$.
\end{enumerate}

\end{assumptionp}

We introduce the control parameters 
\begin{align*}
    \eps_m := \frac{d^{3/2}\log(mT)}{\sqrt{m}}, &\qquad \eps_n := \frac{\sqrt{d}\log^2(n)}{\sqrt{n}}. 
\end{align*}
We will show in Lemma~\ref{lemma:concentration} that with high probability, the error $\|\nu(\bwti, \rtmf) - \nu(\bwti, \brt)\|$ due to sampling only $m$ neurons is uniformly (over $i$ and $t$) bounded by $\eps_m$. Similarly, we will show in Lemma~\ref{lemma:nconcentration} that the error $\|\nu_{\hat{\md}}(\hxiti, \rtm) - \nu(\hxiti, \rtm)\|$ due to using the empirical data distribution $\md$ is uniformly bounded by $\eps_n$.
\begin{restatable}[Parameter-Space Error Dynamics]{lemma}{errdyn}\label{lemma:errdynamicsbody}
Suppose Assumption~\ref{assm:reg} holds. With high probability, for all $t\leq T$ and $i \in [m]$,
$$
\frac{d}{dt}\dit = D^{\perp}_t(i) \dit - \mathbb{E}_{j \sim [m]}H^{\perp}_t(i, j) \djt + {\bm{\epsilon}_{t, i}},$$
where 
$\|{\bm{\eps}_{t, i}}\| \leq 2\eps_m + \eps_n + 2\creg\left(\|\dit\|^2 + \mathbb{E}_j\|\djt\|^2\right).$
\end{restatable}
We prove Lemma~\ref{lemma:errdynamicsbody} by decomposing $\frac{d}{dt}\dit \!=\! \nu(\bwti, \rtmf) \!-\! \nu(\hxiti, \rtm)$ into four differences (see Figure~\ref{fig:dynamics}), and separating the first order terms (in $\Delta_t$) from higher order terms in these differences.

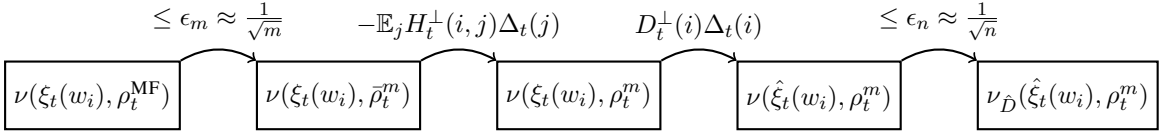
\begin{figure}[t]
\vspace{-2mm}
    \centering
    \input{plots/fig_dynamics}
    \caption{\small Decomposing $\frac{d}{dt}\dit = \nu(\bwti, \rtmf) - \nu_{\hat{\md}}(\hxiti, \rtm)$. The approximate differences between the terms in the rectangles are given above the arrows. }
    \label{fig:dynamics}
\end{figure}

\paragraph{An integral form for $\dit$.}

Duhamel's principle gives us a way to solve the ODE in Lemma~\ref{lemma:errdynamicsbody} using the solution to a simpler dynamics which only involves the local Hessian.
\begin{definition}[Local Stability Matrix]
Define $J^{\perp}_{t, s}(w)$ to be the matrix that solves
\begin{align}\textstyle
\frac{d}{dt} J^{\perp}_{t, s}(w) = D^{\perp}_t(w)J^{\perp}_{t, s}(w); \qquad J^{\perp}_{s, s}(w) = (I - \xi_s(w)\xi_s(w)^\top).
\end{align}
We call this the local stability matrix, because $J^{\perp}_{t, s}(w) = \mathbf{J}_{\xi_{t, s}}(\xi_s(w))$, where $\xi_{t, s}(u)$ denotes the position of a neuron at time $t$ which evolves in the mean field dynamics starting at position $u$ at time $s$, and $\mathbf{J}$ denotes the Jacobian. We use the shorthand $J_{t, s}(i) := J_{t, s}(w_i)$.

\end{definition}
On the same assumptions as Lemma~\ref{lemma:errdynamicsbody}, 
Duhamel's principle yields
 \begin{align}\label{eq:duhamel}
    \dit &= \int_{0}^t J^{\perp}_{t, s}(i)\left(-\mathbb{E}_j H^{\perp}_s(i, j) \djs + \bm{\eps_{s, i}}\right)ds.
\end{align}

%% file: plots/fig_dynamics.tex
\begin{tikzpicture}[scale=0.9, every node/.append style={scale=0.9}]
    \node[draw, rectangle, minimum width=2cm, minimum height=1cm, thick] (rect1) {$\nu(\bwti, \rtmf)$};
    \node[draw, rectangle, minimum width=2cm, minimum height=1cm, right=1cm of rect1, thick] (rect2) {$\nu(\bwti, \brt)$};
    \node[draw, rectangle, minimum width=2cm, minimum height=1cm, right=1cm of rect2, thick] (rect3) {$\nu(\bwti, \rtm)$};
    \node[draw, rectangle, minimum width=2cm, minimum height=1cm, right=1cm of rect3, thick] (rect4) {$\nu(\hxiti, \rtm)$};
    \node[draw, rectangle, minimum width=2cm, minimum height=1cm, right=1cm of rect4, thick] (rect5) {$\nu_{\hat{D}}(\hxiti, \rtm)$};

    \draw[->, thick] (rect1.north east) to[out=30, in=150] node[midway, above] {$\leq \eps_m \approx \frac{1}{\sqrt{m}}$} (rect2.north west);
    \draw[->, thick] (rect2.north east) to[out=30, in=150] node[midway, above] {$-\mathbb{E}_{j}H^{\perp}_t(i, j)\djt$} (rect3.north west);
    \draw[->, thick] (rect3.north east) to[out=30, in=150] node[midway, above] {$D^{\perp}_t(i)\dit$} (rect4.north west);
    \draw[->, thick] (rect4.north east) to[out=30, in=150] node[midway, above] {$\leq \eps_n \approx \frac{1}{\sqrt{n}}$} (rect5.north west);
\end{tikzpicture}

%% file: Body/main_result.tex
\section{Main Result: Propagation of Chaos}\label{sec:result}

\subsection{Intuition and Key Challenges}\label{sec:intuition}
To bound $\mathcal{E}(\rtm, \rtmf)$, it suffices to analyze the dynamics of $\Delta_t$ given by the ODE in Lemma~\ref{lemma:errdynamicsbody}:
\begin{align}\label{eq:dyn}\textstyle
    \frac{d}{dt}\dit &= D^{\perp}_t(i) \dit - \mathbb{E}_{j \sim [m]}H^{\perp}_t(i, j) \djt + {\bm{\epsilon}_{t, i}} \qquad \|{\bm{\epsilon}_{t, i}}\| \leq \eps.
\end{align}
One might hope to leverage the linearity of \eqref{eq:dyn} to solve this ODE in closed form, but unfortunately, the time-dependent coefficient matrix, $\text{diag}(\dpt) - \hpt$, does not commute at different times $t$. 
 \paragraph{Going Beyond Grönwall.}
The conventional approach (see e.g., \cite{mei2018mean,mei2019mean}), uses the maximum Lipschitzness of $\nu(w, \rho)$
 -- in our spherical case, this translates to a bound on $\sup_{i, j, t} \|\dpt\|,  \|\hpt(i, j)\|$ -- to bound the RHS of \eqref{eq:dyn} as
\begin{align}\label{eq:gronwall}
    \frac{d}{dt} \|\dit\| \leq 2\on{Lip}_{\on{max}}\sup_{j \in [m]} \|\djt\| + \epsilon.
\end{align}
In standard settings, this maximum Lipschitzness is a constant, so this method can achieve no better than the bound $W_1(\rho_t^m, \brt) \leq \exp(\Theta(t))\eps$. The work of \cite{mahankali2023beyond} goes further to bound \eqref{eq:gronwall} using a tight time-dependent Lipschitz constant, yielding propagation of chaos for $\log(d)$ time. 
However, for problems with polynomial-in-$d$ time to converge, such as learning a SIM with a high information exponent, the approach in \eqref{eq:gronwall} is overly pessimistic, because both the local Lipschitzness at neuron $i$, and the $\|\djt\|$ are extremely non-uniform in $i$ and $j$ (See Figure~\ref{fig:nonuniform}).

Equation~\eqref{eq:duhamel} gives us an alternative way to approach \eqref{eq:dyn} which can leverage the non-uniform Lipschitzness. Ignoring for a moment the interaction terms in Equation~\eqref{eq:duhamel}, we have $\|\dit\| \approx \int_{0}^t J^{\perp}_{t, s}(i)\bm{\eps_{s, i}}ds$, where we recall that the perturbation matrix $J^{\perp}_{t, s}(i)$ measures of the stability of $\xi_t(w)$ with respect to perturbations at time $s$. Naively, $J^{\perp}_{t, s}(w_i)$ appears to grow at an exponential rate whenever the local landscape of the linearized loss around $\bwti$ (see Remark~\ref{rem:ll}) is non-convex.

A key observation of our work is that when $w_i$ escapes certain higher-order saddles, $\|J^{\perp}_{t, s}(i)\|$ will be bounded polynomially in $t - s$. We achieve this by showing a certain \textit{self-concordance}-like property which upper bounds $\dpt(i)$ using the velocity (which is small near the saddle).
Thus one part of our assumptions will be a worst-case polynomial bound on $\|J^{\perp}_{t, s}(w)\|$ (see Assumption~\ref{assm:growth}).

\begin{figure}[t]
\vspace{-2.5mm}
    \begin{minipage}{0.33\textwidth}
        \centering
        \includegraphics[width=0.95\textwidth]{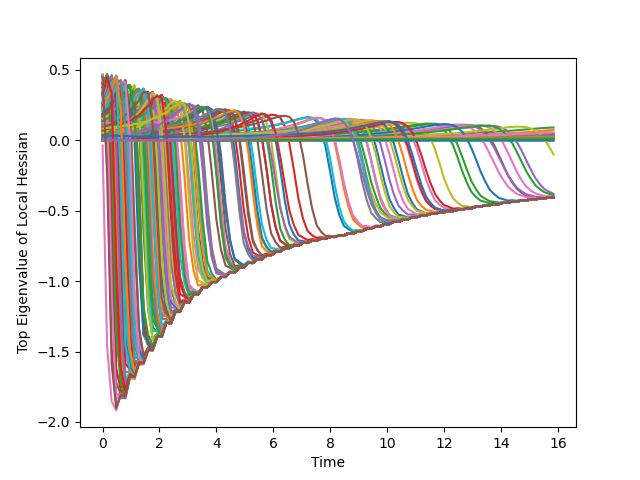}
    \end{minipage}
    \begin{minipage}{0.33\textwidth}
        \centering
  \includegraphics[width=0.95\textwidth]{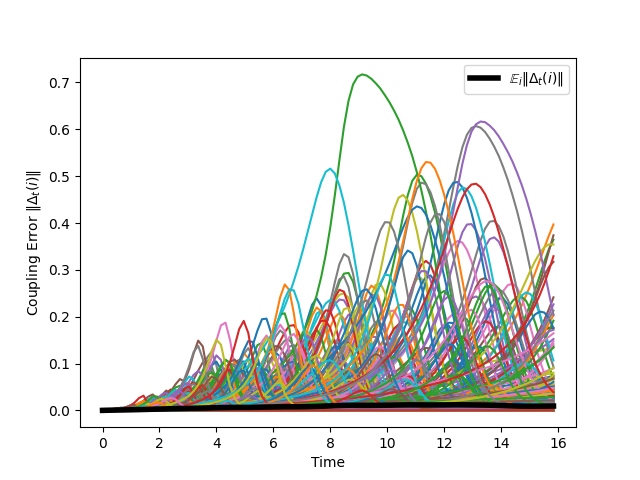}
    \end{minipage} 
    \begin{minipage}{0.33\textwidth}
        \centering
        \includegraphics[width=0.95\textwidth]{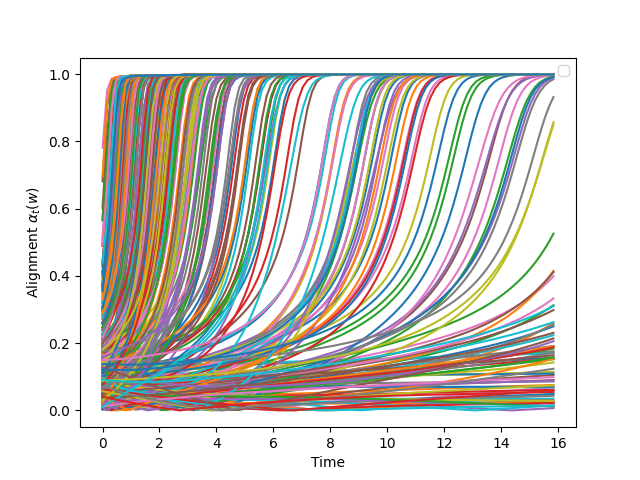}
    \end{minipage}
   \vspace{-5.5mm}
    \caption{\small Non-Uniform Dynamics in SIM with IE $4$ $(f^*(x) = \he_4(x^\top w^*)$ for $x \in \mathbb{R}^{32}$. We plot $\dpt(i), \|\dit\|, \alpha_t(w_i) = |w^*\xi_t(w_i)|$ for each neuron. Left: Top eigenvalue of the local hessians $\dpt(i)$. Center: $\|\dit\|$, Right: Alignment $\alpha_t(w_i)$ with the teacher neuron. A key challenge in the IE $> 2$ setting is the variance in Lipschitzness among the different neurons, and in $\|\dit\|$.}\label{fig:nonuniform}
   \vspace{-5mm}
\end{figure}
\vspace{-0.15cm}
\paragraph{The Interaction Term: A Blessing and a Curse.}\label{para:blessing}
At first glance, the presence of the PSD interaction term $\hpt$ in \eqref{eq:dyn} seems like it can only help us bound $\mathbb{E}_i\|\dit\|$. Indeed, if we ignore the local $\dpt$ terms in the ODE, we would have that $\frac{d}{dt}\Delta_t = - \hpt \Delta_t$, and thus we could show that $\mathbb{E}_i \|\dit\|^2$, an upper bound on the Wasserstein-2 distance $W_2(\rtm, \rtmf)$, is non-increasing.

However, the interaction of $\hpt$ and $\dpt$ can lead to precarious situations if the neurons move at non-uniform rates. To see this possibility, suppose for some neuron $w_i$, $\dit$ first grows by a polynomial factor due to $\dpt(i)$, and then propagates that error, via the interaction term, to a different neuron $w_j$. Later on, when neuron $w_j$ escapes the saddle, it will grow $\djt$ by a polynomial factor. The process can then continue by ``passing off'' the error between neurons such that it grows in an exponential fashion, without any neuron doing more than ``polynomial growth'' of the error itself. 

To rule out such a scenario, we will impose an assumption that leverages the intuition that in many teacher-student settings with uniform initialization, the neurons are dispersed before converging to the teacher neurons. Thus on average, the interaction term -- whose scale is dictated by inner product $w_i^\top w_j$ -- is small, and cannot propagate too much error to these neurons. Specifically, the interaction term drives changes in the error according to the interaction Hessian, $\hpt$: an error of $\djt$ at neuron $w_j$ causes a force of $-\hpt(i, j)\djt$ on the error of neuron $w_i$. Following Equation~\eqref{eq:duhamel}, this force propagates into an error of scale $R_{t, s}(i, j)\djs$ on neuron $w_i$ at time $t$, where $R_{t, s}(i, j) := J^{\perp}_{t, s}(i)H^{\perp}_s(i, j)$.

The second part of Assumption~\ref{assm:growth}  states that the \em average \em of $R_{t, s}(i, j)$, over all neurons $i$ far from $\supp{\rho^*}$, is small. 
\paragraph{Behavior Near the Teacher Neurons.}
While the second part of Assumption~\ref{assm:growth} is quite powerful, we cannot hope that it holds for neurons near the teacher neurons. Indeed, when $i$ and $j$ are both near some $w^* \in \supp{\rho^*}$, then $\|R_{t, t}(i, j)\| = \|\hpt(i, j)\| = \Omega(1)$. Thus for neurons near $\supp{\rho^*}$, we will need to leverage the fact that $\hpt$ is PSD. A key contribution of our work is constructing a novel potential function which can leverage this term. We discuss this at length in Section~\ref{sec:pfoverview}.

\vspace{-0.3cm}
\subsection{Theorem Statement}
We will now present an informal version of our assumptions and propagation of chaos result. Due to the technicality of some of the assumptions, we defer some full statements to Section~\ref{sec:assm}. Define
\begin{align}
B_{\tau} := \{w \in \mathbb{S}^{d-1}: \exists w^* \in \text{supp}(\rho^*): \|w^* - w\|\leq \tau \}.
\end{align}
The following key assumption gives average and worst-case bounds on some of the stability parameters of the MF dynamics.
\begin{assumptionp}{Stability}[Worst-Case and Average Stability]\label{assm:growth}
    Suppose that we have
    \begin{align}
        \jmax := \sup_{s \leq t \leq T, w \in \sd}\left(\|J^{\perp}_{t, s}(w)\|, \mathbb{E}_{w \sim \rho_0} \|J^{\perp}_{t, s}(w)\|^2\right) \leq \on{poly}(d, T).\tag{J1} \label{eq:JMAX}
    \end{align}
    \mgn{\begin{align}
        \jmax(t) := \sup_{s \leq t, w \in \sd}\left(\|J^{\perp}_{t, s}(w)\|, \mathbb{E}_{w \sim \rho_0} \|J^{\perp}_{t, s}(w)\|^2\right) \leq \on{poly}(d, t). 
    \end{align}}
    Further suppose that for all $\tau > 0$, and given a target horizon $T>0$,
    \begin{align}
        \javg(\tau) := \sup_{s \leq t \leq T, w', v \in \mathbb{S}^{d-1}}\mathbb{E}_{w \sim \rho_0}\|J^{\perp}_{t, s}(w)H_s^{\perp}(w, w')v\|\mathbf{1}(\xi_t(w) \notin B_{\tau}) \leq \frac{\on{poly}(1/\tau)}{T}.
    \end{align}
        \mgn{Weaker 1:
            \begin{align}
        \javg(t, \tau) := \sup_{s \leq t, w', v \in \mathbb{S}^{d-1}}\mathbb{E}_{w \sim \rho_0}\|J^{\perp}_{t, s}(w)H_s^{\perp}(w, w')v\|\mathbf{1}(\xi_t(w) \notin B_{\tau}) \leq \frac{\on{poly}(1/\tau) + \Theta(\log(d))}{t}.
    \end{align}
    Weaker 2:
        \begin{align}
        \int_{s = 0}^t \javg(s, \tau) \leq \on{poly}(1/\tau) + \Theta(\log(d) + \log(t)).
    \end{align}
        }
    \end{assumptionp}
Next, we will state our local strong convexity assumption. We remark that such an assumption can only hold when $\rho^*$ is atomic (see Remark~\ref{rem:lsc}, and additional comments in Section~\ref{dis:lsc}). 
\begin{assumptionp}{LSC (abbv)}[Local Strong Convexity \textnormal{(Abbreviated; see Assumption~\ref{def:lsc})}]\label{assm:lsc}
We have $(\clsc, \tlsc)$ \em locally strongly convex \em up to time $T$, meaning that for any $t \leq T$, for any $w$ with $\xi_t(w) \in B_{\tau}$, we have
\vspace{-0.15cm}
\begin{align} 
D^{\perp}_t(w)\preceq - \clsc P^{\perp}_{\xi_t(w)} \| f_{\rtmf} - f^* \| .
\end{align}
\end{assumptionp}
Both ~\ref{assm:growth} and \ref{assm:lsc} assumptions are verifiable via solving the deterministic mean-field dynamics $\rtmf$. For technical reasons, our result requires two additional conditions. First, our theorem depends on  the rank of the interaction Hessian as $\rtmf \rightarrow \rho^*$ being a constant independent of the ambient dimension $d$. 
This rank can be bounded by the following parameter, which will appear in our main theorem:
\begin{align}\label{def:cstarr}
    \cstarr := \min\left(|\supp{\rho^*}|, \on{dim}(\supp{\rho^*})^{2\on{degree}(\sigma) + 1}\right).
\end{align}
Here $\on{degree}(\sigma)$ is the degree of the polynomial $\sigma$ (or $\infty$ if $\sigma$ is not a polynomial). We do not expect such an assumption to be critical; see Section~\ref{dis:cstarr}.

Second, we require a technical symmetry condition stated in Assumption~\ref{assm:symmetries} (in Section~\ref{sec:assm}). Loosely, this requires that the atomic set $\supp{\rho^*}$ is \em transitive \em with respect to the group of rotational symmetries that describe the problem. We remark that such an assumption still covers many non-trivial problems, for instance, learning two teacher neurons in non-orthogonal positions, many neurons in orthogonal positions, or a ring of evenly spaced neurons in a circle. See Section~\ref{dis:sym} for further discussion.

We are now ready to state the main theorem.



\begin{theorem}[Propagation of Chaos]\label{thm:main}
Assume that ~\ref{assm:reg}, \ref{assm:growth}, \ref{def:lsc} and \ref{assm:symmetries} hold up to time $T$ (if relevant). Let $C$ be a constant depending on $\creg, \clsc, \tlsc, \cstarr$, and $\delta_T:= \| f_{\rho_T^{\textsc{MF}}} - f^*\|$. Suppose $n$, $m$ are large enough such that $\jmax^4T^3 (\eps_n + \eps_m) \leq 1/C$. Then with high probability over the draw $\rho_0^m$, for all $t \leq T$,
\begin{align}
\| f_{\rtmf} - f_{\rtm}\|^2 \leq O_{\creg}(\mathbb{E}_i \|\dit\|)^2 + \frac{\log(m)}{m}
     \leq (C\jmax t (\eps_m + \eps_n))^2.
\end{align}
where $\eps_m = \frac{\log(mT)\max(d^{1/2}\jmax, d^{3/2})}{\sqrt{m}}$ and $\eps_n = \frac{\sqrt{d}\log^2(n)}{\sqrt{n}}$. 
\end{theorem}
Theorem~\ref{thm:main} follows directly from Lemma~\ref{fact} and Corollary~\ref{cor:main} in Section~\ref{apx:balance}. In Theorem~\ref{theorem:SIM}, we will apply this theorem to the example of learning a single-index function with high information exponent which takes $T = \on{poly}(d)$ time to learn.


\vspace{-0.25cm}
\begin{remark}[Local Strong Convexity]\label{rem:lsc}
Our local strong convexity is similar to assumptions appearing in prior mean-field analyses~\cite[Assumption A5]{chizat2022sparse}\cite[Lemma D.9]{ chen2020dynamical}. In comparison to these works, our assumption is \em stronger \em in that we require it for all $t$, not just as $t \rightarrow \infty$; this is necessary for our non-asymptotic analysis. However, our assumption is also \em weaker \em in that we allow the strong convexity parameter to depend on the loss, similarly to the notion of one-point strong convexity (see e.g., \cite{safran2021effects}). Attaining the stronger non-loss-dependent strong convexity requires a strongly convex regularization term.

In problems where the mean-field dynamics converge to $\rho^*$, our local strong convexity condition enforces that when a neuron $w_t$ is close a teacher neuron $w^* \in \supp{\rho^*}$, it will be attracted to $w^*$ and thus any small perturbations are dampened. Local strong convexity can \em only \em hold when $\rho^*$ is atomic. Similar properties have been shown for various sparse optimization problem over measures~\cite{flinth2021linear, poon2023geometry}.

\end{remark}

\vspace{-0.8cm}
\subsection{Application to Single-index Model with High Information Exponent}\label{sec:applications}
 We now study the setting of learning a well-specified single-index function $f^*(x) = \sigma(x^\top w^*)$, where $w^* \in \sd$, and $\sigma(z) = \sum_{k = k^*}^K c_k \hek(z)$, where $(a)$ $k^* \geq 4$, and $\frac{1}{\csim} \leq c_{k^*} \leq \csim \max_k {c_k}$, $(b)$ $\sigma$ is an even function\footnote{If $\sigma$ is not even, the loss may not go to zero, since $1/2$ of the neurons may be stuck on the side of the equator with $w^\top w^* < 0$.}. We restrict to the case when $k^* > 2$ because  because the escape time for $k^* = 2$ is logarithmic in $d$, and thus can be handled via Gronwall's inequality; see  Remark~\ref{rem:ie2} for further comments. 
We assume the initial distribution $\rho_0$ of the neurons is uniform on $\sd$, and the data is drawn i.i.d from the distribution $\md$, which has Gaussian covariates, and subGaussian label noise: that is, 
\begin{align}
    x \sim \mathcal{N}(0, I_d), \qquad y = f^*(x) + \zeta(x); \qquad \mathbb{E}[\zeta(x)] = 0, \quad \mathbb{E}[\zeta(x)^2] \leq 1.
\end{align}
\vspace{-0.5cm}
\begin{restatable}[PoC in Single-Index Model]{theorem}{thmsim}
\label{theorem:SIM}
Fix any $\delta>0$, and suppose $d$ is large enough in terms of $\delta$, $\csim$ and $K$.  Let $T(\delta) := \arg\min\{t : \| f_{\rtmf} - f^*\| ^2 \leq \delta^2\}$. Then $T(\delta) = O_{K, \csim}(\sqrt{d}^{k^* - 2}\delta^{-(k^* - 1)})$.
If $n \geq d^{11k^*}$ and $m \geq d^{13k^*}$, then with high probability, for all $t \leq T(\delta)$,
$$\| f_{\rtmf} - f_{\rtm}\|^2 \leq \frac{O_{K, \delta}(d^{3k^*})}{\min\left(\sqrt{m}, \sqrt{n}\right)} \leq 3 \delta^2~.$$
\end{restatable}

\begin{remark*}
The above theorem provides, to the best of our knowledge, the first polynomial-width learning guarantee for one-hidden-layer neural network in the mean-field regime that holds for polynomial-in-$d$ time horizon. When $\on{degree}(\sigma) \gg k^*$, our result demonstrates the statistical advantage of the mean-field parameterization over the lazy/NTK alternative; specifically, under the NTK parameterization, when the width $m$ is sufficiently large, the sample complexity of gradient descent training on the empirical risk must scales as $n\gtrsim d^{\Theta(\on{degree}(\sigma))}$ \cite{ghorbani2021linearized}, whereas the mean-field scaling only requires $n\gtrsim d^{\Theta(k^*)}$ samples. 
\end{remark*}

%% file: Body/proof_overview.tex
\vspace{-0.5cm}
\section{Overview of Proof Ideas}\label{sec:pfoverview}

\subsection{Potential-Based Analysis to Prove Theorem~\ref{thm:main}}\label{sec:pfoverview:pot}
We introduce a potential function of $\Delta_t$ which dominates $W_1(\rtm, \brt)$. Building upon the observations from Section~\ref{sec:intuition}, we design this potential function to have the following three properties:
\begin{enumerate}[{\bfseries{P\arabic{enumi}}}]

\item\label{P1} When many neurons are near the teacher neurons, the dynamics due to the interaction hessian $\hpt$ should cause the potential to decrease.

\item\label{P2} When a neuron $w_i$ is in a locally convex region ($\dpt(i) \preceq 0$), the dynamics due to the local Hessian at $w_i$ should decrease the potential.

\item \label{P3} The change in potential due to a perturbation of $\Delta$ should be bounded proportionally to the \em average \em change over the $\Delta_i$.
\end{enumerate}
    A natural choice of potential function would be $\mathbb{E}_i\|\dit\|^2$ (which upper bounds $W_2(\rtm, \brt)$) because when $\rtmf \approx \rho^*$,  $D_t(i)$ are negative definite so $$\frac{d}{dt} \mathbb{E}_i\|\dit\|^2 \approx - \Delta_t^\top H^{\perp}_t \Delta_t - 2\mathbb{E}_i \dit^\top  D_t(i)\dit \leq 0.$$ However, such a function does not satisfy \ref{P3} whenever there is a lot of variance among the $\|\dit\|$.
    
    To achieve \ref{P3}, intuitively, the potential should behave more like $W_1(\rtm, \brt)$ than $W_2(\rtm, \brt)$, making $\mathbb{E}_i \|\dit\|$ another natural choice. Unfortunately, this alone does not work as potential function, because even when all neurons have converged to the support of $\rho^*$, it may \em increase \em under the dynamics from the interaction Hessian\footnote{Using $W_1(\rtm, \brt)$ alone (instead of $\mathbb{E}_i \|\dit\|$) fails for the same reason.}. 
    As an example, consider the case where $\rho^* = \delta_{w^*}$, and thus near convergence, $\hpt \approx \mathbf{1}\mathbf{1}^\top  \otimes P^{\perp}_{w^*}$, where $\mathbf{1} \in \{\sd \rightarrow \mathbb{R}\}$ sends all inputs to $1$; then if $\Delta_t$ is very ``imbalanced'' (in the sense that $\hpt \Delta_t = \mathbb{E}_i \dit$ is large), we may have $\frac{d}{dt} \mathbb{E}_i \|\dit\| > 0$. For instance suppose $\Delta_t(i) = u$ for a $p$ fraction of the neurons, and $\Delta_t(i) = 0$ for the remaining neurons. Then $\frac{d}{dt} \mathbb{E}_i \|\dit\| = -p + (1 - p) > 0$ for $p < 0.5$. To counteract the increase in $\mathbb{E}_i \|\dit\|$, we need to include in the potential function a term which decreases whenever $\Delta_t$ is very imbalanced, yet it retains a flavor of an $\ell_1$ norm. In order to tame the interactions, such a term should naturally take into account the
    eigendecomposition of $\hpt$.
To construct such a potential function, we will instead consider the eigendecomposition of the map $\hpi$ (defined explicitly in Defintion~\ref{def:hpi}), which closely approximates $\hpt$ on neurons in $B_{\tau}$ and avoids tracking the temporal evolution of the eigendecomposition. This ultimately lets us leverage the PSD structure of $\hpt$.
\begin{definition}\label{def:hpi}
    Define
    \begin{align}
        \hpit(w, w') = P^{\perp}_{\xii(w)}\nabla_{\xii(w')} \nabla_{\xii(w)} K(\xii(w), 
    \xii(w'))P^{\perp}_{\xii(w')},
    \end{align}
    where $\xii(w) := \argmin_{w^* \in \supp{\rho^*}} \|\xi_T(w) - w^*\|$ and we break ties in the argmin arbitrarily.
\end{definition}

Let $\mathcal{Z}:=L^2(\sd, \rho_0;  \mathbb{R}^d)$ be the Hilbert space with the dot product $\langle{f, g}\rangle_{\mathcal{Z}} := \mathbb{E}_{w \sim \rho_0} f(w)^\top g(w)$.
Define the action of $H: (\sd)^{\otimes 2} \to \R^{d \times d}$ on $\mathcal{Z}$ as $v \mapsto \overline{H}v(w) := \mathbb{E}_{w' \sim \rho_0}H(w, w')v(w').$
In Section~\ref{sec:app_balancedbody}, we verify that $\overline{\hpit}$ is well defined, self-adjoint, and due to the atomic nature of $\rho^*$, the span of $\overline{\hpit}$ is has some finite dimension $J$. Therefore, $\overline{\hpit}$ admits a spectral decomposition in $\mathcal{Z}$ in terms of an orthonormal basis $\{ \varphi_j \}_{j \leq J}$:
\begin{align}
\label{eq:spectral_basic0}\textstyle
    \overline{\hpit} &= \sum_{j \leq J} \lambda_j \varphi_j \otimes \varphi_j~,~ \lambda_j \in \R~,~\varphi_j \in \mathcal{Z}~,
\end{align}
such that $\|\overline{\hpit}\|_*:= \sum_j | \lambda_j | < \infty$. 
Note that one can have multiplicities in this spectral decomposition. 
For that purpose, denote by $\Lambda = \{ \lambda_j ; j \leq J\}$ the support of the spectrum. For each $\lambda \in \Lambda$, we denote by $V_\lambda$ the subspace spanned by $\{ \varphi_j; \lambda_j = \lambda \}$, and let $P_{\lambda}$ be the orthogonal projector onto that space.

\begin{definition}[Balanced Spectral Decomposition of $\hpi$ (\wed)]
We say that the spectral decomposition \eqref{eq:spectral_basic0} is $\cbal$-\emph{balanced} if, for all $\lambda \in \Lambda$, there exists an orthonormal basis $\mathcal{B}_\lambda$ of $V_\lambda$, and some $\eta_{\lambda} > 0$ such that 
for all $w \in \sd$,
$\sum_{ v \in \mathcal{B}_{\lambda}} v(w) v(w)^\top \preceq \eta^2_{\lambda} I_d~,$ and $\sum_{\lambda \in \Lambda} \eta^2_{\lambda} \leq \cbal$. We denote by $\mathcal{Q}:=\{ (\mathcal{B}_\lambda, \eta_\lambda)\}_{\lambda \in \Lambda}$ the resulting set of eigenfunctions and constants. 
\end{definition}

Now, for any $v \in \mathcal{Z}$ and $\Delta \in (\R^d)^{\otimes m}$, we define $\phi_v(\Delta):= | \mathbb{E}_i v(w_i)^\top \Delta(i) |$, and 
\begin{align}
    \Psi_{\mq}(\Delta) := \sum_{\lambda \in \Lambda}  
    \eta_{\lambda} \left(\textstyle \sum_{v \in \mathcal{B}_\lambda} \phi_{v}(\Delta)^2 \right)^{1/2},
\end{align}
Finally, our potential function is 
\begin{align}
\label{eq:potentialf}
\Phi_{\mq}(\Delta) &:= \Omega(\Delta) + \Psi_{\mq}(\Delta),    
\end{align}
with $\Omega(\Delta) = \mathbb{E}_i \|\Delta(i)\|$. 
When the context is clear, we will write $\Phi_{\mq}(t) = \Phi_{\mq}(\Delta_t)$. 

When the context is clear, we will write $\Phi_{\mq}(t) = \Phi_{\mq}(\Delta_t)$. 
\begin{lemma}[Balanced Spectral Decomposition]\label{lemma:balancedsimple} 
    Suppose Assumption~\ref{assm:symmetries} holds. Then there exists an spectral distribution $\mq$ 
    which is $\cstarr = \min\left(|\supp{\rho^*}|, \on{dim}(\supp{\rho^*})^{\on{degree(\sigma)}}\right)$-balanced.
\end{lemma}

The next three lemmas show that the potential function $\Phi_{\mq}$ has the desired properties \ref{P1}-\ref{P3}.
\begin{restatable}[Descent with Respect to Interaction Term]{lemma}{dech}\label{lemma:dec1body}
Let $\Phi_{\mq}(t)$ be as defined above, where $\mq$ is a $\cbal$-balanced spectral decomposition of $\hpit$. Then for any $\tau > 0$ for which the concentration event of Lemma~\ref{lemma:conctau} holds for $S = B_{\tau}$, we have
    $$
        \langle{\nabla \Phi_{\mq}(t), -\hpt \Delta_t}\rangle\leq (1 + \cbal)\mathbb{E}_i \|\mathbb{E}_j \hpt(i, j)\djt\|\mathbf{1}(\xi_t(w_i) \notin B_{\tau}) + \mathcal{E}_{\ref{lemma:dec1body}},$$
    where  $\mathcal{E}_{\ref{lemma:dec1body}} = C_{\ref{lemma:dec1body}}(\mathbb{E}_i \|\dit\|\mathbf{1}(\xi_t(w_i) \notin B_{\tau}) + (\tau + \cbal \eps_m^{\ref{lemma:conctau}}) \Omega(t))$ for some $C_{\ref{lemma:dec1body}} = O_{\creg, \cbal}(1)$.
\end{restatable}
\begin{restatable}[Descent with Respect to Local Term]{lemma}{decd}\label{lemma:dec2body}
Suppose Assumption~\ref{def:lsc} holds with $(\clsc, \tlsc)$. Let $\mq$ be a $\cbal$-balanced spectral distribution. Then with $\cdecd = O_{\creg, \cbal}(1)$, we have
$$\langle{\nabla \Phi_{\mq}(t), \dpt \odot \Delta_t}\rangle \leq - \Big(\textstyle\frac{c\sqrt{L_{\md}(\rtmf)}}{2} - \cdecd \tau \Big) \Phi_{\mq}(t) + \cdecd\mathbb{E}_{i}\|\dit\|\mathbf{1}(\bwti \notin B_{\tau}) + \cbal\mathbb{E}_i \|\dit\|^2.$$
\end{restatable}

\begin{restatable}[L1 Perturbation Lemma]{lemma}{pert}\label{lemma:l1pertbody}
Let $\mq$ be a $\cbal$-balanced spectral distribution. 
Let $G : [m] \rightarrow \mathbb{R}^d$. Then
$\left|\langle{\nabla \Phi_{\mq}(t), G}\rangle \right|\leq (1 + \cbal) \mathbb{E}_i \|G(i)\|.$
\end{restatable}

Combining the three key properties of the potential function, along with Assumption~\ref{assm:growth} allows us to bound the dynamics of the potential function in the following way (formalized in Theorem~\ref{lemma:ddtphi}):
\begin{align}\label{eq:dynphi}
\frac{d}{dt} \Phi_{\mq}(t) &\leq - \frac{\clsc \sqrt{L(\rtmf)}}{C} \Phi_{\mq}(t) + C\javg \int_{s = 0}^t \Phi_{\mq}(s)ds + C\jmax(\eps_{m} + \eps_n),
\end{align}
where $C = O_{\cstarr, \creg}(1)$.
Theorem~\ref{thm:main} follows by analyzing this differential equation.
We leverage Assumption~\ref{assm:growth} to prove \eqref{eq:dynphi}, by bounding the term $\mathbb{E}_{i}\|\dit\|\mathbf{1}(\bwti \notin B_{\tau})$ which arises from Lemmas~\ref{lemma:dec1body} and~\ref{lemma:dec2body}. 

\subsection{Self-Concordance Argument to Bound $\jmax$}\label{sec:pfoverview:jmax}


To avoid exponential growth in $J^{\perp}_{t, s}$, we make the following observation. 
\begin{observation}\label{obs:1}
When the velocity $\nu(w, \rtmf)$ of a particle $w$ is small, so is $\|D^{\perp}_t\|$.
\end{observation}
To make this observation more concrete, consider the simplified case of learning a single-index function $f^*(x) = \sigma(x^\top w^*)$ with Gaussian data, where $\sigma(z) = \hek(z)$. We expect a similar property may hold in other low-dimensional feature-learning problems, where the local non-convexity arises only in a low-dimensional subspace.
For a neuron $w_t$, when $\alpha_t := w_t^\top w^*$ is small (and assume for simplicity that $\alpha_t$ is positive), we have that
\begin{align}
\nu(\alpha_t) &:= \frac{d}{dt} \alpha_t \approx \alpha_t^{k-1}~,\text{ thus }~
\frac{d}{d\alpha} \nu(\alpha_t)  \approx (k - 1)\alpha_t^{k-2} \approx \frac{k - 1}{\alpha_t} \nu(\alpha_t)~.
\end{align}
By showing that $\|D^{\perp}_t\|$ is dominated by $\frac{d}{d\alpha} \nu(\alpha_t)$,
we get the desired ``self-concordance'' property:
\begin{align}
    \|D^{\perp}_t\| = \left\|\nabla_{w_t} \nu(w_t, \rtmf)\right\| \lessapprox \frac{(k-1)}{\alpha_t} \nu(\alpha_t).
\end{align}

Recalling the differential equation of $J^{\perp}_{t, s}$ , we have just shown that $\frac{d}{dt} \|J^{\perp}_{t, s}\| \leq \frac{(k-1)}{\alpha_t} \nu(\alpha_t) \|J^{\perp}_{t, s}\|$. Note that trivially, $\alpha_t$ satisfies the differential equation $\frac{d}{dt} \alpha_t = \frac{1}{\alpha_t} V(\alpha_t) \alpha_t.$ 
Comparing the two differential equations above, we see that $\|J^{\perp}_{t, s}\| \leq \left(\frac{|\alpha_t|}{|\alpha_s|}\right)^{k-1}$. We have almost shown that $\|J^{\perp}_{t, s}\|$ is polynomially bounded in $d$ and $t$. 
Indeed, for a typical neuron, $|\alpha_0|$ is on the order of $\frac{1}{\sqrt{d}}$, and $|\alpha_t| \leq 1$, so $\|J^{\perp}_{t, s}\| \leq O(\sqrt{d}^{k-1})$. To make the argument work for neurons with small $\alpha_s$, observe that since $\frac{d}{dt}|\alpha_t| \leq O(|(\alpha_t)^{k-1}|)$, if $k \geq 3$\footnote{ $\|J^{\perp}_{t, s}\|$ can also be uniformly bounded when $k=2$; this requires showing that $|\frac{d}{dt}\alpha_t| \leq \sqrt{L_{\rtmf}} \leq \frac{1}{1 + \max(0, t - \Theta(\log(d)))}$}, we must have that
if $t - s \leq O\left(\frac{1}{(\alpha_s)^{k-2}}\right)$, then $\alpha_t \leq 2\alpha_s$. Thus, either $\alpha_t \leq 2\alpha_s$, or $ \frac{1}{(\alpha_s)^{k-2}} \leq O(t - s)$. It follows that for all neurons, $ \|J^{\perp}_{t, s}\| \leq O\left((t-s)^{\frac{k-1}{k-2}}\right).$

\subsection{Averaging Argument to Bound $\javg$}\label{sec:pfoverview:javg}
Recall that in order to use our approach to achieve a propagation of chaos for polynomially sized networks, for any $w', v \in \mathbb{S}^{d-1}$ and $\tau$, we must have
\begin{align}
  \sup_{s,t \leq T, w', v \in \sd} \mathbb{E}_{w \sim \rho_0}\|J^{\perp}_{t, s}(w)H_s^{\perp}(w, w')v\|\mathbf{1}(\xi_t(w) \notin B_{\tau}) \leq O_{\tau}\left(\frac{1}{T}\right),
\end{align} where $T$ 
is the desired training time. We briefly give some intuition for why this holds in the single-index model $f^*(x) = \hek(x^\top w^*)$, which requires $T = \Theta(d^{(k - 2)/2})$. 
To tightly bound $\javg(\tau)$, we leverage the fact that neurons far from $\pm w^*$ are dispersed. 
By averaging over the ``level set'' of neurons with $\alpha_s(w) = \alpha$ (where $\alpha_t(w) := |{w^*}^\top \xi_t(w)|$) we have
\begin{align}
    \sup_{w', v \in \sd} \mathbb{E}_{w : |\alpha_s(w)| = \alpha} \|H_s^{\perp}(w, w')v\| \leq \max\left(d^{-1/2}, \alpha\right)^{k-1}.
\end{align}
Plugging this in for $t \leq T$, along with the bound $\|J^{\perp}_{t, s}\| \leq \left(\frac{|\alpha_t|}{|\alpha_s|}\right)^{k-1}$ from above, yields
\begin{align}
    \javg(\tau) &\leq \mathbb{E}_w \left(\frac{|\alpha_t(w)|}{|\alpha_s(w)|}\right)^{k-1} \max\left(\sqrt{d}^{-1}, \alpha_s(w)\right)^{k-1}\mathbf{1}(|\alpha_t(w)| \leq 1 - \tau)\nonumber \\
    &\lessapprox \mathbb{E}_w |\alpha_t(w)|^{k-1}\mathbf{1}(|\alpha_t(w)| \leq 1 - \tau),
\end{align}
Bounding this final term results from the observation the particles escape the saddle  at roughly uniform time in the interval $[0, T]$ (see Figure~\ref{fig:nonuniform}(right) and Proposition~\ref{prop:SIMmf}). 

%% file: Body/sims.tex
\section{Full Statement of Assumptions and Discussion}\label{sec:assm}

In this section, we explore whether propagation of chaos may hold more generally than beyond our setting and conditions. We provide several remarks on the necessity of our assumptions, both in the context of our proof approach, and based on empirical simulations, which are given in full in Section~\ref{sec:sims}.

\subsection{Omitted Assumptions}\label{sec:omitted}
We begin by stating the full versions of the assumptions which were omitted in Section~\ref{sec:result}. We then briefly discuss the definition of propagation of chaos and several related phenomena in Section~\ref{sec:pocdefs}. Finally, in Sections~\ref{dis:spherical}-\ref{dis:cstarr}, we provide remarks on the assumptions.

Let $V = \on{span}(\supp{\rho^*})$ and let $U$ be the space orthogonal to $V$ in $\mathbb{R}^d$. Let 
\begin{align}
    \cstarr := \min\left(|\supp{\rho^*}|, \on{dim}(V)^{2\on{degree}(\sigma) + 1}\right).
\end{align}

\begin{assumptionp}{LSC}[Local Strong Convexity (Full Version of Assumption~\ref{assm:lsc})]\label{def:lsc}
The problem is $(\clsc, \tlsc)$ \em locally strongly convex \em up to time $T$ if for any $t \leq T$ and any $w$ with $\xi_t(w) \in B_{\tau}$, we have
\begin{align} 
D^{\perp}_t(w)\preceq -\clsc P^{\perp}_{\xi_t(w)} \| f_{\rtmf} - f^* \|~. 
\end{align}
Further, the strong convexity is \em structured, \em if there exist values $c_t^1, c_t^2 \geq \clsc$ such that for any  $w$ with $\xi_t(w) \in B_{\tau}$, we have 
\begin{align}
    \|c_t^1VV^\top  P^{\perp}_{\xii(w)}VV^\top  + c_t^2UU^\top  - \dpt(w)\| \leq \left(\frac{\clsc  \| f_{\rtmf} - f^* \|}{2\sqrt{\cstarr}} + \creg \tau\right).
\end{align}
\end{assumptionp}

\begin{assumptionp}{Symmetry}[Symmetries of $\rho^*$]\label{assm:symmetries}
    The \em automorphism group \em $\cg$ of a problem $(\rho^*, \rho_0, \md_x)$ is the group of rotations $g$ on $\sd$ where for any $A \subset \sd$:
        \begin{align*}
            \mathbb{P}_{\rho^*}[A] &= \mathbb{P}_{\rho^*}[g(A)] \qquad \mathbb{P}_{\md}[A] = \mathbb{P}_{\md_x}[g(A)] \qquad \mathbb{P}_{\rho_0}[A] = \mathbb{P}_{\rho_0}[g(A)]
        \end{align*}
    We assume:
    \begin{enumerate}[{\bfseries{I\arabic{enumi}}}]
        \item\label{def:transitiveI1} $\supp{\rho^*}$ is \em transitive \em under $\cg$, that is, for any $w^*, {w^*}' \in \supp{\rho^*}$, there exists $g \in \cg$ such that $g(w^*) = {w^*}' $. Further, $\mathbb{P}_{w \sim \rho_0}[\{ \|w - w^*\| = \|w - {w^*}'\| \exists {w^*, w^*}' \in \supp{\rho^*} \}] = 0$.
        \item\label{assm:rhostarnew} Let $V = \on{span}(\supp{\rho^*})$ and let $U$ be the space orthogonal to $V$ in $\mathbb{R}^d$. Then the distribution $\md_x$ on covariates $x$ factorizes over $U$ and $V$, that is $\md_x = \md_U \otimes \md_V$, where $\md_U$ is a distribution on $V$ and $\md_U$ is a distribution on $U$. Further, $\mathbb{E}_{x_U \sim \md_U}x = 0$, and $\mathbb{E}_{x_U \sim \md_U}xx^\top  = UU^\top $. 
    \end{enumerate}
    \end{assumptionp}

\subsection{Measuring Propagation of Chaos}\label{sec:pocdefs}

A standard definition\footnote{Stronger notions of uniform convergence over $t$ are also available; see \cite[Section 3.4]{chaintron2022propagation}.} of propagation of chaos (PoC) (see \cite[Prop. 2.2]{sznitman1991topics}) is that for all $t$, we have the convergence in law of the random distribution $\rtm$ to the constant distribution $\rtmf$:
\begin{align}\label{eq:poc}
    \lim_{m \rightarrow \infty} \rtm \rightarrow \rtmf.
\end{align}
Equivalently, for any two continuous test functions $\psi_1, \psi_2$, we have that
\begin{align}\label{eq:poc_k2}
    \lim_{m \rightarrow \infty} \mathbb{E}_{w_1,w_2 \sim \rtm} \psi_1(w_1)\psi_2(w_2) = \left(\mathbb{E}_{w \sim \rtmf} \psi_1(w)\right)\left(\mathbb{E}_{w \sim \rtmf} \psi_2(w)\right)~.
\end{align}
Of primary interest in our paper is a weaker PoC phenomenon, which we will henceforth refer to as \em PoC in function error\em: almost surely with respect to the draw of $\hat{\rho}_0^m$, 
\begin{align}
     \lim_{m \rightarrow \infty} \| f_{\rtm} - f_{\rtmf}\|^2 \rightarrow 0.
\end{align}
It is easy to check that PoC in function error is implied by \eqref{eq:poc_k2} by using test functions of the form $\psi_x(w) := \sigma(w^\top x)$. PoC in function error implies convergence of the risk of $\rtm$ to the risk of $\rtmf$, and thus is the most practically relevant (see e.g.~\cite{suzuki2022uniform}). 
On the other hand, our proof considers a much stronger PoC phenomenon. Our potential-function based proof yields almost surely over the initialization
\begin{align}
     \lim_{m \rightarrow \infty } \Omega(\Delta_t) := \mathbb{E}_i \|\dit\|_2 = 0.
\end{align}
Here the $m$ is implicit in $\Delta_t$. This is a much stronger notion than \eqref{eq:poc} (it implies $\lim_{m \rightarrow \infty} W_1(\rtm, \rtmf) = 0$), and we will refer to it as \em PoC via fixed parameter-coupling. \em

Remarkably in our neural network setting (though not necessarily for general interacting particle systems), up to the parameter $\jmax$ and a time horizon $t$, PoC in function error for all $s \leq t$ implies PoC via fixed parameter-coupling. Indeed, by \eqref{eq:duhamel}, we have that 
\begin{align}\label{eq:ferrperr}
    \|\dit\| &\leq \int_0^t \|J^{\perp}_{t, s}(i)\| \left(\| \mathbb{E}_j H^{\perp}_s(i, j)\djs\| + \| \bm{\eps_{s, i}}\|\right)ds \nonumber \\
    &= O\left(\jmax \int_{0}^t \left(\sqrt{\Delta_s^\top  H^{\perp}_s \Delta_s} + \mathbb{E}_j \|\djs\|^2 + \|\dis\|^2 + \eps_m + \eps_n\right)ds\right) \nonumber \\
    &= O\left(\jmax \int_{0}^t \left(\| f_{\rsm} - f_{\rsmf}\| + \mathbb{E}_j \|\djs\|^2 + \|\dis\|^2 + \eps_m + \eps_n\right)ds\right).
\end{align}
Here the last equation follows from the following lemma proved using a Taylor expansion of $f_{\rtm}$.
\begin{restatable}{lemma}{fparam}\label{lemma:fparam}
Suppose Assumption~\ref{assm:reg} holds. For any $t$, we have
$$\Delta_t^\top  \hpt \Delta_t \leq 2 \|f_{\rtmf} - f_{\rtm} \|^2  + \frac{\log(m)}{m} + O_{\creg}\left(\mathbb{E}_i \|\dit\|^2\right)^2.$$
\end{restatable}

Solving Eq.~\eqref{eq:ferrperr} yields that for $m$ such that $\| f_{\rsm} - f_{\rsmf}\| + \eps_m + \eps_n \ll \frac{1}{(t \jmax)^2}$ for all $s \leq t$,
\begin{align}
\label{eq:linkPOC}
    \mathbb{E}_i\|\dit\| = O\left(\jmax t \max_{s \leq t}\left(\|f_{\rsm} - f_{\rsmf}\| + \eps_m + \eps_n\right)\right).
\end{align}
Indeed, one can show this by inductively bounding the second order terms from time $0$ to $t$.

All of the above PoC phenomena can be quantified non-asymptotically, and the main question of this paper is whether for certain problems the above quantities (or their differences) decay at a rate
$
    \frac{\on{poly}(t, d)}{\on{poly}(m)},
$
uniformly over all $d > 0$ and all $t \in [0, T(d)]$. Here $T(d)$ is a desired stopping time, e.g., when some fixed population loss $\eps$ is achieved. 



\subsection{Spherical Constraint and Second Layer Weights}\label{dis:spherical}

When the weights of the neural network are not constrained to the sphere, propagation of chaos may fail even in simple well-specified settings: to see this, consider the case of learning a SIM with information exponent $k > 2$ using a neural network with homogeneous activation function. With polynomial width, we expect the standard $T \approx d^{(k-2)/2}$ convergence time. Whereas at the infinite-width limit, we may learn the target function by amplifying neurons that already attain large alignment at initialization due to homogeneity. 
In particular, we can achieve $T = d^{(k-2)/k}$ convergence time by leveraging neurons with initial alignment greater than $d^{-1/k}$ --- to see this, observe there is roughly $\exp(- d^{(k-2)/k})$ fraction of neurons in the network with such initial overlap, and thus we need to grow these neurons to a scale of $\exp(d^{(k-2)/k})$, which takes $d^{(k-2)/k}$ time. 
A similar phenomenon occurs if we train the second-layer weights in the network and allow them to be unbounded. Note, however, that there is nothing precluding our results from holding if the (fixed) second layer is initialized differently.  

\subsection{Local Strong Convexity (Assumption~\ref{def:lsc})}\label{dis:lsc}
We focus here on the main part captured in Assumption~\ref{assm:lsc}. See also the previous Remark~\ref{rem:lsc}. The additional \em structured \em condition in Assumption~\ref{def:lsc} is discussed with the symmetry conditions. 

Local strong convexity plays a key part in how we bound the potential $\Phi_{\mq}(t)$ via the differential equation in \eqref{eq:dynphi}. Indeed, plugging in the bound $\javg \leq 1/T$ yields:
$$\frac{d}{dt} \Phi_{\mq}(t) \leq - \frac{\clsc \sqrt{L(\rtmf)}}{C} \Phi_{\mq}(t) + \frac{C}{T} \int_{s = 0}^t \Phi_{\mq}(s)ds + C\jmax(\eps_{m} + \eps_n).$$
If $\clsc$ goes to $0$, then the best bound on this differential equation becomes
\begin{align}
    \textstyle \Phi_{\mq}(t) \lessapprox C\jmax(\eps_{m} + \eps_n)\exp\left(t\sqrt{\frac{C}{T}}\right),
\end{align}
which would require that $m$ be super-polynomially large in $T$ in order to bound $\Phi_{\mq}(T)$. 

As discussed in Remark~\ref{rem:lsc}, local strong convexity can only hold in problems where $\rho^*$ is \em atomic. \em Thus it cannot capture example when $\rho^*$ is distributed on a manifold, or for ``misspecified'' problems where the target link function differs from the network activation, e.g., $f^*(x) = \phi(x^\top w^*)$ for $\phi \neq \sigma$. These examples are particularly interesting because training with the correlation loss is insufficient. In our $1$- or $2$-index non-atomic experiments, however, we still observed propagation of chaos for the values of $m$ we simulated (see e.g., the \mis~ and \mantwo~ problems depicted in Figures~\ref{fig:body:mis}, \ref{fig:circle}). 

In non-atomic examples, it is unreasonable to hope that $\mathbb{E}_i \|\dit\|$ will remain bounded for all $t$; thus addressing this case would require proving either a bound on the Wasserstein-1 distance, $W_1(\rtm, \brt)$, or a bound on the function error $\|f_{\rtm} - f_{\rtmf}\|^2 \approx \Delta_t^\top \hpt \Delta_t$.

\subsection{Stability Conditions (Assumption~\ref{assm:growth})}\label{dis:stability}

To achieve propagation of chaos with polynomially many neurons, we believe it is necessary in standard settings that $\sup_{s, t \leq T} \|J_{t, s}(w)\|$ is polynomially bounded with high probability over $w$. This is only a slightly weaker condition than the current $\jmax$ assumption. Getting around such an assumption would require strong directional control over the $\bm{\eps}_{t, i}$, which we do not expect to be possible. 

The necessity of the strong assumption on $\javg$, however, is mysterious to us. The neuron-to-neuron error-propagation described in \autoref{para:blessing} 
seems hard to prevent without a similar assumption. Even if we leverage the fact that the interaction term is PSD (and thus creates a repulsion between the neurons), there could be oscillatory exponential growth of the $\dit$'s. 
Nevertheless, in our simulations, we were not able to find an example where violating the $\javg$ assumption precluded propagation of chaos; see for example the \hoponethree~ problem depicted in Figure~\ref{fig:staircase}.



\begin{remark}[Order-1 Saddles /Information Exponent 2]\label{rem:ie2}
The standard Gronwall-inequality approach can yield propagation of chaos with $\on{poly}(d)$ neurons for information exponent (IE) 2 problems, which require $\log(d)$ time to converge. For this reason, our work does not focus on the IE 2 case, though we believe our techniques could be useful for proving propagation of chaos for longer ($\on{poly}(d)$) time horizons in IE 2 problems. Certain modifications are needed, however. 

Indeed, the reader may notice that the assumption on $\javg$ fails in simple cases with IE $2$. This occurs because the neurons all escape the saddle at roughly the same time (in contrast to the non-uniform escape times for higher IE shown in Figure~\ref{fig:nonuniform}(c)). Thus there exists some time $t$ (roughly this escape time) where the expression in Assumption~\ref{assm:growth} is of order 1. We believe overcoming this obstacle should be possible by working with a $t$-dependent version of $\javg$, which is small on average over $t$. Further, for longer time horizons, a more careful analysis is required to show that $\jmax \leq \on{poly}(d, t)$. For SIMs, we believe this can be accomplished by showing that $\|D_t(w)\|$ decays with the square root of the loss $L(\rtmf)$, which decays like $1/t^2$. This would yield the bound $\|J_{s, t}(w)\| \leq \exp\left(\int_{r = s}^t \|D_r(w)\| dr\right) \approx \exp(\Theta(\log(t - s))) = \on{poly}(t)$.

\end{remark}

\subsection{Symmetry Conditions}\label{dis:sym}
\paragraph{The \em structured \em condition in Assumption~\ref{def:lsc}.}
The structured condition Assumption~\ref{def:lsc} is used in proving Lemma~\ref{lemma:dec2body}, which shows that the potential decreases due to the local strong convexity near the teacher neurons. We believe its necessity is an artifact of how we designed the potential function.

In general, the structured condition holds for all 2-index problems with Gaussian data, because at $\xii(w) \in \on{span}(V)$, (a) the projection of $\dpt(w)$ onto the $U$-space will be a multiple of $UU^\top $, and (b)  the projection of $\dpt(w)$ onto the $V$-space will be one-dimensional, and thus a multiple of $VV^\top  P^{\perp}_{\xii(w)}VV^\top $. Using a continuity argument between $\xii(w)$ and $\xi_t(w)$ yields the condition for all $\xi_t(w) \in B_{\tau}$. 
Beyond 2-index problems, we are not aware of exactly when this condition holds, though we expect it does not hold for many non-symmetric problems.

\paragraph{Symmetry Assumption (Assumption~\ref{assm:symmetries}).}
The transitivity condition (\ref{def:transitiveI1} in Assumption~\ref{assm:symmetries}) plays an important role in our proof. Namely, Lemma~\ref{lemma:cri} (see also Definition~\ref{def:cri}) uses transitivity to guarantee that the eigenfunctions of the interaction matrix \em at convergence time, \em are also eigenfunctions of the interaction submatrix of neurons that have converged at time $t$ (for any $t$!). This ``restricted isometry'' property allows us to define our potential function \em independently \em of the time $t$. We expect that without restricted isometry, one would have to design a potential function which depends on the time $t$.

The transitive condition in \ref{def:transitiveI1} holds for various non-trivial teacher-student problems, for example: learning $k$ orthogonal teacher neurons for any $k$, learning any two non-orthogonal teacher neurons, or learning a ring of equally spaced teacher neurons on a circle. In the latter two examples, training with the correlation loss may fail (for instance, for two non-orthogonal neurons with a small angle between them, correlation loss may converge to the linear combination of teacher neurons); to our knowledge, gradient training of many of these ``simple'' symmetric examples is still not understood.  
Note also that the second part of \ref{def:transitiveI1} holds when $\rho_0$ has bounded Radon-Nikodym derivative with respect to the Lebesgue measure on the sphere, or when $\rtmf \rightarrow \rho^*$. Both imply $0$ mass on the boundary points.

In all the non-symmetric examples we simulated, the lack of symmetry does not pose an obstacle to propagation of chaos; see for example the \xorfour, \hoponethree, \mis, 
problems in Figures~\ref{fig:body:mis}, \ref{fig:body:xor}, \ref{fig:staircase}.



\subsection{Dependence on $\cstarr$}\label{dis:cstarr}
The value $\cstarr$ is bounded whenever $\rho^*$ is atomic with constant number of neurons, or when $\rho^*$ is in a finite-dimensional subspace, and $\sigma$ is polynomial. This includes all polynomial multi-index functions.

The value $\cstarr$, defined in \eqref{def:cstarr}, functions as an upper bound on the rank of the interaction-kernel $k(w, w') = \mathbb{E}_x \sigma'(x^\top w)\sigma'(x^\top w')$ over points in $w, w' \in \on{supp}(\rho^*)$. Having a constant upper bound on the rank of this kernel is useful in constructing a balanced spectral decomposition of $\hpi$, which (loosely) ensures that near convergence time, small $L_1$-bounded changes in $\Delta_t$ cannot propagate (via the force of the interaction kernel) to large changes, measured in $L_1$. While it may be possible, we have not been able to find any simple ways to prevent $L_1$-growth in $\Delta_t$ near convergence time without this constant-rank assumption. In Section~\ref{sec:sims}, we simulate several examples in which $\cstarr$ grows polynomially in $d$. The presence of various PoC phenomena did not seem to be correlated with the size of $\cstarr$ --- observe that the \mis~example (Figure~\ref{fig:body:mis}), which has $\cstarr$ that grows polynomially in $d$, demonstrated PoC for relatively small widths.


\section{Experiments}\label{sec:sims}

We conduct simulations both to validate our theory in settings which we expect satisfy our assumptions, and to examine what happens when these assumptions fail to hold. Table~\ref{table:sims} in Appendix~\ref{apx:sims} describes all of the settings we simulated, and documents which assumptions we believe they satisfy. We remark that we did not preferentially chose these examples because we expected (or observed) propagation of chaos: we in fact ran these simulations with the goal of finding a multi-index function in which propagation of chaos fails, and have so far been unsuccessful.


\subsection{Experiment Setting}

Since we could not simulate an infinite-width network, we measured certain proxies for the distance between $\rtm$ and $\rtmf$ by comparing a neural network of width $m$ to a neural network of width $M$, for $M \gg m$. The full experimental design is described in Section~\ref{apx:sims:design}. 
In brief, we initialize the smaller width-$m$ network to be a subset of the neurons in the larger width-$M$ network; in this way, we can track for all $i \in [m]$ the coupling differences $\hdit$ (a proxy for $\dit$\footnote{Triangle inequality yields that $\|\hdit\| \leq 2\|\dit\|$, though the converse is not necessarily true.}) throughout training. In our plots, we estimate the following quantities from data throughout the training dynamics: $(a)$ the prediction risk or generalization error, $(b)$ the function error $m\cdot\| f_{\rtm} - f_{\rtmf}\|^2$, and $(c)$ the fixed parameter-coupling error $m\cdot\mathbb{E}_i \|\hdit\|$. In our full plots in the appendix, we also include several histograms of $\|\hdit\|$. We plot the above values for a range of widths $m$, and examine the decay rate in $m$. 

In Figures~\ref{fig:body:hefour},\ref{fig:body:mis},\ref{fig:body:xor} we consider $(i)$ the well-specified Gaussian single-index setting with $\text{He}_4$ activation function, which satisfies all our assumptions in Section~\ref{sec:assm}, $(ii)$ a misspecified single-index setting where we do not expect $\rho^*$ to be atomic (see e.g., \cite{mahankali2023beyond}), and $(iii)$ the multi-index setting of $4$-parity function similar to that studied in \cite{glasgow2023sgd,telgarsky2023feature,suzuki2023feature}. 

\subsection{Takeaways from Simulations}

We describe some of our takeaways from the experiments below. More figures can be found in Appendix~\ref{apx:sims}.

\paragraph{PoC in function error.}

In all examples we simulated, we observed that for $m$ large enough, the function error $\| f_{\rtm} - f_{\rtmf}\|^2$ decayed at least linearly with the width $m$ -- this is evident in Figures~\ref{fig:body:hefour},\ref{fig:body:mis},\ref{fig:body:xor}(b). Surprisingly, in all examples the function error decayed nearly monotonically in time. 

\paragraph{PoC via fixed parameter-coupling.}
Similarly to the function error, we observed that the parameter-coupling error $\left(\mathbb{E}_i \|\hdit\|\right)^2$ decayed at at $1/m$ rate -- this is evident in Figure~\ref{fig:body:hefour},\ref{fig:body:mis},\ref{fig:body:xor}(c). However, unlike the function error, the growth of this error appeared to be linear in $t$, which is consistent with the upper bound on parameter-coupling error in terms function error given in \eqref{eq:linkPOC}. We note that our experiments show that \eqref{eq:linkPOC} is not quite tight, as the parameter-coupling error seems to grow slower than (a scaling of) the integral of the function error over time. More extensive experiments could provide more insight on this.

\begin{figure}[t]
\begin{minipage}{0.32\linewidth}
\centering
{\includegraphics[width=0.96\linewidth]{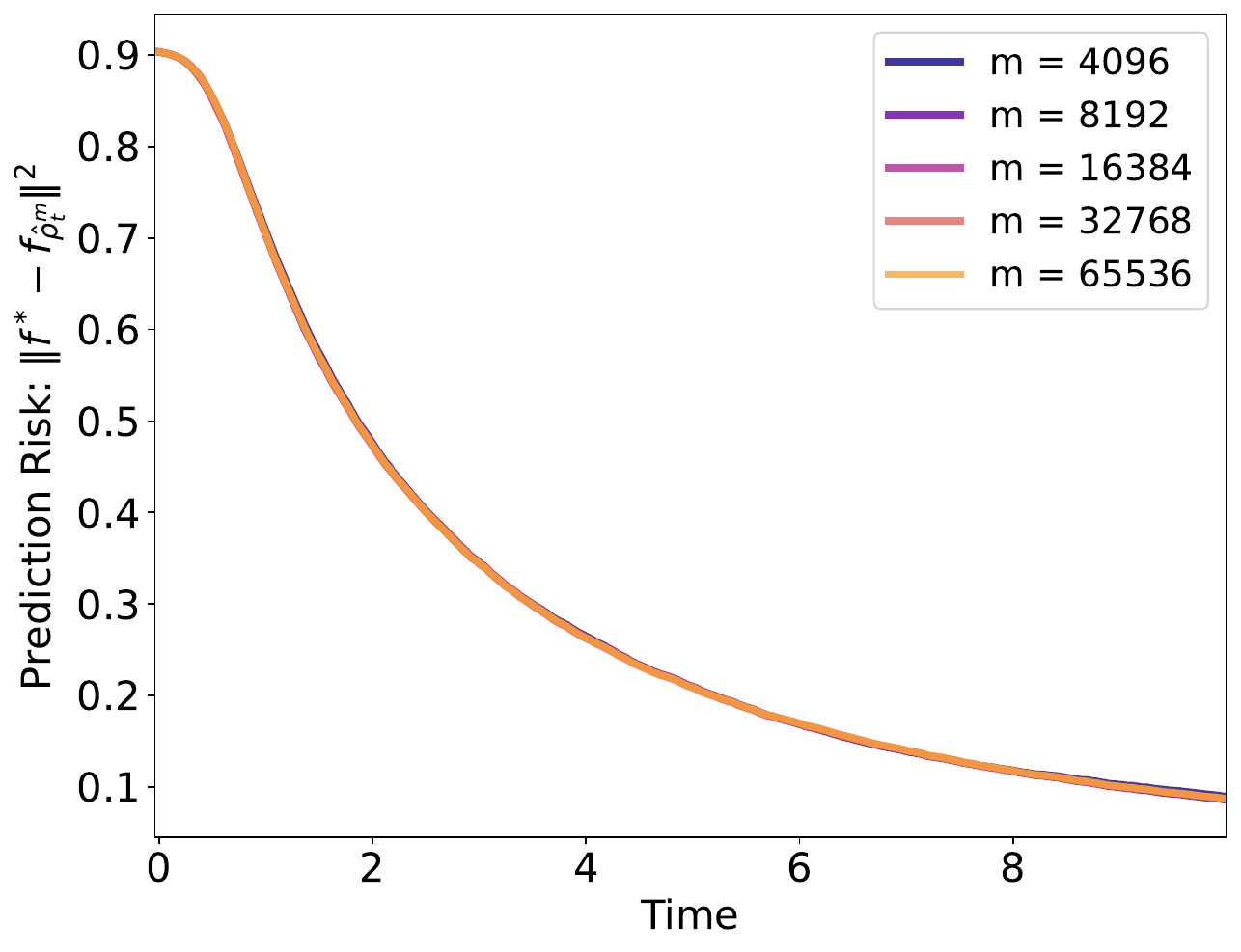}}  \\ \vspace{-1mm}
\small (a) prediction risk \\ $\|f_{\rtm} - f^*\|^2$.  
\end{minipage}%
\begin{minipage}{0.32\linewidth}
\centering
{\includegraphics[width=0.96\linewidth]{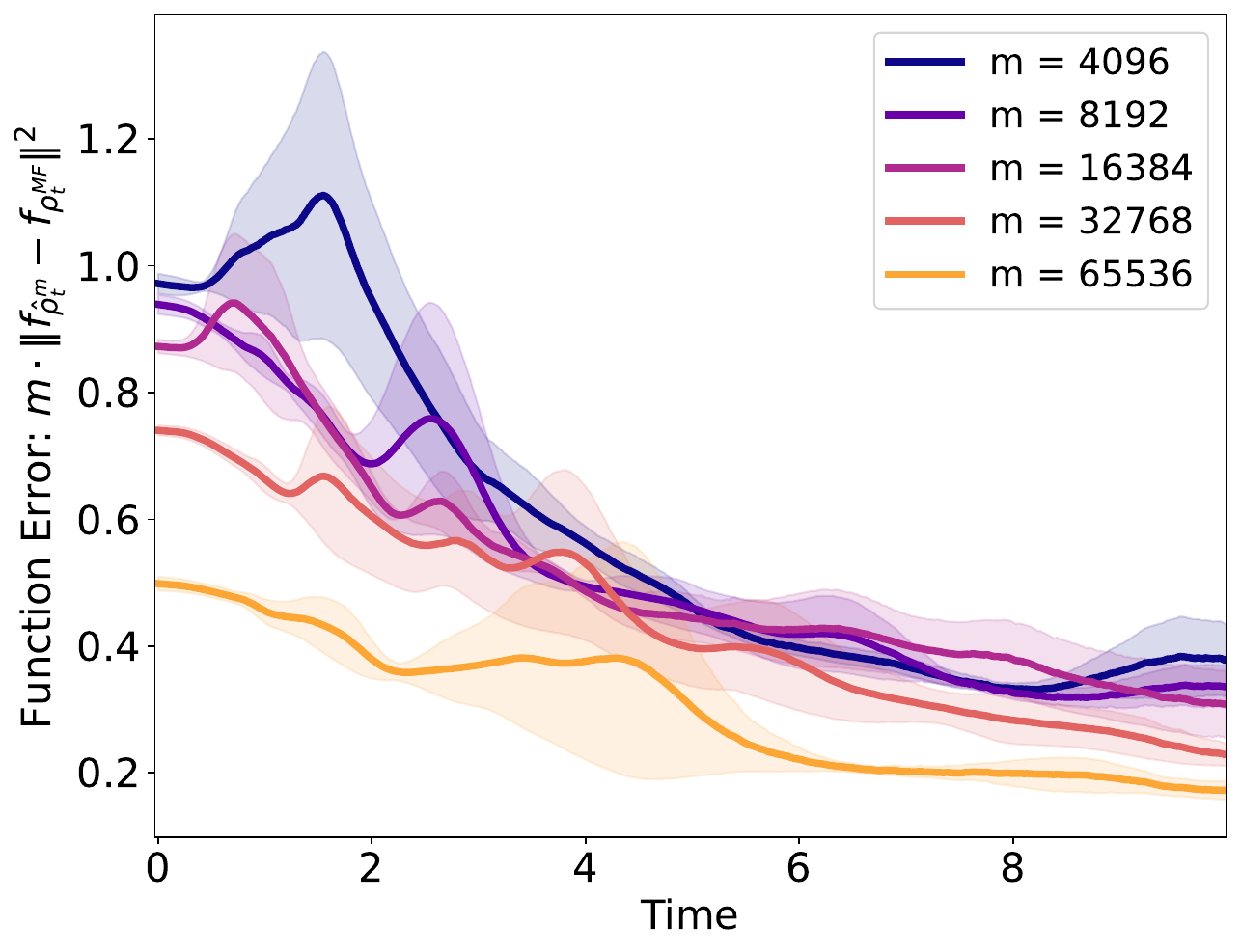}}  \\ \vspace{-1mm}
\small (b) scaled function error $m \|f_{\rtm} - f_{\rtmf}\|^2$.  
\end{minipage}%
\begin{minipage}{0.32\linewidth}
\centering
{\includegraphics[width=0.96\linewidth]{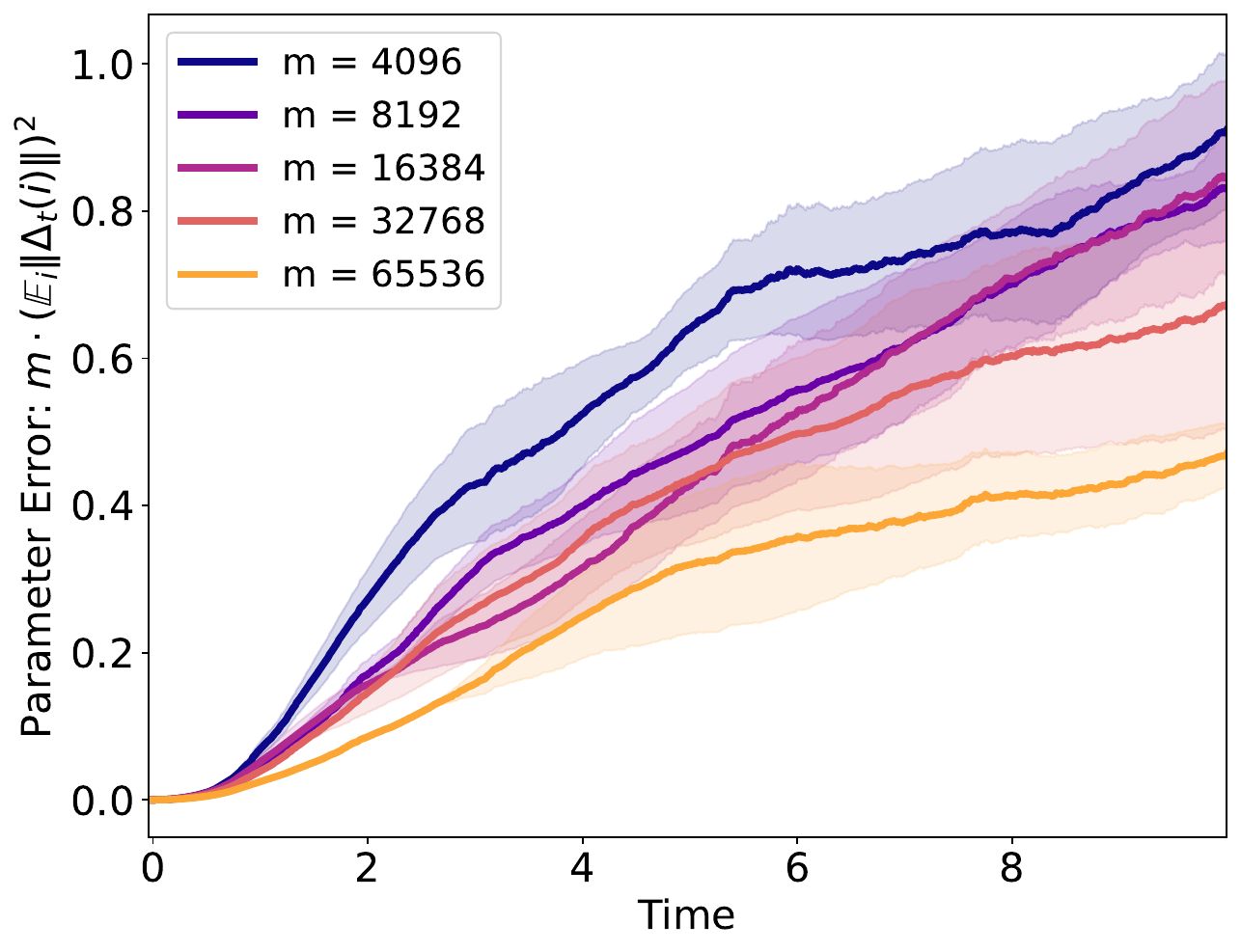}}  \\ \vspace{-1mm}
\small (c) scaled parameter coupling error $m (\mathbb{E}_i \|\hdit\|)^2$.  
\end{minipage}%
\vspace{-1mm}
\caption{Well-specified single-index (\hefour) target function $f^*(x) = \text{He}_4(x^\top w^*), x\sim\mathcal{N}(0,I_d)$, and $\sigma = \text{He}_4$. We set $d=32$ and learning rate $\eta=0.01$. 
    }
    \label{fig:body:hefour}
\end{figure}

\begin{figure}[t]
\begin{minipage}{0.32\linewidth}
\centering
{\includegraphics[width=0.96\linewidth]{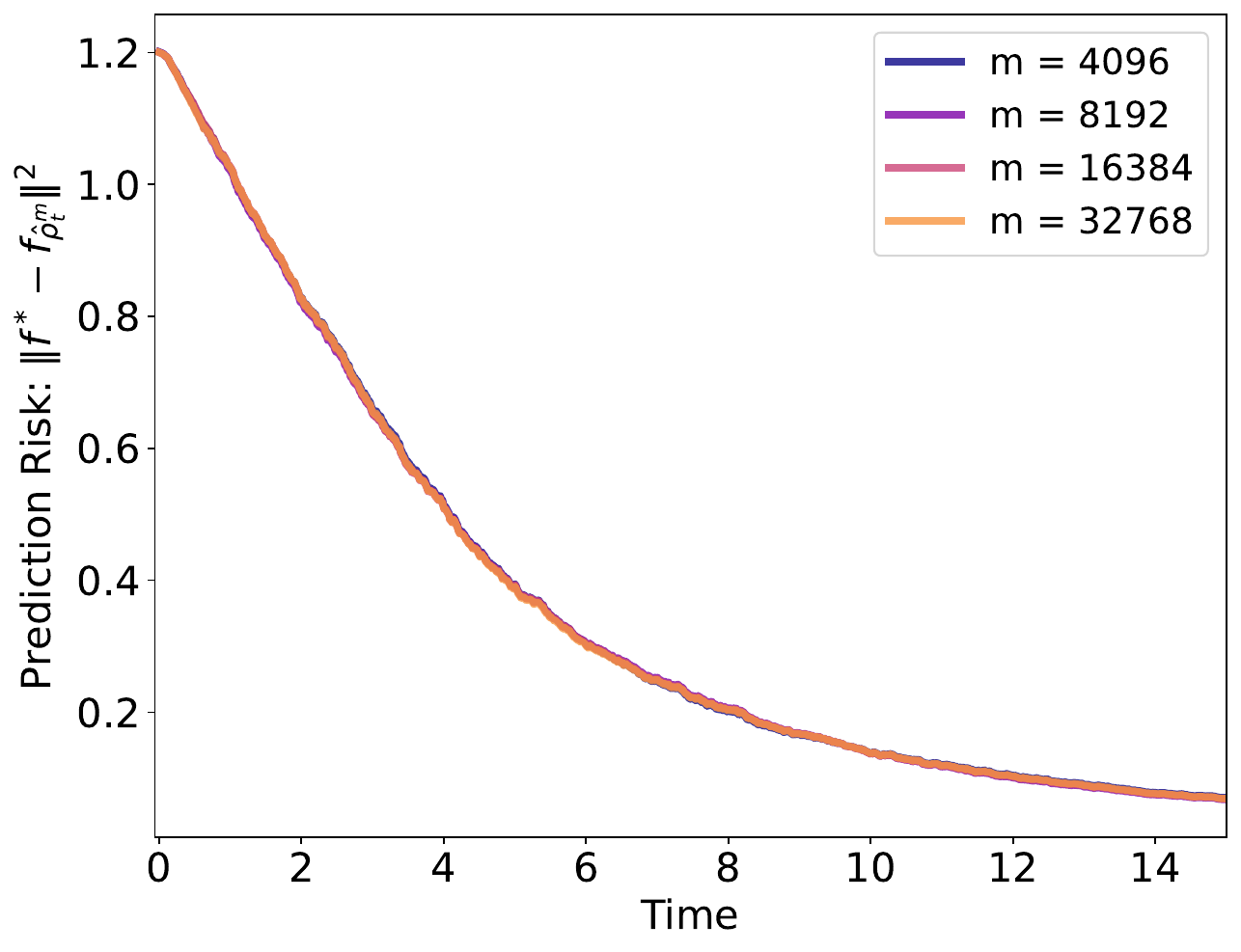}}  \\ \vspace{-1mm}
\small (a) prediction risk \\ $\|f_{\rtm} - f^*\|^2$.  
\end{minipage}%
\begin{minipage}{0.32\linewidth}
\centering
{\includegraphics[width=0.96\linewidth]{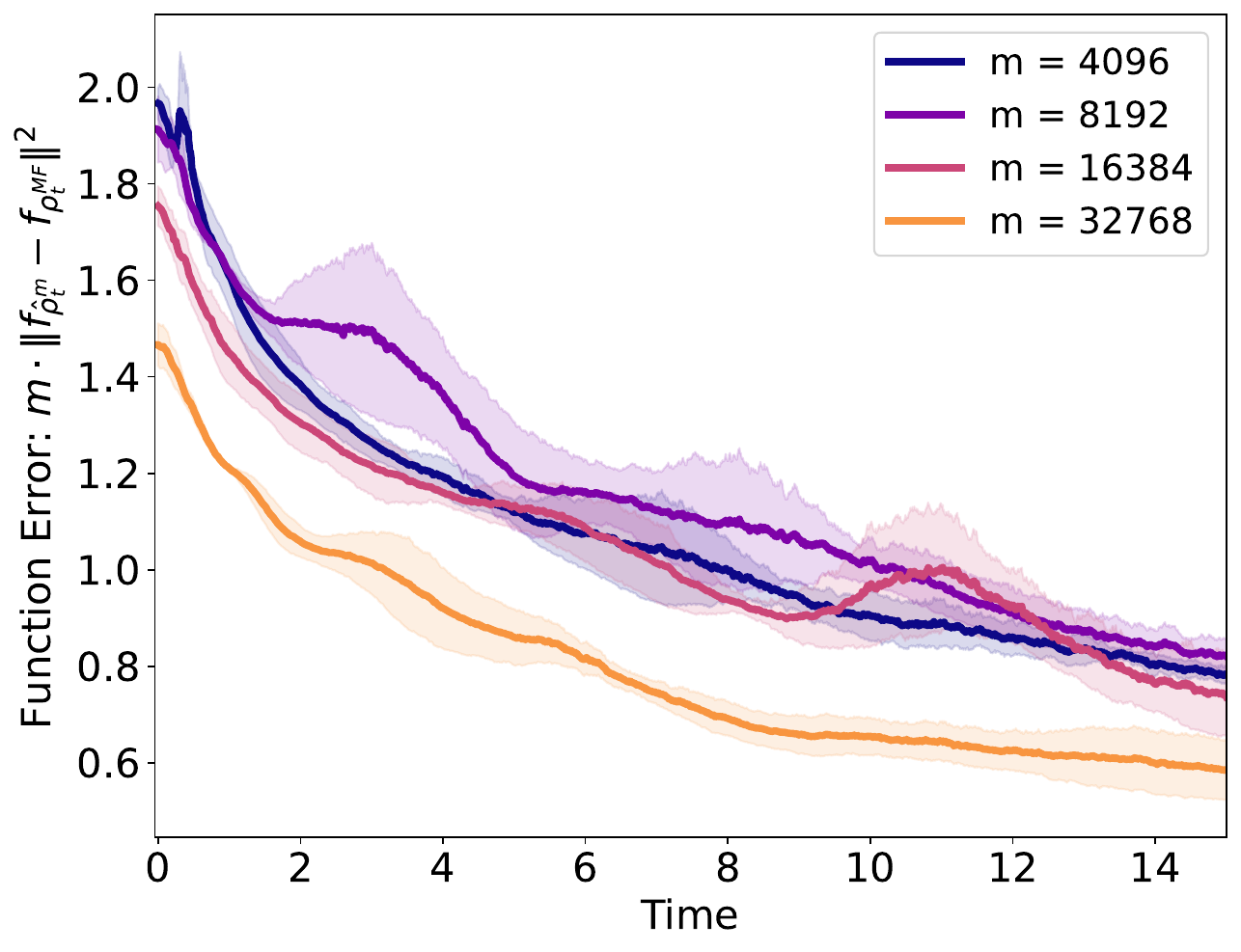}}  \\ \vspace{-1mm}
\small (b) scaled function error $m \|f_{\rtm} - f_{\rtmf}\|^2$.  
\end{minipage}%
\begin{minipage}{0.32\linewidth}
\centering
{\includegraphics[width=0.96\linewidth]{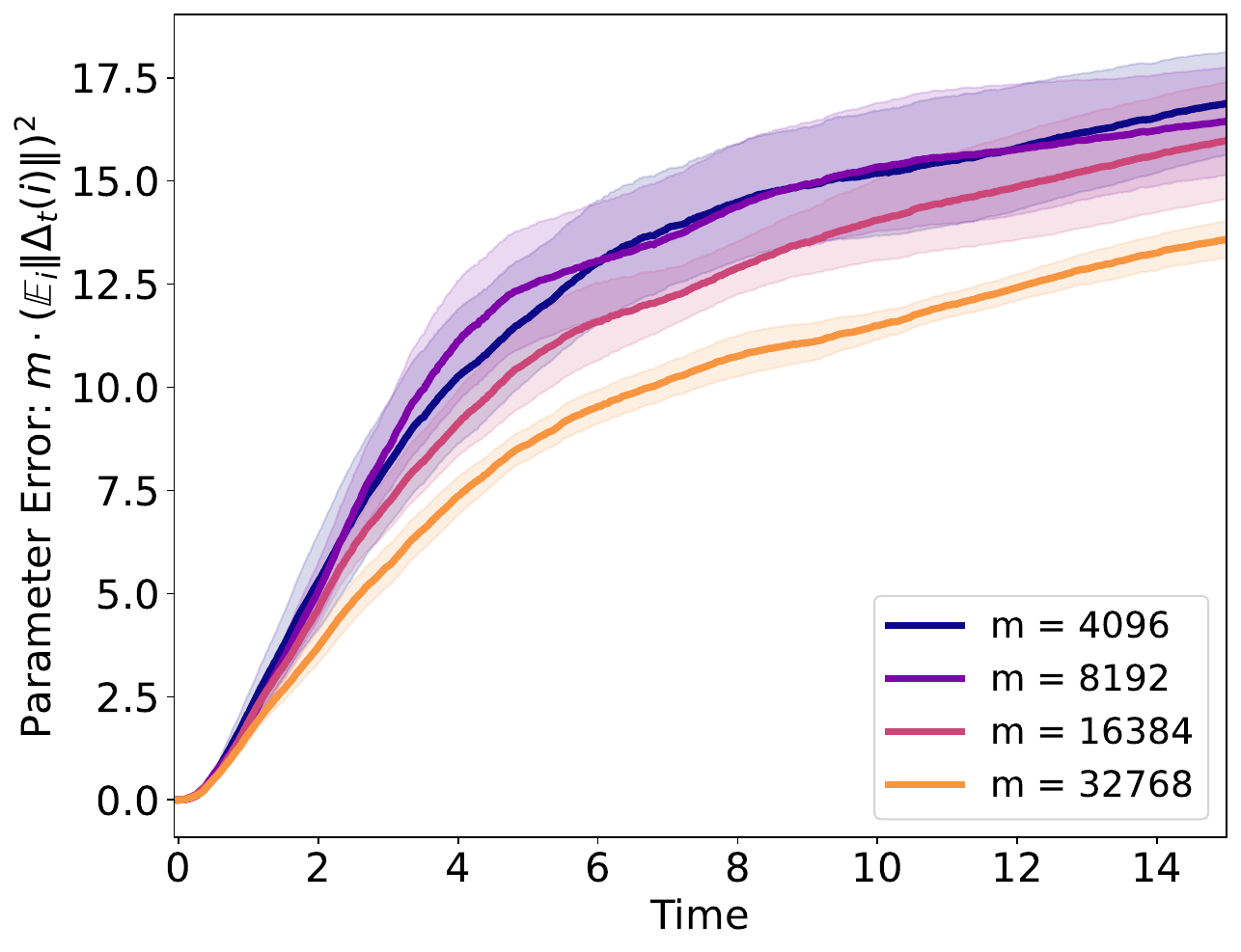}}  \\ \vspace{-1mm}
\small (c)  scaled parameter coupling error $m (\mathbb{E}_i \|\hdit\|)^2$. 
\end{minipage}%
\vspace{-1mm}
\caption{Misspecified single-index (\mis) target function $f^*(x) = 0.8 \text{He}_4(x^\top w^*) + 0.6\text{He}_6(x^\top w^*), x\sim\mathcal{N}(0,I_d)$, and $\sigma = \text{He}_4 + \text{He}_6$. We set $d=32$ and learning rate $\eta=0.01$. 
    }
    \label{fig:body:mis}
\end{figure}

\begin{figure}[!htb]
\begin{minipage}{0.32\linewidth}
\centering
{\includegraphics[width=0.96\linewidth]{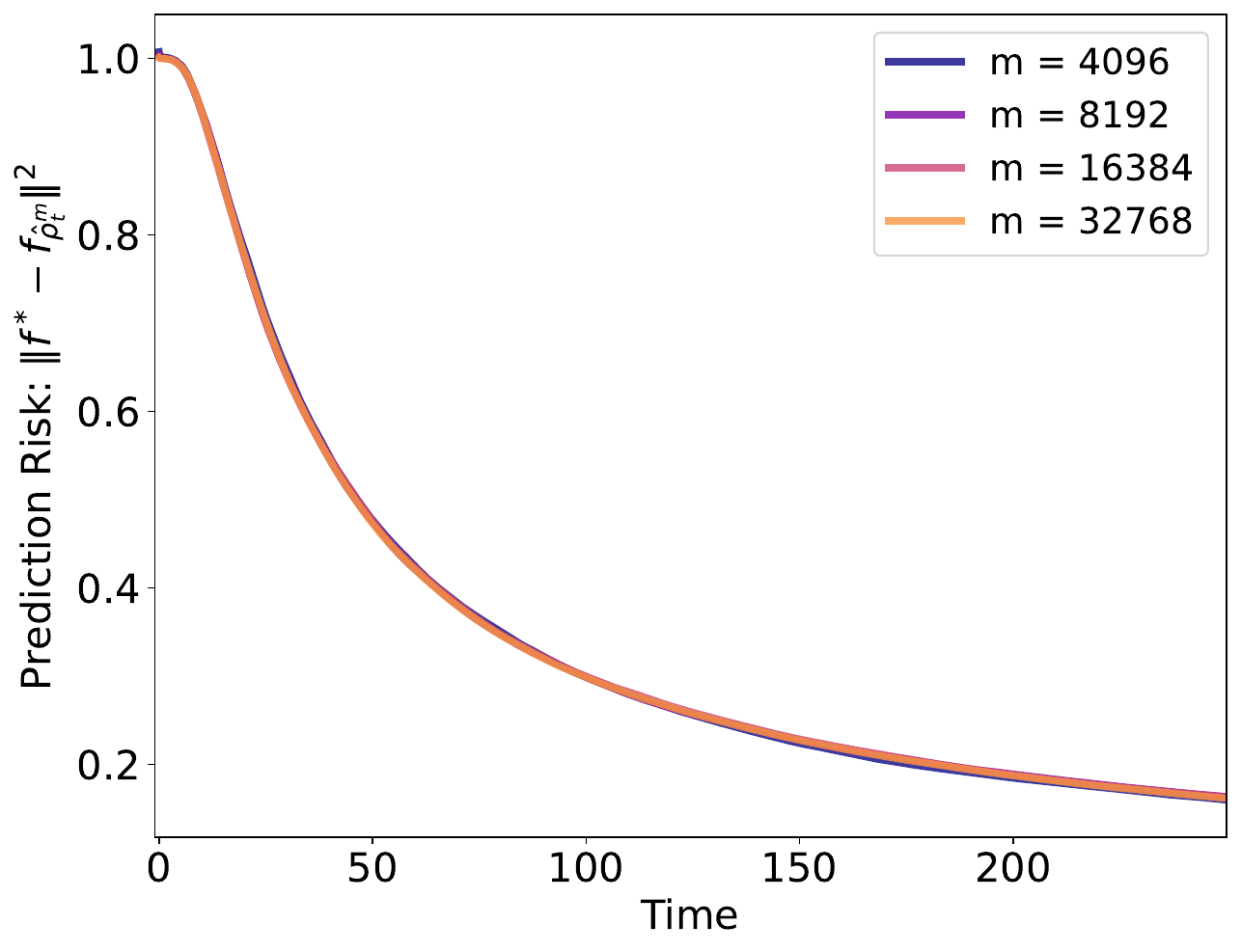}}  \\ \vspace{-1mm}
\small (a) prediction risk \\ $\|f_{\rtm} - f^*\|^2$.
\end{minipage}%
\begin{minipage}{0.32\linewidth}
\centering
{\includegraphics[width=0.96\linewidth]{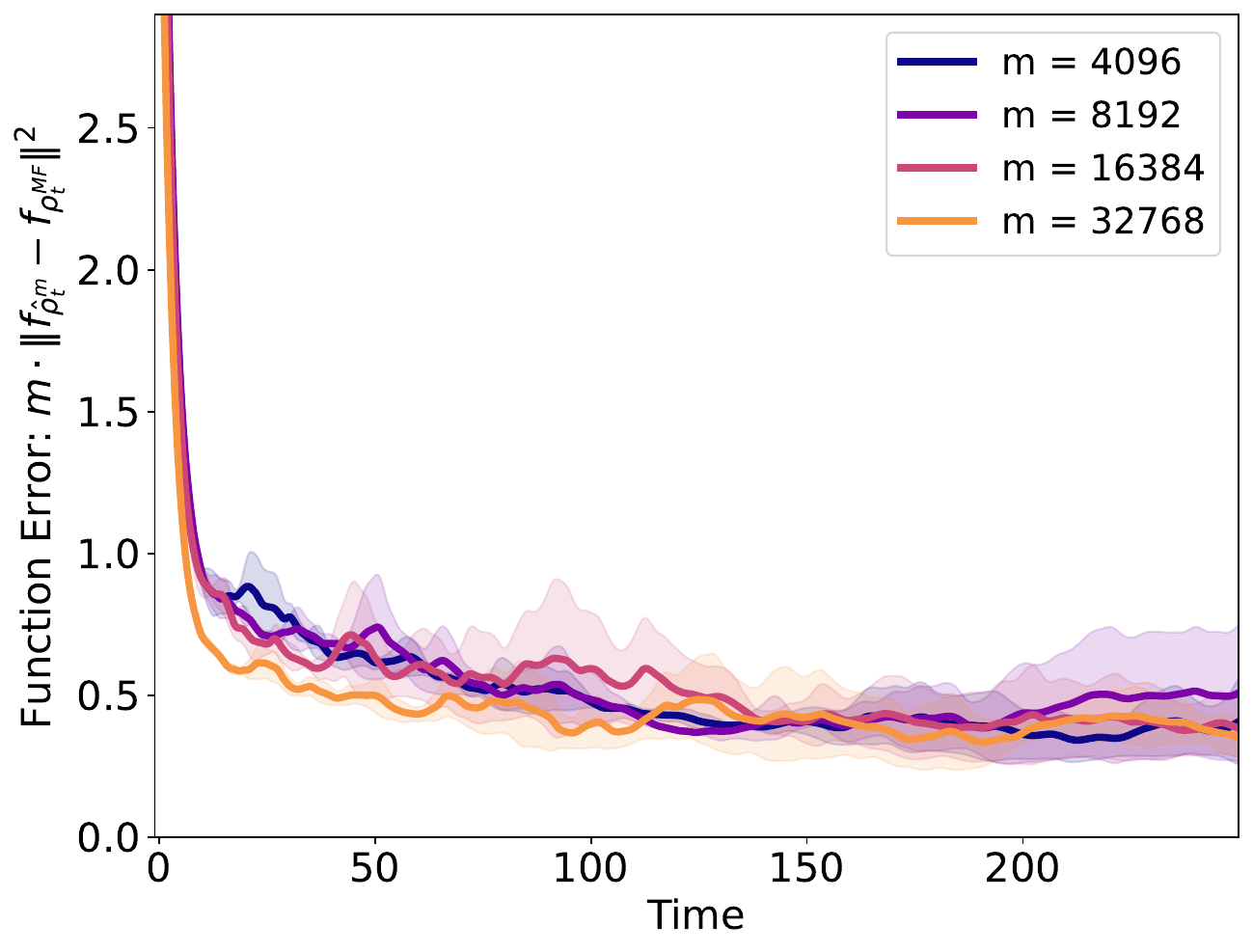}}  \\ \vspace{-1mm}
\small (b) scaled function error $m \|f_{\rtm} - f_{\rtmf}\|^2$.  
\end{minipage}%
\begin{minipage}{0.32\linewidth}
\centering
{\includegraphics[width=0.96\linewidth]{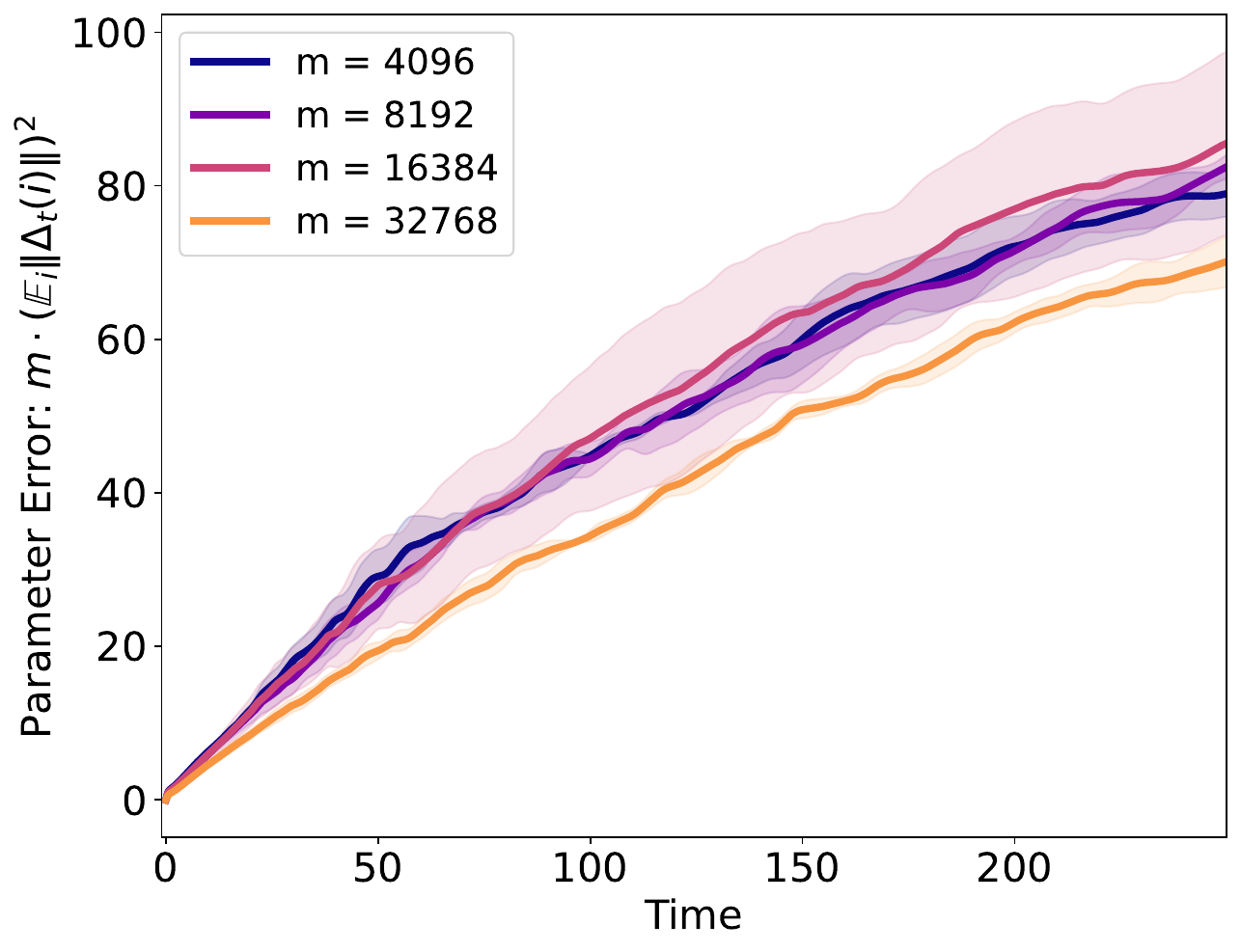}}  \\ \vspace{-1mm}
\small (c)  scaled parameter coupling error $m (\mathbb{E}_i \|\hdit\|)^2$. 
\end{minipage}%
\vspace{-1mm}
\caption{$4$-parity (\xorfour) target function $f^*(x) = \prod_{j\le 4} [x]_j, [x]_i\sim\text{Unif}\{1,-1\}$, and $\sigma = \text{SoftPlus}$ with temperature 16. We set $d=32$ and learning rate $\eta=0.05$. 
    } 
    \label{fig:body:xor}
\end{figure}

\begin{remark}
We note that our experiments are technically insufficient to guarantee that the PoC rate is polynomial in $d, T$, because we did not conduct extensive comparisons of the decay rate across growing values of $d, T$. However, in all of the experiments we plotted, we observed linear decay in function error $\| f_{\rtm} - f_{\rtmf}\|^2 \lesssim m^{-1}$ starting even at the smallest value of $m = 2^{12}$, which suggests that if the parameter-coupling PoC rate is some $g(d, t)/m$, then for all of experiments, $g(d, t) \leq t$ for the value of $d$ and range of $t$ we simulated.
Thus, we conjecture that in all these problems there is PoC in function error and in fixed parameter-coupling error at a rate at most $\on{poly}(d, t)/{m}$.  
\end{remark}

%% file: Body/conclusion.tex
\section{Conclusion}\label{sec:conclusion}
We studied propagation of chaos in the context of gradient-based training of shallow neural networks. By leveraging several key geometric 
assumptions of the optimization landscape, we established non-asymptotic 
guarantees of finite-width dynamics with polynomial dependency in all relevant parameters. 
At the heart of our technical contributions is a tailored potential function that balances the intricate interactions that arise between particle fluctuations around their idealized mean-field evolution. In essence, our assumptions exploit a form of self-concordance in the instantaneous potentials, as well as symmetries in the minimizing mean-field measure. While these assumptions rule out generic interaction particle systems, they crucially capture several problems of interest, such as planted models including single-index target functions. An enticing future direction is remove the local strong convexity assumptions to extend to the case when $\rho^*$ is a manifold; among other settings, this captures the learning a misspecified SIM. Another interesting direction is to go beyond the Monte Carlo scale of fluctuations, which has been established asymptotically under certain conditions \citep{chen2020dynamical,pham2021limiting}.

%% file: Appendices/app_related.tex
\section{Additional Related Works}\label{apx:related}

\paragraph{Mean-field analysis of shallow neural networks.}
The mean-field analysis views the training of two-layer neural network \eqref{eq:two-layer-nn} as an interacting particle system, and studies the evolution of the distribution of particles via the mean-field PDE \cite{nitanda2017stochastic,chizat2018global,mei2018mean,rotskoff2018neural,sirignano2020mean}. While most optimization guarantees for mean-field neural networks are \textit{qualitative} in nature, quantitative convergence rate can also be established under additional structural assumptions on the learning problem \cite{javanmard2019analysis,chizat2021sparse,chen2020dynamical,chen2022feature} or modification of the training dynamics \cite{rotskoff2019global,wei2019regularization,nitanda2022convex,chizat2022mean}. 

Recent works have studied the statistical efficiency of mean-field neural networks in learning low-dimensional target functions including multi-index models and $k$-parity. These existing analyses can be divided into two approaches: $(i)$ simplify the mean-field PDE using the symmetry and low-dimensional structure, and study the \textit{dimensional-free} dynamics \cite{hajjar2022symmetries,arnaboldi2023high} at short timescale $T=\tilde{O}_d(1)$ \cite{abbe2022merged,mahankali2023beyond,joshi2024complexity}; $(ii)$ directly characterize the converged solution using global optimality conditions \cite{wei2019regularization,telgarsky2023feature,suzuki2023feature,mousavi2024learning}. While the latter approach establishes a much larger learnable function class (e.g., see \cite{bach2017breaking}), the computational complexity is exponential in the (intrinsic) dimensionality of the problem. 

\paragraph{Gradient-based learning of single/multi-index models.}  
Outside of the mean-field regime, feature learning in neural networks has also been studied in a ``narrow-width'' setting, where neurons evolve (almost) independently and align with the low-dimensional target function during gradient-based training. Prior analyses in this regime mostly considered target functions that depends on $k=O_d(1)$ directions of the input, such as single-index models \cite{benarous2021online,ba2022high,bietti2022learning,mousavi2023neural,damian2023smoothing,dandi2024benefits,lee2024neural,arnaboldi2024repetita} and multi-index models \cite{damian2022neural,abbe2022merged,abbe2023sgd,dandi2023learning,bietti2023learning,collins2023hitting,glasgow2023sgd,arous2024high}. For the ``rank-extensive'' setting $k\gg 1$, recent works have investigated the additive setting where the target function is a sum of $k$ orthogonal single-index models \cite{li2020learning,oko2024learning,ren2024learning,simsek2024learning}. 

%% file: Appendices/apx_dynamics.tex
\section{Proofs of Lemmas from Basic Setup }\label{apx:dynamics}
\subsection{Notations}
Throughout this section, we will use the following notation, which builds upon the notation in our setup from the main body. \footnote{To emphasize the relationship with $\mathscr{f}$ and $\mathscr{k}$, we deviate from our standard notation convention here in using the lower-case letters $\mathscr{f}'$ and $\mathscr{k}'$ to denote vector-valued functions.}
\begin{align}
    \mathscr{f}(w) &:= \mathbb{E}_{x \sim \md_x} f^*(x)\sigma(w^\top x)\\ 
    \mathscr{f}'(w) &:= (I - ww^\top )\nabla_w \mathscr{f}(w)
\end{align}
and 
\begin{align}\label{eq:FKdef}
    \mathscr{k}(w, w') &:= \mathbb{E}_{x \sim \md_x} \sigma(w'^\top x)\sigma(w^\top x)\\
    \mathscr{k}'(w, w') &:= (I - ww^\top )\nabla_w \mathscr{k}(w, w').
\end{align}
In additional the interaction Hessian $\hpt$ introduced in the introduction, we also define a versions without the orthogonal projection, that is:
\begin{align}
    H_t(w, w') &:= \mathscr{k}'(\xi_t(w), \xi_t(w'))\\
    \hpt(w, w') &= H_t(w, w')(I - \xi_t(w')\xi_t(w'))
\end{align}
We also define the \em empirical local Hessian \em $\bar{D}_t$ (closely related to $\dpt$), where the expectation is taken over $\brt$ instead of $\rtmf$:
\begin{align}
    \bar{D}_t(w) &:= \nabla_{\xi_t(w)} \nu(\xi_t(w), \brt) = \nabla_{\xi_t(w)} \mathscr{f}'(\xi_t(w)) - \mathbb{E}_{w' \sim \brt}\nabla_{\xi_t(w)} \mathscr{k}'(\xi_t(w), w').\\
    \dpt(w) &= \nabla_{\xi_t(w)} \nu(\xi_t(w), \rtmf) = \nabla_{\xi_t(w)} \mathscr{f}'(\xi_t(w)) - \mathbb{E}_{w' \sim \rtmf}\nabla_{\xi_t(w)} \mathscr{k}'(\xi_t(w), w').
\end{align}
\subsection{Proof of Lemma~\ref{lemma:errdynamicsbody}}
We being with a basic lemma which uses the regularity of $\sigma$ to bound the smoothness of various problem parameters.
\begin{lemma}\label{lemma:smoothness}
Assume Assumption~\ref{assm:sigma} holds for the constant $\creg$. Then the following holds for any $w$ and $w'$ with norm at most $1$.
\begin{enumerate}[{\bfseries{S\arabic{enumi}}}]
    \item \label{S6}  $\left\|\nabla_w\mathscr{k}'(w, w')\right\| \leq \creg$ and $\left\|\nabla_w\mathscr{f}'(w)\right\| \leq \creg$ 
    \item\label{S1} $\left\|\nabla^2_{w'}\mathscr{k}'(w, w')\right\| \leq \creg$
    \item\label{S2} $\left\|\nabla^2_{w}\mathscr{k}'(w, w')\right\| \leq \creg$
    \item\label{S3} $\left\|\nabla_{w'}\nabla_w\mathscr{k}'(w, w')\right\| \leq \creg$
    \item\label{S4} $\left\|\nabla^2_{w} \mathscr{f}'(w)\right\|_{op} \leq \creg$
    \item\label{S5} For any distribution $\rho \in \Delta(\mathbb{S}^{d-1})$, we have $\left\|\nabla^2_{w} \nu(w, \rho)\right\|_{op} \leq \creg$
\end{enumerate}
\end{lemma}
\begin{proof}[Proof of Lemma~\ref{lemma:smoothness}]
These are straightforward to check from the definitions. First note that the operator norm of the first and second derivatives of $I - ww^\top $ is at most $2$. Thus for any vector-valued function $\xi(w)$, by chain rule, we have
\begin{align}
    \left\|\nabla_w (I - ww^\top )\xi(w)\right\| &\leq \left\|\nabla_w \xi(w)\right\| + 2\left\|\xi(w)\right\|~, \\
    \left\|\nabla^2_{w} (I - ww^\top )\xi(w)\right\| &\leq 3\left\|\nabla^2_{w} \xi(w)\right\| + 8\left\|\nabla_w \xi(w)\right\|.
\end{align}
So to prove the lemma, it suffices to bound (over all $w, w' \in \sd$):
\begin{align}
    \left\|\nabla_w \mathscr{f}(w)\right\|, \left\|\nabla^2_{w} \mathscr{f}(w)\right\|, \left\|\nabla^3_w \mathscr{f}(w)\right\| \leq \creg/11,
\end{align}
and 
\begin{align}
    \left\|\nabla_w \mathscr{k}(w, w')\right\|, \left\|\nabla^2_{w} \mathscr{k}(w, w')\right\|, \left\|\nabla^3_w \mathscr{k}(w, w')\right\|, \left\|\nabla_w \nabla_{w'} \nabla_w \mathscr{k}(w, w')\right\|,\left\|\nabla^2_w \nabla_{w'}\nabla_w \mathscr{k}(w, w')\right\| \leq \creg/11~.
\end{align}

As an example, for \ref{S1}, we have
\begin{align}
   \left\|\nabla^2_{w'}\mathscr{k}'(w, w')\right\|_{op} &\leq \sup_{v_2, v_2, v_3 \in \mathbb{S}^{d-1}}\mathbb{E}_x \sigma(w^\top x)\sigma'''(w'^\top x)v_1^\top (I - ww^\top )x(v_2^\top x)(v_3^\top x)\\
    &\leq \sup_{z, z' \in B_2^d}\left(\mathbb{E}_x |\sigma(z^\top x)|^5\right)^{1/5}\left(\mathbb{E}_x |\sigma'''(z'^\top x)|^5\right)^{1/5}\sup_{v \in \mathbb{S}^{d-1}}\left(\mathbb{E}_x|(v^\top x)|^5\right)^{3/5}\\
    &\leq \creg/11,
\end{align}
where here the second inequality holds by Holder's inequality, and the final inequality by Assumption~\ref{assm:sigma}. For \ref{S2}, the argument is the same as the previous one, except we use the product rule to account for the derivatives of $(I - ww^\top )$, which have operator norm at most $1$. 

For the rest of the terms involving derivatives --- up to third order --- of $K$, the argument is near identical, following from Holder's inequality and Assumption~\ref{assm:sigma}. Thus each of these terms above are bounded by $\creg/11$. 

For the terms involving $F$, as an example, let's expand the third order term. We have
\begin{align}
    \left\|\nabla^3_w \mathscr{f}(w)\right\| &\leq \sup_{v_1, v_2, v_3 \in \sd} \mathbb{E}_x |\sigma^{(3)}(w^\top x)(v_1^\top x)(v_2^\top x)(v_3^\top x)f^*(x)|\\
    &\leq \sup_{z, z' \in B_2^d}\left(\mathbb{E}_x |\sigma^{(3)}(z^\top x)|^5\right)^{1/5}\sup_{v \in \mathbb{S}^{d-1}}\left(\mathbb{E}_x|(v^\top x)|^5\right)^{3/5} \left(\mathbb{E}_x (f^*(x))^5\right)^{1/5}\\
    & \leq \creg/11.
\end{align}
It follows that all the terms in the lemma are bounded by $\creg$. 

\end{proof}
We also prove Lemma~\ref{fact} and Lemma~\ref{lemma:fparam} here, which we restate for the reader's convenience.
\factw*
\fparam*

\begin{proof}[Proof of Lemma~\ref{fact} and Lemma~\ref{lemma:fparam}]

First we decompose
    \begin{align*}
    \mathbb{E}_x (f_{\rtmf}(x) - f_{\rtm}(x))^2 &\leq 2\mathbb{E}_x (f_{\rtmf}(x) - f_{\brt}(x))^2 + 2\mathbb{E}_x (f_{\brt}(x) - f_{\rtm}(x))^2.
    \end{align*}
Now we can expand
\begin{align}
    \mathbb{E}_x &(f_{\brt}(x) - f_{\rtm}(x))^2\\
    &= \mathbb{E}_x \left(\mathbb{E}_i \sigma(\bwti^\top x) - \sigma((\bwti - \dit)^\top x)\right)^2 \\
    &= \mathbb{E}_x \left(\mathbb{E}_i \sigma'(\bwti^\top x)x^\top \dit  + \int_{s = 0}^1 \int_{s' = 0}^s (\sigma''((\bwti - s'\dit^\top x)(x^\top \dit)^2 ds' ds\right)^2.
\end{align}

Letting $\zeta(i, x) := \int_{s = 0}^1 \int_{s' = 0}^s \sigma''(\bwti - s'\dit^\top x) ds' ds$, we have
\begin{align}\label{eq:conc1}
     \mathbb{E}_x (f_{\brt}(x) - f_{\rtm}(x))^2 &\leq 2\mathbb{E}_x \left(\mathbb{E}_i \sigma'(\bwti^\top x)x^\top \dit\right)^2 + 2\mathbb{E}_x \left(\mathbb{E}_i (x^\top \dit)^2 \zeta(i, x) \right)^2,
\end{align}
and likewise,
\begin{align}\label{eq:conc2}
    \mathbb{E}_x \left(\mathbb{E}_i \sigma'(\bwti^\top x)x^\top \dit\right)^2 &\leq 2\mathbb{E}_x (f_{\brt}(x) - f_{\rtm}(x))^2 + 2\mathbb{E}_x \left(\mathbb{E}_i (x^\top \dit)^2 \zeta(i, x) \right)^2.
\end{align}

Let us bound the second term. We have 
\begin{align}
    \mathbb{E}_x &\left(\mathbb{E}_i (x^\top \dit)^2 \zeta(i, x) \right)^2 \\
    &= \mathbb{E}_i \mathbb{E}_j \mathbb{E}_x (x^\top \dit)^2 \zeta(i, x) (x^\top \djt)^2 \zeta(j, x) \\
    &\leq \mathbb{E}_i \mathbb{E}_j \left(\mathbb{E}_x ((x^\top \dit)^2)^4\right)^{1/4}\left(\mathbb{E}_x (\zeta(i, x))^4\right)^{1/4}\left(\mathbb{E}_x ((x^\top \djt)^2)^4\right)^{1/4}\left(\mathbb{E}_x (\zeta(j, x))^4\right)^{1/4}\\
    &= \left(\mathbb{E}_i \left(\mathbb{E}_x ((x^\top \dit)^2)^4\right)^{1/4}\left(\mathbb{E}_x (\zeta(i, x))^4\right)^{1/4}\right)^2 \\
    &\leq \left(\mathbb{E}_i \left(\mathbb{E}_x ((x^\top \dit)^8\right)^{1/4}\right)^2 \left(\max_i \left(\mathbb{E}_x (\zeta(i, x))^4\right)^{1/4}\right)^2\\
    &\leq O_{\creg}\left(\mathbb{E}_i \|\dit\|^2\right)^2 \left(\max_i \left(\mathbb{E}_x (\zeta(i, x))^4\right)^{1/4}\right)^2
\end{align}
Here the final line follows by the $\creg$-subgaussianity assumption on $x$ in \ref{assm:data}.

Now since for any $s' \in [0, 1]$, we have that $\|\bwti + s'\dit\| \leq 1$ (as it interpolates between two points on the sphere), we have by Assumption~\ref{assm:reg} that 
\begin{align}
    \mathbb{E}_x (\zeta(i, x))^4 \leq (\creg/11)^4.
\end{align}
Defining 
\begin{align}
    \zeta_s(i, x) = \int_{r = 0}^s \sigma''((\bwti + r\dit^\top x)dr,
\end{align}
and thus since $\|\bwti + r\dit\| \leq 1$ (as it interpolates between two points on the sphere), we have by Assumption~\ref{assm:reg} that 
\begin{align}
    \mathbb{E}_x (\zeta_s(i, x))^4 \leq (\creg/11)^4,
\end{align}
and thus
\begin{align}
    \mathbb{E}_x \left(\mathbb{E}_i (x^\top \dit)^2 \zeta(i, x) \right)^2 &\leq O_{\creg}\left(\mathbb{E}_i \|\dit\|^2\right)^2.
\end{align}
Returning to Equations~\eqref{eq:conc1} and \eqref{eq:conc2}, and observing that $\mathbb{E}_x \left(\mathbb{E}_i \sigma'(\bwti^\top x)x^\top \dit\right)^2 - \Delta_t^\top  \hpt \Delta_t \leq \creg\left(\mathbb{E}_i \|\dit\|^2\right)^2$ (to account for the projections orthogonal to $\bwti$ in $\hpt$; we omit the details), we have that
\begin{align}\label{eq:conc12}
     \mathbb{E}_x (f_{\brt}(x) - f_{\rtm}(x))^2 &\leq 2\Delta_t^\top  \hpt \Delta_t + O_{\creg}\left(\mathbb{E}_i \|\dit\|^2\right)^2\\
     &\leq O_{\creg}\left(\mathbb{E}_i \|\dit\|\right)^2,
\end{align}
and 
\begin{align}\label{eq:conc22}
    \Delta_t^\top  \hpt \Delta_t &\leq 2\mathbb{E}_x (f_{\brt}(x) - f_{\rtm}(x))^2 + O_{\creg}\left(\mathbb{E}_i \|\dit\|^2\right)^2.
\end{align}

It follows that
\begin{align}
    \mathbb{E}_x (f_{\brt}(x) - f_{\rtm}(x))^2 \leq O_{\creg}\left(\mathbb{E}_i \|\dit\|\right)^2.
\end{align}

We will use Chebychev's inequality to bound the first term $\mathbb{E}_x (f_{\rtmf}(x) - f_{\rtm}(x))^2$. We have
    \begin{align*}
        \mathbb{E}_{\rzm \sim \rho_0^{\otimes m}}\left(\mathbb{E}_x (f_{\rtmf}(x) - f_{\brt}(x))^2\right)^2 &\leq \mathbb{E}_{\rzm \sim \rho_0^{\otimes m}}\mathbb{E}_x (f_{\rtmf}(x) - f_{\brt}(x))^4\\
        &= \mathbb{E}_x \mathbb{E}_{\rzm \sim \rho_0^{\otimes m}} (f_{\rtmf}(x) - f_{\brt}(x))^4\\
        &\leq \mathbb{E}_x \frac{1}{m^3}\mathbb{E}_{w \sim \rtmf}(\sigma(w^\top x))^4 + \frac{O(m^2)}{m^4} \left(\mathbb{E}_{w \sim \rtmf}(\sigma(w^\top x))^2\right)^2\\
        &\leq O_{\creg}\left(\frac{1}{m^2}\right),
    \end{align*}
    where in the final inequality we used Assumption~\ref{assm:sigma}. By Chebychev, we have
    \begin{align*}
        \mathbb{P}_{\rzm \sim \rho_0^{\otimes m}}\left[\mathbb{E}_x (f_{\rtmf}(x) - f_{\brt}(x))^2 \geq \frac{\log(m)}{2m}\right] \leq o(1).
    \end{align*}

We thus conclude that with high probability, 
\begin{align}
\mathbb{E}_x (f_{\rtmf}(x) - f_{\rtm}(x))^2 \leq O_{\creg}(\mathbb{E}_i \|\dit\|)^2  + \frac{\log(m)}{m},
\end{align}
which yields Lemma~\ref{fact}.

For Lemma~\ref{lemma:fparam}, we have by \eqref{eq:conc22} that 
\begin{align}
    \Delta_t^\top  \hpt \Delta_t &\leq 2\mathbb{E}_x (f_{\brt}(x) - f_{\rtm}(x))^2 + O_{\creg}\left(\mathbb{E}_i \|\dit\|^2\right)^2 \\
    &\leq 2\mathbb{E}_x (f_{\rtmf}(x) - f_{\rtm}(x))^2 + \frac{\log(m)}{m} + O_{\creg}\left(\mathbb{E}_i \|\dit\|^2\right)^2.
\end{align}
\end{proof}    

Finally, we prove Lemma~\ref{lemma:errdynamicsbody}, which we restate here.
\errdyn*
\begin{proof}[Proof of Lemma~\ref{lemma:errdynamicsbody}]
We first decompose $\frac{d}{dt}\dit$ into four terms:
\begin{align}
    \frac{d}{dt}\dit &= \nu(\bwti, \rtmf) - \nu_{\hat{\md}}(\hxiti, \rtm)\\
        &=  \left(\nu(\bwti, \rtmf) -  \nu(\bwti, \brt)\right) + \left(\nu(\bwti, \brt) - \nu(\bwti, \rtm)\right)\\
        &\qquad + \left(\nu(\bwti, \rtm) - \nu(\hxiti, \rtm)\right) + \left(\nu(\hxiti, \rtm) - \nu_{\hat{\md}}(\hxiti, \rtm)\right).
\end{align}
By Lemma~\ref{lemma:concentration} and Lemma~\ref{lemma:nconcentration}, we can bound the first and fourth terms respectively with high probability:
\begin{align}\label{eq:term1}
    \|\nu(\bwti, \rtmf) -  \nu(\bwti, \brt)\|_2 &\leq \eps_m.\\
    \| \nu(\hxiti, \rtm) - \nu_{\hat{\md}}(\hxiti, \rtm) \| &\leq \eps_n.
\end{align}
For the second term, we have
\begin{align}
    \nu(\bwti, \brt) - \nu(\bwti, \rtm) &= -\mathscr{f}'(\bwti) + \mathbb{E}_{w' \sim \brt}\mathscr{k}'(\bwti, w')\\
    &\qquad + \mathscr{f}'(\bwti) - \mathbb{E}_{w' \sim \rtm}\mathscr{k}'(\bwti, w') \\
    &= - \mathbb{E}_{j}\left(\mathscr{k}'(\bwti, \bwtj) - \mathscr{k}'(\bwti, \bwtj + \djt)\right)\\
    &= \mathbb{E}_{j \sim [m]}\left(H_t(i, j)\djt + \mathbf{v}_j\right),
\end{align}
where $\|\mathbf{v}_j\| \leq \creg\|\djt\|^2.$
Indeed we can plug Lemma~\ref{lemma:smoothness} \ref{S1} into the Lagrange error bound
\begin{align}
    \|\mathscr{k}'(w, w') - \mathscr{k}'(w, w' + \Delta) - \nabla_{w'}\mathscr{k}'(w, w')\Delta\|
    &\leq \|\Delta\|^2 \sup_{w': \|w'\| \leq 1} \left\|\nabla^2_{w'}\mathscr{k}'(w, w')\right\|.
\end{align}
Now note that for any $j$, since both $\bwtj$ and $w_t^{(j)}$ are on $\mathbb{S}^{d-1}$, we have that 
\begin{align}\label{eq:perp}
|\langle{\bwtj\djt}\rangle| = \frac{1}{2}\|\djt\|^2,
\end{align}
and so by \ref{S6},
\begin{align}
    H_t(i, j)\djt = H^{\perp}_t(i, j)\djt + \mathbf{v}'_j
\end{align}
where $\|\mathbf{v}'_j\|_2 \leq \frac{1}{2}\creg\|\djt\|^2$
Summarizing, we have that 
\begin{align}\label{eq:term2}
    \nu(\bwti, \brt) - \nu(\bwti, \rtm)
    &= \mathbb{E}_{j \sim [m]}\left(H^{\perp}_t(i, j)\djt + \frac{3}{2}\mathbf{v}_j\right).
\end{align}

Finally for the third term, we have
\begin{align}
    \nu(\bwti, \rtm) - \nu(\hxiti, \rtm) &= -\nabla_w \nu(w, \rtm)|_{w = \bwti}\dit  + \mathbf{v},
\end{align}
where by \ref{S5},
\begin{align}
    \|\mathbf{v}\| \leq \|\dit\|^2 \left\|\nabla^2_{w} \nu(w, \rtm)\right\|_{op} \leq \creg\|\dit\|^2
\end{align}
Recall that we have defined
\begin{align}
    \bar{D}_t(w) := \nabla_{\xi_t(w)} \nu(\xi_t(w), \brt) = \nabla_{\xi_t(w)} \mathscr{f}'(\xi_t(w)) - \mathbb{E}_{w' \sim \brt}\nabla_{\xi_t(w)} \mathscr{k}'(\xi_t(w), w').
\end{align}
Now
\begin{align}
    \nabla_{\bwti} \nu(\bwti, \rtm) &= \nabla_{\bwti} \mathscr{f}'(\bwti) - \mathbb{E}_{j}\nabla_{\bwti} \mathscr{k}'(\bwti, \hat{\xi}_t(w_j))\\
    &= \nabla_{\bwti} \mathscr{f}'(\bwti) - \mathbb{E}_{j}\nabla_{\bwti} (\mathscr{k}'(\bwti, \bwtj) + \mathbf{M}_{j})\\
    &= \bar{D}_t(i) - \mathbb{E}_j\mathbf{M}_{j}.
\end{align}
where by \ref{S3},
\begin{align}
    \|\mathbf{M}_j\|_{op} \leq \|\djt\|\sup_{w, w'} \left\|\nabla_w\nabla_{w'}\mathscr{k}'(w, w')\right\|_{op} \leq \creg\|\djt\|.
\end{align}
Thus, additionally using the fact that we have conditioned on the fact that $\|D_t(i) - \bar{D}_t(i)\| \leq \eps_m$ --- and thus $\|D^{\perp}_t(i) - \bar{D}^{\perp}_t(i)\| \leq \eps_m$ --- and again using \eqref{eq:perp} and \ref{S6} to swap $D_t(i)\dit$ for $D^{\perp}_t(i)\dit$ with an error term of magnitude $0.5\creg\|\dit\|^2$, we have that 
\begin{align}\label{eq:term3}
   \nu(\bwti, \rtm) - \nu(\hxiti, \rtm) = D^{\perp}_t(i)\dit + \mathbf{v}_3, 
\end{align}
where $\|\mathbf{v}_3\| \leq  \creg(1.5\|\dit\|^2 + \|\dit\|\mathbb{E}_j\|\djt\|) + \eps_m \|\dit\|.$

Putting together Equations~\eqref{eq:term1}, \eqref{eq:term2}, and \eqref{eq:term3}, we have
\begin{align}
    \frac{d}{dt}\dit  &= D^{\perp}_t(i) \dit - \mathbb{E}_{j \sim [m], j \neq i}H^{\perp}_t(i, j) \djt + {\bm{\epsilon}},
\end{align}
    where 
    \begin{align}
        \|{\bm{\eps}}\| &\leq \eps_n + \eps_m(1 + \|\dit\|)  + \creg\left(1.5\|\dit\|^2 + \|\dit\|\mathbb{E}_j\|\djt\| + 1.5\mathbb{E}_j\|\Delta_j\|^2\right)\\
        &\leq \eps_n + \eps_m(1 + \|\dit\|) + 2\creg\left(\|\dit\|^2 + \mathbb{E}_j\|\Delta_j\|^2\right).
    \end{align}
\end{proof}

%% file: Appendices/apx_conc.tex
\section{Proof of Concentration Lemmas}\label{apx:conc_proofs}
\begin{lemma}[Uniformly Bounded Sampling Error]\label{lemma:concentration}
    With probability $1 - o(1)$ over the initialization, for all $t \leq T$ and $i \in [m]$, the following holds with $\eps_m = \frac{d^{3/2}\log(Tm)}{\sqrt{m}}$. 
    \begin{align*}
        \|\nu(\bwti, \rtmf) - \nu(\bwti, \brt)\| &\leq \eps_m.\\
        \|D_t(i) - \bar{D}_t(i)\| &\leq \eps_m.
    \end{align*}
\end{lemma}

\begin{proof}[Proof of Lemma~\ref{lemma:concentration}]
Fix $t \leq T$ and $w \in \sd$ By Equation~\eqref{eq:Vexpansion}, we have that
\begin{align}
    \nu(w, \rtmf) - \nu(w, \brt) &:= (I - ww^\top )\left( \mathbb{E}_{w' \sim \rtmf}\nabla_w K(w, w') - \mathbb{E}_{w' \sim \brt} \nabla_w K(w, w')\right)
\end{align}
Now 
\begin{align}
    \mathbb{E}_{w_0 \sim \rho_t^0}\mathbb{E}_{w' \sim \rtmf}\nabla_w K(w, w') = \mathbb{E}_{w' \sim \rtmf}\mathbb{E}_{w' \sim \rtmf}\nabla_w K(w, w'),
\end{align}
and by Assumption~\ref{assm:sigma}, for any $w', w \in \sd$, $\|\nabla_w K(w, w')\|_{\infty} \leq \creg$. So by Hoeffding’s inequality, taking a union bound over all $d$ coordinates in the random vector, we have
\begin{align}
    \mathbb{P}\left[\|\nu(w, \rtmf) - \nu(w, \brt)\|\geq \frac{\eps_m}{2}\right] \leq 2d\exp\left(-\frac{\Omega(m \eps_m^2)}{4d\creg^2}\right)
\end{align}
Now we need to take a union bound over all $w \in \sd$, and $t \leq T$. Create an net over $\sd$ of maximum distance $\frac{\eps_m}{4\creg}$ between any point and the net: this has size $O\left(\left(\frac{4\creg}{\eps_m}\right)^d\right)$. Similarly make a net over $[0, T]$ of spacing $\frac{\eps_m}{4\creg}$; this has size $\frac{4\creg T}{\eps_m}$. By a union bound, with probability at least 
\begin{align}\label{eq:prob}
    1 - 2d\exp\left(-\frac{\Omega(m \eps_m^2)}{4d\creg^2}\right)O\left(\left(\frac{4\creg}{\eps_m}\right)^d\right)\frac{4\creg T}{\eps_m},
\end{align}
for any $w$ in the net and any $t$ in the net,
we have
\begin{align}
    \|\nu(w, \rtmf) - \nu(w, \brt)\| \leq \frac{\eps_m}{3\creg}.
\end{align}

For any $w, u \in \sd$, for any $\rho$, by Lemma~\ref{lemma:smoothness}, we have
\begin{align}
    \nu(w, \rho) - \nu(u, \rho) \leq \creg\|w - u\|.
\end{align}
Similarly, by Lemma~{\ref{lemma:smoothness}}, for any $s, t \leq T$, and any $w_0$, we have
\begin{align}
    \|\xi_t(w_0) - \xi_s(w_0)\| \leq \creg|t - s|.
\end{align}
Thus, for any $w \in \sd$ and $t \leq T$, there exist $u$ and $s$ in the respective nets of distance at most $\frac{\eps_m}{3\creg}$. By a standard triangle inequality argument, we attain that with the probability in Equation~\ref{eq:prob}, for all $w \in \sd$ and $t \leq T$, we have 
\begin{align}
    \|\nu(w, \rtmf) - \nu(w, \brt)\| \leq \eps_m.
\end{align}
One can check that since $\eps_m \geq \frac{d\log(mT)}{\sqrt{m}}$, this probability is $1 - o(1)$.

The argument for proving concentration for $\bar{D}_t(w)$ uniformly over $w$ and $t$ is identical. The only change is that since $\bar{D}_t(w)$ is a $d \times d$ matrix, we need to take a union bound over $d^2$ indices in this matrix, so we require that $\eps_m \geq \frac{d^{3/2}\log(mT)}{\sqrt{m}}$.
\end{proof}

\begin{lemma}[Concentration  of $J_{t, s}$]\label{lemma:concJ}
With high probability over the random choice of $\brz$, for all $s \leq t \leq T$, all vectors $v \in \sd$, and all $j \in [m]$, we have
\begin{align}
\left| \mathbb{E}_i \|J_{t, s}(i) H_s^{\perp}(i, j)v\|\mathbf{1}(\bwti \in S) - \mathbb{E}_{w \sim \rho_0} \|J_{t, s}(w) H_s^{\perp}(w, \bar{w}_0(j))v\|
\mathbf{1}(\xi_t(w) \in S)\right | \leq \eps_m,
\end{align}
 for $\eps_m = \frac{\sqrt{d}\jmax\log(m T)}{\sqrt{m}}$.
\end{lemma}
\begin{proof}[Proof of Lemma~\ref{lemma:concJ}]
Fix $w', v \in \sd$ and $s \leq t \leq T$. Let 
\begin{align}
    X(w) := \|J_{t, s}(w) H_s^{\perp}(w, w')v\|\mathbf{1}(\xi_t(w) \in S).
\end{align}
To prove the desired bound for $j$ we must bound $\left| \mathbb{E}_{w \sim \rzm} X(w) - \mathbb{E}_{w \sim \rho_0} X(w) \right|$ with high probability for $w' = \bar{w}_0(j)$.

By Lemma~\ref{lemma:smoothness}, we have $|X(w)| \leq \creg\jmax$.
By Hoeffding's inequality, we have
\begin{align}
    \mathbb{P}&\left[\left| \mathbb{E}_{w \sim \rzm} X(w) - \mathbb{E}_{w \sim \rho_0} X(w) \right| \geq \frac{\eps_m}{2}\right]  \leq 2\exp\left(\frac{\Omega(m\eps_m^2)}{4\creg^2\jmax^2}\right).
\end{align}
Now we need to build an $\epsilon$-net of scale $\frac{\eps_m}{6\creg}$ over $s, t \in [0, T]$, $w' \in \sd$, and $v \in \sd$. The product of the size of these nets is 
\begin{align}
\left(\frac{6T\creg}{\eps_m}\right)^2O\left(\left(\frac{6\creg}{\eps_m}\right)^{2d}\right)
\end{align}
Checking Lipschitzness of the various quantities as per the proof of Lemma~\ref{lemma:concentration}, and then using a union bound gives the desired result with high probability whenever $\eps_m \geq \frac{\sqrt{d}\jmax\log(m T)}{\sqrt{m}}$.
\end{proof}

\begin{lemma}\label{lemma:conctau}
    Fix a set $S \subseteq \sd$, any function $v : \sd \rightarrow B_2^d$. 
    With probability $1 - o(1/d)$ over the random choice of $\rzm$, for any $w \in \sd$, with $\epsmh = \frac{d\log(md)}{\sqrt{m}}$ we have
    \begin{align}
        \|\mathbb{E}_{w' \sim \rho_0} \hpit(w, w')v(w')\mathbf{1}(\xi_t(w') \in S) - \mathbb{E}_{w' \sim \rzm} \hpit(w, w')v(w')\mathbf{1}(\xi_t(w') \in S)\| &\leq \|v\|_{\infty}\epsmh\\
       \left|\mathbb{P}_{w' \sim \rho_0}[\xi_t(w') \in S] - \mathbb{P}_{w' \sim \rzm}[\xi_t(w') \in S]\right| &\leq \epsmh.
    \end{align}
    \end{lemma}
    \begin{proof}
The second statement is immediate by a Chernoff bound. 
For the first statement, the proof is similar to the other concentration lemmas. Fix $w$. Let 
\begin{align}
X(w') := \hpit(w, w')v(w')\mathbf{1}(\xi_t(w') \in S)
\end{align}
Since $\|\hpit(w, w')\| \preceq \creg I$ for all $w, w'$, we have the following bound:

By Hoeffding's inequality (unioning over all coordinates of $X(w')$), we have
\begin{align}
    \mathbb{P}&\left[\left\| \mathbb{E}_{w \sim \rzm} X(w') - \mathbb{E}_{w \sim \rho_0} X(w') \right\| \geq \frac{{\epsmh}}{2}\right]  \leq 2\exp\left(\frac{\Omega(m{\epsmh}^2)}{4\creg^2 d}\right).
\end{align}
We need to build an $\epsilon$-net of scale $\frac{\epsmh}{4\creg}$ over $w \in \sd$ since by Lemma~\ref{lemma:smoothness}, $X(w)$ is $\creg$-Lipschitz in $w$. This net has size $\left(\left(\frac{O(\creg)}{\eps_m}\right)^{d}\right)$. Thus with $\epsmh = \frac{d\log(m)}{\sqrt{m}}$, we have that with high probability, for all $w \in \sd$, the desired quantity is uniformly bounded.
\end{proof}

\begin{lemma}[Averaging Lemma]\label{lemma:averaging}
Suppose $\mq$ is $\cbal$-balanced, and the high probability event in Lemma~\ref{lemma:concJ}  holds for $S = B_{\tau}$. If Assumption~\ref{assm:growth} holds, then for any $s \leq t$,
\begin{align}
\mathbb{E}_{i} \| J_{t, s}(i) m_s(i)\| \mathbf{1}(\bwti \notin B_{\tau}) \leq (1 + \cbal)\left(\epsmj + \javg(\tau)\right)\Phi_{\mq}(s).
\end{align}
In particular,
\begin{align}
\mathbb{E}_{i}\|m_t(i)\|\mathbf{1}(\bwti \notin B_{\tau}) \leq (1 + \cbal)\left(\epsmj + \javg(\tau)\right)\Phi_{\mq}(t).
\end{align}
\end{lemma}
\begin{proof}
Recall that 
\begin{align}
m_t(i) = \mathbb{E}_j H_t^{\perp}(i, j)\djt.
\end{align}
Thus
\begin{align}
\|J_{t, s}(i) m_s(i)\| \leq \mathbb{E}_j \| J_{t, s}(i) H_s^{\perp}(i, j) \djs\|.
\end{align}
Now for any vector $v \in \mathbb{R}^d$, by Lemma~\ref{lemma:concJ} and Assumption~\ref{assm:growth}, we have that 
\begin{align}
\mathbb{E}_i \| J_{t, s}(i) H_s^{\perp}(i, j) v\| \leq \epsmj \|v\| + \javg(\tau)\|v\|,
\end{align}
and so 
\begin{align}
\mathbb{E}_{i} \| J_{t, s}(i) m_s(i)\| \mathbf{1}(\bwti \notin B_{\tau}) \leq \left(\epsmj + \javg(\tau) \right)\mathbb{E}_i \|\dis\| \leq \left(\epsmj+ \javg(\tau)\right)\Phi_{\mq}(s).
\end{align}
The second line of the lemma holds by plugging in $s = t$. This concludes the lemma.
\end{proof}
\begin{lemma}\label{lemma:nconcentration}
Suppose the empirical data distribution $\hat{D} = \sum_{i = 1}^n \delta_{(x_i, y_i)}$ satisfies Assumption~\ref{assm:data}. Then with high probability over the draw of $\hat{D}$, we have uniformly over all $w \in \sd$, and all $\rho \in \Delta(\sd)$, we have
\begin{align}
    \|\nu_{\hat{\md}}(w, \rho) - \nu(w, \rho)\| \leq \eps_n,
\end{align}
for $\eps_n = \frac{\sqrt{d}\log^2(n)}{\sqrt{n}}$.
\end{lemma}
\begin{proof}
The velocity is linear in $\rho$, so it suffices to prove that (additionally) uniformly over $w'$, we have
\begin{align}
    \|\nu_{\hat{\md}}(w, \delta_{w'}) - \nu(w, \delta_{w'})\| \leq \eps_n.
\end{align}
We expand
\begin{align}
    \nu_{\hat{\md}}(w, \delta_{w'}) = (I - ww^\top )\mathbb{E}_{x \sim \hat{\md}} (y - \sigma(w'^\top x))\sigma'(w^\top x)x;
\end{align}
it suffices to prove that with high probability, uniformly over $w' \in \sd$, and $v \in \sd$, we have 
\begin{align}
    \left|\mathbb{E}_{x \sim \hat{\md}} \sigma(w'^\top x)\sigma'(w^\top x)x^\top v -  \mathbb{E}_{x \sim \md}\sigma(w'^\top x)\sigma'(w^\top x)x^\top v\right| &\leq \eps_n \\
    \left|\mathbb{E}_{x \sim \hat{\md}} y\sigma'(w^\top x)x^\top v-  \mathbb{E}_{x \sim \md}y\sigma'(w^\top x)x^\top v\right| &\leq \eps_n
\end{align}
For a fixed $w, w', v$, since by Assumption~\ref{assm:sigma}, all the terms in side the expectations are $\creg$-subgaussian, this holds with probability $\exp(-n \eps_n^2/2\creg^2)$. We now take three epsilon-nets over $\sd$ (for $w$, $w'$ and $v$ respectively) at the scale $\frac{\eps_n}{6\creg}$. Note that Lemma~\ref{lemma:smoothness} implies these quantities are $\creg$-Lipschitz with regard to $w$, $w'$ or $v$. Since the epsilon nets have size $\left(O\left(\frac{\creg }{\eps_n}\right)\right)^d$, with $\eps_n = \frac{\sqrt{d}\log^2(n)}{\sqrt{n}}$, we see that 
\begin{align}
 \exp(-n \eps_n^2/2\creg^2)\left(O\left(\frac{\creg }{\eps_n}\right)\right)^{3d}= o(1).
\end{align}
\end{proof}

%% file: Appendices/apx_pot.tex
\section{Proof of Results Relating to Potential Function Analysis}\label{apx:balance}

\subsection{Notations}
For $g, h: \mathcal{X} \rightarrow \mathbb{R}^d$, and a set $S \subseteq \sd$ we will denote the dot product and conditional dot products
\begin{align}
    &\langle{g, h}\rangle = \mathbb{E}_{w \sim \rho_0} g(w)^\top h(w).\\
    &\langle{g, h}\rangle_S = \mathbb{E}_{w \sim \rho_0} g(w)^\top h(w)\mathbf{1}(w \in S).
\end{align}

For a kernel $H: (\sd)^2 \rightarrow \mathbb{R}^{d \times d}$, and two sets $S, T \subseteq \sd$, for $g, h: \sd \rightarrow \mathbb{R}^d$, we use the notation
\begin{align}
    \langle{g, h}\rangle_{H}^{S, T} := \mathbb{E}_{w, w' \sim \rho_0}g(w)^\top H(w, w')h(w')\mathbf{1}(w \in S, w' \in T).
\end{align}
If $S = T$ or $S = T = \sd$, we will abbreviate and use the notation $\langle{g, h}\rangle_{H}^{S}$ or $\langle{g, h}\rangle_{H}$ respectively. If the functions $g, h$ are only defined on $[m]$ (or respectively on $\supp{\rzm}$), then in all the inner products / quadratic forms above, the default distribution should be taken to be $\text{Uniform}([m])$ (resp. $\rzm$) instead of $\rho_0$. 

We will use $\nabla \Phi_{\mq}(t)$ (resp. $\nabla \Omega(t)$, $\nabla \phi_v(t)$, $\nabla \Psi_{\mq}(t)$.) to denote the map on $[m]$ (resp. $\supp{\rzm}$) which takes $i$ (or $w_i$) to $m\nabla_{\dit} \Phi(t)$. We have rescaled these derivative so that this term is on order $1$, so we can take inner products in our notation more easily.

For a set $B \subseteq \sd$, we will use the shorthand $B^t := \xi_t^{-1}(B)$ to denote the set of all $w \in \sd$ with $\xi_t(w) \in B$, and $\bar{B}$ to denote the complement of $B$ in $\sd$.

\subsection{Proof of Lemmas on the Properties of the Potential}

\subsubsection{Restricted Isometry and Related Group Theoretic Definitions and Lemmas} 

\begin{definition}\label{def:cri}
We say a problem $(H, \mu)$ has \em consistent restricted isometry (\cri) \em with a set $S$ if for any eigenfunction $v$ of $(H, \mu)$, (that is, where $\langle{u, v}\rangle_{H} = \lambda_v \langle{u, v}\rangle$ for all $u : \sd \rightarrow \mathbb{R}^d$), we have that for all $w \in \sd$, we have 
\begin{align}
    \mathbb{E}_{w' \sim \mu}H(w, w')v(w')\mathbf{1}(w' \in S) = \lambda_v v(w) \mathbb{P}_{w' \sim \rho_0}[w' \in S].
\end{align}
In other words, for any $u : \sd \rightarrow \mathbb{R}^d$,
\begin{align}
    \langle{u, v}\rangle_{H}^{S} = \lambda_v \langle{u, v}\rangle^{S} \mathbb{P}_{\rho_0}[S],
\end{align}
\end{definition}

\begin{definition}\label{def:transitive}
The \em automorphism group \em $\cg$ of a problem $(\rho^*, \rho_0, \md_x)$ is the set group of rotations $g$ on $\sd$ where for any $A \subset \sd$:
\begin{align*}
    \mathbb{P}_{\rho^*}[A] &= \mathbb{P}_{\rho^*}[g(A)]\\
    \mathbb{P}_{\md}[A] &= \mathbb{P}_{\md_x}[g(A)]\\
    \mathbb{P}_{\rho_0}[A] &= \mathbb{P}_{\rho_0}[g(A)]
\end{align*}
We say that a problem $(\rho^*, {\md_x}, \rho_0)$ is \em transitive \em if for any $w^*, {w^*}' \in \supp{\rho^*}$, there exists some $g$ in the automorphism group $\cg$ such that $g(w^*) = {w^*}' $.
\end{definition}
\begin{lemma}
Suppose \ref{def:transitiveI1} holds. For any time $t$, for all $g \in \cg$ in the automorphism group of $(\rho^*, \rho_0, {\md_x})$, we have
\begin{enumerate}[{\bfseries{A\arabic{enumi}}}]
    \item\label{A1} If $\xi_t(w) \in A$, then $\xi_t(g(w)) \in g(A)$
    \item\label{A2} Almost surely over $w \sim \rho_0$, $\xii(w) = \argmin_{w^* \in \supp{\rho^*}} \|w - w^*\|$. So a.s., for all $A \subset \sd, g \in \cg$, if $\xii(w) \in A$, then $\xii(g(w)) \in g(A)$. Further, $\xii_{\#}\rho_0 = \rho^*$.
    \item\label{A3} $g(B_{\tau}) = B_{\tau}$.
\end{enumerate}
\end{lemma}
\begin{proof} 
We will prove the first item by induction. It suffices to prove the following claim, because if the velocity is symmetric, then $\rtmf$ will be symmetric.
\begin{claim} 
Conditional on \ref{A1} holding up to time $t$, we have
\begin{align}
    \frac{d}{dt} \xi_t(w) = \nu(w, \rtmf) = g^{-1}(\nu(g(w), \rtmf))
\end{align}
\end{claim}
\begin{proof} 
\begin{align}
    \nu(w, \rtmf) &= -(I - ww^\top )\nabla_w F_{\md}(w) + (I - ww^\top )\nabla_w \mathbb{E}_{w' \sim \rtmf} K_{\md}(w, w')\\
    &= -(I - ww^\top ) \mathbb{E}_{x \sim {\md_x}} f^*(x)\sigma'(w^\top x)x + (I - ww^\top )\mathbb{E}_{w' \sim \rtmf} \mathbb{E}_{x \sim {\md_x}} \sigma(w'^\top x)\sigma'(w^\top x)x
\end{align}
Now 
\begin{align}
    P^{\perp}_w\mathbb{E}_{w' \sim \rtmf}& \mathbb{E}_{x \sim {\md_x}} \sigma(w'^\top x)\sigma'(w^\top x)x\\
    &=P^{\perp}_w\mathbb{E}_{w' \sim \rtmf} \mathbb{E}_{x \sim {\md_x}} \sigma(g(w')^\top g(x))\sigma'(g(w)^\top g(x))x\\
    &=P^{\perp}_w\mathbb{E}_{w' \sim \rtmf} \mathbb{E}_{x \sim {\md_x}} \sigma(g(w')^\top x)\sigma'(g(w)^\top x)g^{-1}(x)\\
    &=P^{\perp}_w\mathbb{E}_{w' \sim \rtmf} \mathbb{E}_{x \sim {\md_x}} \sigma(w'^\top x)\sigma'(g(w)^\top x)g^{-1}(x)\\
    &= (g^{-1}(x) - ww^\top g^{-1}(x))\mathbb{E}_{w' \sim \rtmf} \mathbb{E}_{x \sim {\md_x}} \sigma(w'^\top x)\sigma'(g(w)^\top x)\\
    &= (g^{-1}(x) - wg(w)^\top x)\mathbb{E}_{w' \sim \rtmf} \mathbb{E}_{x \sim {\md_x}} \sigma(w'^\top x)\sigma'(g(w)^\top x)\\
    &= g^{-1}(x - g(w)g(w)^\top x)\mathbb{E}_{w' \sim \rtmf} \mathbb{E}_{x \sim {\md_x}} \sigma(w'^\top x)\sigma'(g(w)^\top x)\\
    &= g^{-1}\left(P^{\perp}_{g(w)}\nabla_{g(w)} \mathbb{E}_{w' \sim \rtmf} K_{\md}(g(w), w')\right).
\end{align}
Here (1) is because $z^\top y$ = $z^\top U^\top Uy$ for any rotation $U$ and any $y, z \in \mathbb{R}^d$ (2) is because ${\md_x}$ is invariant with respect to $\cg$, (3) is because $\rtmf$ is invariant with respect to $\cg$ (by induction), (5) again because of the same reason as (1), and (4), (6) and (7) are simple algebraic operations.
Similarly, we can show that
\begin{align}
    P^{\perp}_w\mathbb{E}_{x \sim {\md_x}} f^*(x)\sigma'(w^\top x)x &= \mathbb{E}_{x \sim {\md_x}} f^*(x)\sigma'(w^\top x)P^{\perp}_w x\\
    &= \mathbb{E}_{x \sim {\md_x}} f^*(x)\sigma'(w^\top x)g^{-1}(P^{\perp}_{g(w)}) g(x)\\
    &= \mathbb{E}_{x \sim {\md_x}} f^*(x)\sigma'(g(w)^\top g(x))g^{-1}(P^{\perp}_{g(w)} g(x))\\
    &= \mathbb{E}_{x \sim {\md_x}} \mathbb{E}_{w^* \sim \rho^*} \sigma({w^*}^\top x)\sigma'(g(w)^\top g(x))g^{-1}(P^{\perp}_{g(w)} g(x))\\
    &= \mathbb{E}_{x \sim {\md_x}} \mathbb{E}_{w^* \sim \rho^*} \sigma({w^*}^\top g^{-1}(x))\sigma'(g(w)^\top x)g^{-1}(P^{\perp}_{g(w)}x)\\
    &= \mathbb{E}_{x \sim {\md_x}} \mathbb{E}_{w^* \sim \rho^*} \sigma(g({w^*})^\top x)\sigma'(g(w)^\top x)g^{-1}(P^{\perp}_{g(w)}x)\\
    &= \mathbb{E}_{x \sim {\md_x}} f^*(x)\sigma'(g(w)^\top x)g^{-1}(P^{\perp}_{g(w)}x)\\
    &= g^{-1}\left(P^{\perp}_{g(w)} \cf_{\md}(g(w))\right).
\end{align}
Putting these two computations together yields the desired conclusion,
\begin{align}
    \nu(w, \rtmf) = g^{-1}(\nu(g(w), \rtmf)).
\end{align}
\end{proof}

Next consider \ref{A2}. Observe that if $w$ is closest to some $w^*$, then it either is the case that $\xi_t(w^*)$ is always closest to $w^*$, or at some point there is a tie in the distances $\xi_t(w^*)$ and $\xi_t({w^*}')$. By \ref{A1}, such a tie would imply however that $\|w - w^*\| = \|w - {w^*}'\|$, which we have assumed in \ref{def:transitiveI1} is a measure $0$ event. The rest follows immediately from the transitivity of $\supp{\rho^*}$. 


Finally for \ref{A3}, 
\begin{align}
    g(B_{\tau}) &= \{g(w) : w \in B_{\tau}\}\\
    &= \{g(w) : \min_{w^* \in \supp{w^*}}\|w - w^*\|\leq \tau\}\\
    &= \{g(w) : \min_{w^* \in \supp{w^*}}\|g(w) - g(w^*)\|\leq \tau\}\\
    &= \{g(w) : \min_{w^* \in \supp{w^*}}\|g(w) - w^*\|\leq \tau\}\\
    &= \{w : \min_{w^* \in \supp{w^*}}\|w - w^*\|\leq \tau\}\\
    &= B_{\tau}.
\end{align}
\end{proof}

\begin{lemma}\label{lemma:cri}
Suppose $(\rho^*, {\md_x}, \rho_0)$ is transitive. Then $(\hpit, \rho_0)$ has consistent restricted isometry with $B_{\tau}^t = \xi_t^{-1}(B_{\tau})$ for any $t \leq T$, $\tau \geq 0$.
\end{lemma}
\begin{proof} We will use a series of small claims.
\begin{claim}\label{claim:cons}
Fix any $t$ and $\tau$. Let $\tilde{\rho}$ be the distribution of $\xii(w)$ with $w \sim \rho_0$ \em conditional \em on $\xi_{t}(w) \in B_{\tau}$. Then
\begin{align}
    \tilde{\rho} \sim {\xii}_{\#}\rho_0.
\end{align}
\end{claim}
\begin{proof}
We will show that both $\tilde{\rho}$ and $ {\xii}_{\#}\rho_0$ are uniform on the support of $\rho^*$. Fix $w^*, {w^*}' \in \supp{\rho^*}$, and let $g \in \cg$ be the element in the automorphism group of $(\rho^*, \rho_0, {\md_x})$ which takes $w^*$ to ${w^*}'$. Now
\begin{align}
    \tilde{\rho}(w^*) &= \mathbb{P}_{w \sim \rho_0}[\xii(w) = w^* \land \xi_t(w) \in B_{\tau}]\\
    &= \mathbb{P}_{w \sim \rho_0}[\xii(g(w)) = g(w^*) \land \xi_t(g(w)) \in g(B_{\tau})]\\
    &= \mathbb{P}_{w \sim \rho_0}[\xii(g(w)) = {w^*}'\land \xi_t(g(w)) \in B_{\tau}]\\
    &= \mathbb{P}_{w \sim \rho_0}[\xii(w) = {w^*}' \land \xi_t(w) \in B_{\tau}]\\
    &= \tilde{\rho}({w^*}').
\end{align}
Here (1) is by definition, (2) is by \ref{A1}, and \ref{A2}, (3) is by choice of $g$ and \ref{A3}, and (4) is by the symmetry of $\rho_0$ with respect to $\cg$. It follows that $\tilde{\rho}$ is uniform on the support of $\rho^*$. Now lets check that $ {\xii}_{\#}\rho_0$ is also uniform on $\supp{\rho^*}$. We have by similar use of \ref{A1} and \ref{A2} that 
\begin{align}
    {\xii}_{\#}\rho_0(w^*) &= \mathbb{P}_{w \sim \rho_0}[\xii(w) = w^* \land \|\xii(w), w^*\| \leq \|\xii(w), \tilde{w}^*\| \forall \tilde{w}^* \in \supp{\rho^*}]\\
    &= \mathbb{P}_{w \sim \rho_0}[\xii(g(w)) = g(w^*) \land \|\xii(g(w)), g(w^*)\| \leq \|\xii(g(w)), g(\tilde{w}^*)\| \forall \tilde{w}^* \in \supp{\rho^*}]\\
    &= \mathbb{P}_{w \sim \rho_0}[\xii(g(w)) = {w^*}' \land \|\xii(g(w)), {w^*}'\| \leq \|\xii(g(w)), \tilde{w}^*\| \forall \tilde{w}^* \in \supp{\rho^*}]\\
    &= \mathbb{P}_{w \sim \rho_0}[\xii(w) = {w^*}' \land \|\xii(w), {w^*}'\| \leq \|\xii(w), \tilde{w}^*\| \forall \tilde{w}^* \in \supp{\rho^*}]\\
    &= {\xii}_{\#}\rho_0({w^*}').
\end{align}
\end{proof}
\begin{claim}
Let $v$ be an eigenfunction of $(\hpit, \rho_0)$, that is $\langle{u, v}\rangle_{\hpit} = \lambda_v \langle{u, v}\rangle$ for all $u : \sd \rightarrow \mathbb{R}^d$. Then $v(w) = v'(\xii(w))$ for some function $v': \supp{\rimf} \rightarrow \sd$.
\end{claim}
\begin{proof}
For all $w$, we have
\begin{align}
    \lambda_v v(w) = \mathbb{E}_{w' \sim \rho_0}\hpit(w, w')v(w') = \mathbb{E}_{w' \sim \rho_0}K'(\xii(w), \xii(w'))v'(\xii(w')) .
\end{align}
This value only depends on $w$ through $\xii(w)$.
\end{proof}
We will now use the previous two claims to show consistency. Fix $t$ and $\tau$, and let $v$ be some eigenfunction of $(\hpit, \rho_0)$. Let  $v': \supp{\rimf} \rightarrow \sd$ be the function guaranteed by the previous claim with $v(w) = v'(\xii(w))$. Then for all $w$,
\begin{align*}
\mathbb{E}_{w' \sim \rho_0}& \hpit(w, w')v(w')\mathbf{1}(w' \in \xi_t^{-1}(B_{\tau}))\\
&=\mathbb{E}_{w' \sim \rho_0} \hpit(w, w')v(w')\mathbf{1}(\xi_t(w') \in B_{\tau})\\
&=\mathbb{E}_{w' \sim \rho_0} K'(\xii(w), \xii(w'))v'(\xii(w'))\mathbf{1}(\xi_t(w') \in B_{\tau})\\
&= \mathbb{P}_{\rho_0}[\xi_t^{-1}(B_{\tau})]\mathbb{E}_{w' \sim \rho_0}K'(\xii(w), \xii(w'))v'(\xii(w'))\\
    &= \mathbb{P}_{\rho_0}[\xi_t^{-1}(B_{\tau})]\mathbb{E}_{w \sim \rho_0} \hpit(w, w')v(w')\\
    &= \mathbb{P}_{\rho_0}[\xi_t^{-1}(B_{\tau})]\lambda_v v(w),
\end{align*}
as desired. Here the third equality follows from Claim~\ref{claim:cons}.
\end{proof}

\subsubsection{Construction of the Potential}
\label{sec:app_balancedbody}

\begin{remark}
We can verify that the action $\overline{\hpit}$ (from Section~\ref{sec:pfoverview:pot}) is well-defined in $\mathcal{Z}$ since 
$\| \overline{\hpit}v \|_{\mathcal{Z}} \leq \sup_{w,w'} \| \hpit(w, w') \| \| v\|_{\mathcal{Z}}$). We verify that $\overline{\hpit}$ is self-adjoint in $\mathcal{Z}$, ie 
$\langle v, \overline{\hpit} v' \rangle_{\mathcal{Z}} = \langle \overline{\hpit} v, v' \rangle_{\mathcal{Z}}$. 
We also verify that the span of $\overline{\hpit}$ is finite-dimensional, thanks to the atomic nature of $\rho^*$. Indeed, for each $w^* \in \text{supp}(\rho^*)$ and $l \in \{1,d\}$, let $\chi_{w^*,l} \in \mathcal{Z}$ be the indicator $\chi_{w^*}(w) = e_l \mathbf{1}( \xi^\infty(w) = w^*)$, where $e_l$ is the $l$-th canonical basis vector. 
We verify that if $v \perp \mathcal{W}:=\text{span}( \chi_{w^*,l}; \, w^* \in \text{supp}(\rho^*), l \in \{1,d\} )$, then $\overline{\hpit}v = 0$.
\end{remark}



The following lemma implies Lemma~\ref{lemma:balancedsimple}. Recall that $\cstarr = \min\left(|\supp{\rho^*}|, \on{dim}(\rho^*)^{2\on{degree(\sigma)} + 1}\right)$.

\begin{lemma}\label{lemma:balancedbody} 
Suppose Assumption~\ref{assm:rhostarnew} holds. Then for any $\mu$, there exists an balanced spectral distribution $\mq$ of $(\hpit, \mu)$ which is $\frac{2\cstarr}{\min_{w^* \in \supp{\rho^*}}\mathbb{P}_{\xii_{\#}\mu}[w^*]}$ balanced. If \ref{def:transitiveI1} additionally holds, then there exists an balanced spectral distribution $\mq$ of $(\hpit, \rho_0)$ which is $2\cstarr$-balanced.
\end{lemma}

\begin{proof}[Proof of Lemma~\ref{lemma:balancedbody} ]
We will show that the linear operator induced by $(\hpit, \mu)$ has an \wed $\mq$ which is balanced for some constant depending on $\rho^*$.

\begin{claim}\label{claim:hdecomp}
We can write 
\begin{align}
\hpit(w, w') = M_1(\xii(w), \xii(w')) UU^\top  + M_2(\xii(w), \xii(w')),
\end{align}
where for $w^*, {w^*}' \in \supp{\rho^*}$, 
\begin{align}
    M_1(w^*, {w^*}') &:= \mathbb{E}_{x \sim \md_V} \sigma'(x^\top w^*)\sigma'(x^\top {w^*}') \\
    M_2(w^*, {w^*}') &:= \mathbb{E}_{x \sim \md_V} \sigma'(x^\top w^*)\sigma'(x^\top {w^*}')P^{\perp}_{w^*} xx^\top  P^{\perp}_{{w^*}'}.
\end{align}
Further, both $M_1$ and $M_2$ have rank at most $\cstarr = \min\left(|\supp{\rho^*}|,  \on{dim}(\rho^*)^{2\on{degree(\sigma) + 1}}\right)$.
\end{claim}
\begin{proof}
Let $V$ be the orthonormal basis spanning $\supp{\rho^*}$, and let $U$ be any orthonormal basis of $\R^d \setminus \on{span}(V)$. Recall that Assumption~\ref{assm:rhostarnew} guarantees that the distribution of $x$, ${\md_x}$, can be factorized as $\md_{U} \otimes \md_{V}$, where $\on{span}(\md_{U}) \in \on{span}(U)$, $\on{span}(\md_{V}) \in \on{span}(V)$, $\mathbb{E}_{x \sim \md_U}xx^\top  = UU^\top $, and $\mathbb{E}_{x \sim \md_U}x = 0$.

Recall that $\hpit(w, w') = \mathbb{E}_{x \sim {\md_x}}\sigma'(x^\top \xii(w))\sigma'(x^\top \xii(w'))xx^\top $. Observe that for $u, v \in \text{Span}(U)$, we have 
\begin{align}
u^\top  \hpit(w, w') v &= \mathbb{E}_{x \sim {\md_x}} \sigma'(x^\top \xii(w))\sigma'(x^\top \xii(w'))(u^\top x)(v^\top x)\\
&=\mathbb{E}_{x \sim \md_V} \sigma'(x^\top \xii(w))\sigma'(x^\top \xii(w')) \mathbb{E}_{x \sim \md_U}u^\top xx^\top v\\
&= \mathbb{E}_{x \sim \md_V} \sigma'(x^\top \xii(w))\sigma'(x^\top \xii(w')) \mathbb{E}_{x \sim \md_U}u^\top v.
\end{align}
If $u \in \text{Span}(U)$, $v \in \text{Span}(V)$, then it is easy to check by the fact that $\mathbb{E}_{x \sim \md_U}x = 0$ that
\begin{align}
u^\top  \hpit(w, w') v &= \mathbb{E}_{x_V \sim \md} \sigma'(x_V^\top \xii(w))\sigma'(x_V^\top \xii(w'))(v^\top x_V)\mathbb{E}_{x_U \sim \md_U}(u^\top x_U) = 0.
\end{align}

For $w^*, {w^*}' \in \supp{\rho^*}$, let 
\begin{align}
    M_1(w^*, {w^*}') &:= \mathbb{E}_{x \sim \md_V} \sigma'(x^\top w^*)\sigma'(x^\top {w^*}') \\
    M_2(w^*, {w^*}') &:= \mathbb{E}_{x \sim \md_V} \sigma'(x^\top w^*)\sigma'(x^\top {w^*}')P^{\perp}_{w^*} xx^\top  P^{\perp}_{{w^*}'}
\end{align}

such that by the above computations,
\begin{align}
\hpit(w, w') = M_1(\xii(w), \xii(w')) UU^\top  + M_2(\xii(w), \xii(w')).
\end{align}
The statement about the rank follows from the observations that (1) both $M_1$ and $M_2$ are defined on a space of size at most $|\supp{\rho^*}|$, and (2) Alternatively, we can replace the expectation of $x \sim \md_V$ with the expectation over some $x \sim \md'_V$, where $\md'_V$ is supported on at most $\on{dim}(V)^{2\on{degree(\sigma)} + 1}$ points, and all the moments of $\md'_V$ up to the $\on{degree(\sigma)}$th degree match those of $\md_V$ (as this requires matching at most $\sum_{j = 0}^{2\on{degree}(\sigma)}\on{dim}(V)^j \leq \on{dim}(V)^{2\on{degree(\sigma)} + 1}$ terms.) 
\end{proof}

We will construct $\mq$ using the eigenfunctions of each of these two parts. Let $\mathcal{F} \subset L^2(\supp{\rho^*}, ({\xii})_{\#}\mu, \mathbb{R}^d)$ be an orthonormal basis of eigenfunctions of the linear operator $(M_2, ({\xii})_{\#}\mu)$, that is, we have
\begin{align}
    &\sum_{f \in \mathcal{F}} \lambda_f f(w^*)f({w^*}')^\top  = M_2(w^*, {w^*}')\\
    &\mathbb{E}_{{w^*}' \sim ({\xii})_{\#}\mu} M_2(w^*, {w^*}')f({w^*}') = \lambda_f f(w^*),
\end{align}

Let $\mathcal{Y}  \subset L^2(\supp{\rho^*}, ({\xii})_{\#}\mu)$ be an orthonormal basis of eigenfunctions of the linear operator $(M_1,({\xii})_{\#}\mu) $, that is, we have
\begin{align}
    &\sum_{y \in \mathcal{Y}} \lambda_y y(w^*)y({w^*}') = M_1(w^*, {w^*}') \\
    &\mathbb{E}_{{w^*}' \sim ({\xii})_{\#}\mu} M_1(w^*, {w^*}')y({w^*}') = \lambda_y y(w^*)
\end{align}

Let $\Lambda = \Lambda_1 \cup \Lambda_2$, where
\begin{align}
    \Lambda_2 := \{ \lambda_f : f \in \cf\} \qquad \Lambda_1 = \{ \lambda_y : y \in \cy\}.
\end{align}

The following claim is immediate to check from the decomposition of $\hpit$ in Claim~\ref{claim:hdecomp}.
\begin{claim}
Let $\mathcal{P}_{\lambda}$ be the projector onto the eigenspace of $\hpit$ with eigenvalue $\lambda$. Then $\mathcal{P}_{\lambda} = \overline{P}_{\lambda}$, where
\begin{align}
    P_{\lambda}(w, w') &:= \sum_{f \in \cf}f(\xii(w))f(\xii(w')^\top \mathbf{1}(\lambda_f = \lambda) + UU^\top \sum_{y \in \cy}y(\xii(w))y(\xii(w')\mathbf{1}(\lambda_y = \lambda)\\
    &= \sum_{v \in \mathcal{B}_{\lambda}}v(w)v(w')^\top ,
\end{align}
where 
\begin{align}
    \mathcal{B}_{\lambda} &:= \{v^f : \lambda_f = \lambda\}_{f \in \cf} \cup \{v^{y, i} : \lambda_y = \lambda\}_{y \in \cy},
\end{align}
and 
\begin{align}
    v^f(w) &:= f(\xii(w));\\
    v^{y, i}(w) &:= y(\xii(w))U_i.
\end{align}
\end{claim}

It remains to check how balanced this spectral decomposition is. Let $p := \min_{w^* \in \supp{\rho^*}} \mathbb{P}_{\xii_{\#}\mu}[w^*]$, and observe that $\max_{w, f \in \cf, y \in \mathcal{Y}} \left(\|f(w)\|, |y(w)|\right) \leq \frac{1}{\sqrt{p}}$, since the eigenfunctions are orthonormal. Fix $\lambda \in \Lambda$. We have
\begin{align}
    \sum_{v \in \mathcal{B}_{\lambda}} v(w)v(w)^\top  &= \sum_{f \in \cf} v^f(w)(v^f(w))^\top \mathbf{1}(\lambda_f = \lambda) + \sum_{y \in \cy}\sum_{i = 1}^{\on{dim}(U)} v^{y, i}(w)v^{y, i}(w)\mathbf{1}(\lambda_y = \lambda)\\
    &= \sum_{f \in \cf} f(\xii(w))(f(\xii(w)))^\top \mathbf{1}(\lambda_f = \lambda) + \sum_{y \in \cy} y(\xii(w))y(\xii(w)) UU^\top  \mathbf{1}(\lambda_y = \lambda)\\
    &\preceq \frac{I}{p}\left(\sum_{f \in \cf} \mathbf{1}(\lambda_f = \lambda) + \sum_{y \in \cy} \mathbf{1}(\lambda_y = \lambda)\right).
\end{align}
Thus letting 
\begin{align}
    \eta_{\lambda}^2 := \frac{1}{p}\left(\sum_{f \in \cf} \mathbf{1}(\lambda_f = \lambda) + \sum_{y \in \cy} \mathbf{1}(\lambda_y = \lambda)\right),
\end{align}
by Claim~\ref{claim:hdecomp}, we have that 
\begin{align}
    \sum_{\lambda \in \Lambda} \eta_{\lambda}^2 = \frac{|\cf| + |\cy|}{p} \leq \frac{\on{rank}(M_1) + \on{rank}(M_2)}{p} \leq \frac{2\cstarr}{p}.
\end{align}

Thus $\mq = \{(\mathcal{B}_{\lambda}, \eta_{\lambda})\}_{\lambda \in \Lambda}$ is $\frac{2\cstarr}{p}$-balanced.
This proves the first statement in the lemma.

If $(\rho^*, \rho_0, {\md_x})$ is transitive (as per Definition~\ref{def:transitive}), then we can get rid of the denominator and show that almost surely over $w \sim \rho_0$,
\begin{align}
    \sum_{v \in \mathcal{B}_\lambda} v(w)v(w)^\top \preceq I\left(\sum_{f \in \cf} \mathbf{1}(\lambda_f = \lambda) + \sum_{y \in \cy} \mathbf{1}(\lambda_y = \lambda)\right)
\end{align}
This suffices to prove the lemma.

To do this, let $\cg$ be the set of automorphisms of $(\rho^*, \rho_0, {\md_x})$ as per Definition~\ref{def:transitive}. For $h \in L^2(\sd, \rho_0, \mathbb{R}^d)$, define $g(h)$ by 
\begin{align}
    g(h)(w) := g^{-1}(f(g(w))).
\end{align}
For convenience, for $y \in \cy$, we will abuse notation and define
\begin{align}
     g(y)(w) := y(g(w)).
\end{align}


\begin{claim}[$\cg$-invariance of Eigenspaces.]\label{claim:inv}
If $f \in \mathcal{F}$ is an eigenfunction of $M_2$, then $g(f)$ is an eigenfunction of $M_2$ with the same eigenvalue. Simlary, if $y \in \mathcal{Y}$ is an eigenfunction of $M_1$, then $g(y)$ is an eigenfunction of $M_1$ with the same eigenvalue.
\end{claim}
\begin{proof}
We have 
\begin{align}
    M_2 & (g(f))(w^*) = \mathbb{E}_{{w^*}' \sim (\xii)_{\#}\rho_0}  \mathbb{E}_{x \sim \md_V} \sigma'(x^\top w^*)\sigma'(x^\top {w^*}')P^{\perp}_{w^*} xx^\top  P^{\perp}_{{w^*}'} g^{-1}(f(g({w^*}')))\\
    &= \mathbb{E}_{{w^*}' \sim (\xii)_{\#}\rho_0}  \mathbb{E}_{x \sim \md_V} \sigma'(x^\top w^*)\sigma'(x^\top {w^*}')g^{-1}\left(P^{\perp}_{g(w^*)} g(x)\right) x^\top  g^{-1}\left(P^{\perp}_{g({w^*}')} f(g({w^*}'))\right)\\
    &= \mathbb{E}_{{w^*}' \sim (\xii)_{\#}\rho_0}  \mathbb{E}_{x \sim \md_V} \sigma'(g(x)^\top g(w^*))\sigma'(g(x)^\top g({w^*}'))g^{-1}\left(P^{\perp}_{g(w^*)} g(x)\right) g(x)^\top  P^{\perp}_{g({w^*}')} f(g({w^*}'))\\
    &= \mathbb{E}_{{w^*}' \sim (\xii)_{\#}\rho_0}  \mathbb{E}_{x \sim \md_V} \sigma'(x^\top g(w^*))\sigma'(x^\top {w^*}')g^{-1}\left(P^{\perp}_{g(w^*)} x\right) x^\top  P^{\perp}_{{w^*}'} f({w^*}')\\
    &= g^{-1}\left(\mathbb{E}_{{w^*}' \sim (\xii)_{\#}\rho_0}  \mathbb{E}_{x \sim \md_V} \sigma'(x^\top g(w^*))\sigma'(x^\top {w^*}')P^{\perp}_{g(w^*)} xx^\top  P^{\perp}_{{w^*}'} f({w^*}')\right)\\
    &= g^{-1}\left(M_2 f (g(w^*))\right)\\
    &= g^{-1}\left(\lambda_f f (g(w^*))\right)\\
    &= \lambda_f g(f)(w^*)
\end{align}
Here in the second line with used the fact that for any $w$ and $z$, we have
$$(I - ww^\top )z = z - w w^\top z = z - w g(w)^\top g(z) = g^{-1}\left((I - g(w)g(w)^\top )g(z)\right)$$
If the third line, we just used that for $z, z' \in \mathbb{R}^d$, we have $z^\top z' = g(z)^\top g(z')$. In the fourth line, we used the symmetry of ${\md_x}$ and $(\xii)_{\#}\rho_0$ with respect to $\cg$ (see \ref{A2}).
The proof for that $ M_1 g(y)(w^*) = \lambda_y g(y)(w^*)$ is similar (but simpler); we omit the computation.
\end{proof}

Let $\mu_{\cg}$ the uniform measure over the group generated by the set of all $g_{w^*, {w^*}'} \in \cg$ for $w^*, {w^*}' \in \supp{\rho^*}$, where $g_{w^*, {w^*}'}(w^*) = {w^*}'$. Observe that $\mu_{\cg}$ a left-invariant measure on $\cg$, that is, for any $w^* \in \supp{\rho^*}$, we have that the distribution of $g(w^*)$ is uniform on $\rho^*$ when $g \sim \mu_{\cg}$ (that is, it equals $\rho^*$, since $\rho^*$ is atomic). Also note that for $g \in \supp{\mu_{\cg}}$ and $v \in \on{span}(V)$, we have that $g(v) \in \on{span}(V)$. Thus for $u \in \on{span}(U)$, we have $g(u) \in \on{span}(U)$, and thus in particular, since $g$ preserves dot products, and thus orthonormality,
\begin{align}\label{eq:gu}
    g^{-1}(U)g^{-1}(U)^\top  = UU^\top .
\end{align}

\begin{claim}\label{claim:basis}
Let $g \in \supp{\mu_{\cg}}$, and define $g(\mathcal{B}_{\lambda}) := \{g(v)\}_{v \in \mathcal{B}_{\lambda}}$. Then $g(\mathcal{B}_{\lambda})$ is an orthonormal basis for $\mathcal{P}_{\lambda}$. 
\end{claim}
\begin{proof}
First we will check that almost surely over $w$, $w'$,
\begin{align}\label{eq:goal}
\sum_{f \in \cf} \lambda_f g(v^f)(w)g(v^f)(w')^\top  + \sum_{y \in \cy} \lambda_y g(v^{y, i})(w)g(v^{y, i})(w')^\top  = \hpit(w, w').
\end{align}
Using the definition of $g(f)$ and \ref{A2}, almost surely over $w, w'$, we have for $z, z' \in \sd$, 
\begin{align}\label{eq:M2}
z^\top \sum_{f \in \cf} \lambda_f g(v^f)(w)g(v^f)(w')^\top z' &= z^\top \sum_{f \in \cf} \lambda_f g^{-1}\left(f(\xii(g(w)))\right)g^{-1}\left(f(\xii(g(w')))^\top \right)z'\\
     &= z^\top \sum_{f \in \cf} \lambda_f g^{-1}\left(f(g(\xii(w)))\right)g^{-1}\left(f(g(\xii(w')))^\top \right)z'\\
     &= \sum_{f \in \cf}\lambda_f g(z)^\top f(g(\xii(w)))f(g(\xii(w')))^\top g(z')\\
     &=  g(z)^\top  M_2(g(\xii(w)), g(\xii(w')))^\top g(z')\\
     &=   z^\top  M_2(\xii(w), \xii(w'))^\top z',
\end{align}
where here in the last line, we used the fact that 
\begin{align}
    z^\top  M_2(w^*, {w^*}')^\top z' = g(z)^\top  M_2(g(w^*), g({w^*}'))^\top g(z')
\end{align}
for any $g \in \cg$, $w^*, {w^*}'$. This can verified from the definition of $M_2$ and the fact that ${\md_x}$ is invariant with respect to $\cg$. 

We can perform a similar (much easier) calculation to show that 
\begin{align}
     \sum_{y \in \mathcal{Y}} \lambda_y g(y)(\xii(w))g(y)(\xii(w')) = M_1(\xii(w), \xii(w'));
\end{align}
this arises from the fact that $M_1(w^*, {w^*}') = M_1(g(w^*), g({w^*}'))$ since ${\md_x}$ is invariant with respect to $\cg$. We omit the details. Thus by \eqref{eq:gu},
\begin{align}\label{eq:M1}
    \sum_{y \in \cy} \lambda_y g(v^{y, i})(w)g(v^{y, i})(w')^\top  &= M_1(\xii(w), \xii(w'))g^{-1}(U)g^{-1}(U)^\top \\
    &= M_1(\xii(w), \xii(w'))UU^\top .
\end{align}

Employing \eqref{eq:M1} and \eqref{eq:M2}  yields \eqref{eq:goal} almost surely as desired.

Now, to prove the claim, we use (1) the fact from Claim~\ref{claim:inv} guarantees that $g(v)$ is an eigenfunction with the same values as $v$, and (2) the fact that the set $\{g(v)\}_{v \in \mathcal{B}_{\lambda}}$ is orthonormal (since dot products are preserved under rotations). These two facts guarantee that $g(\mathcal{B}_{\lambda})$ is a basis for $\mathcal{P}_{\lambda}$.
\end{proof}

The following claim now suffices to prove the lemma.
\begin{claim}
For any $w \in \sd$, we have
\begin{align}
    \sum_{v \in \mathcal{B}_{\lambda}} v(w)v(w)^\top  & \preceq I\left(\sum_{f \in \cf} \mathbf{1}(\lambda_f = \lambda) + \sum_{y \in \cy} \mathbf{1}(\lambda_y = \lambda)\right).
\end{align}
\end{claim}
\begin{proof}
Fix any $w \in \sd$, and let $w^* = \xii(w)$. For $z \in \mathbb{R}^d$, let $\pi_{z} \in L^2(\sd, \rho_0, \mathbb{R}^d)$ be defined by $\pi_z(w') = z \mathbf{1}(\xii(w') = w^*)$. Then since for $v \in \mathcal{B}_{\lambda}$, we have $v(w) = v(w')$ if $\xii(w) = \xii(w')$, it follows that
\begin{align}
    z^\top P_{\lambda}(w, w)z = \sum_{v \in \mathcal{B}_{\lambda}}z^\top v(w)v(w)^\top z = \frac{\langle{\overline{P}_{\lambda}\pi_z, \pi_z}\rangle}{(\mathbb{P}_{w' \sim \rho_0}[{\xii}^{-1}(w^*)])^2} = |\supp{\rho^*}|^2\langle{\overline{P}_{\lambda}\pi_z, \pi_z}\rangle.
\end{align}
To see the last equality, observe that $\rho^* = \xii_{\#}\rho_0$ by \ref{A2}.

Now recall that by Claim~\ref{claim:basis}, for any $\lambda \in \Lambda$ and $g \in \supp{\mu_{\cg}}$, we have that $\{g(v)\}_{v \in \mathcal{B}_{\lambda}}$ is a basis for $\mathcal{P}_{\lambda} = \overline{P}_{\lambda}$, and thus
\begin{align}\label{eq:mug}
    z^\top P_{\lambda}(w, w)z &= |\supp{\rho^*}|^2\langle{\overline{P}_{\lambda}\pi_z, \pi_z}\rangle\\
    &= |\supp{\rho^*}|^2 z^\top \mathbb{E}_{g \sim \mu_{\cg}} \sum_{v \in g(\mathcal{B}_{\lambda})} \mathbb{E}_{w', w'' \sim \rho_0} v(w)v(w')^\top  \mathbf{1}(\xii(w'), \xii(w'') = w^*)z\\
    &= z^\top \mathbb{E}_{g \sim \mu_{\cg}} \left(\sum_{f \in \cf | \lambda_f = \lambda} g^{-1}(f(g(w^*)))g^{-1}(f(g(w^*)))^\top z + \sum_{y \in \cf | \lambda_y = \lambda}|y(g(w^*)|^2 g^{-1}(U)g^{-1}(U)^\top \right)z.
\end{align}

Now for any $f \in \cf$,
\begin{align}\label{eq:fmug}
\mathbb{E}_{g \sim \mu_{\cg}} g^{-1}(f(g(w^*)))(g^{-1}(f(g(w^*))))^\top  &\preceq \mathbb{E}_{g \sim \mu_{\cg}} \|f(g(w^*))\|^2 I\\ 
    &= \mathbb{E}_{{w^*}' \sim \rho^*} \|f({w^*}')\|^2 I\\
    &= I.
\end{align}
Here the second to last inequality holds because we have defined $\mu_{\cg}$ to be a left-invariant measure on $\cg$ that induces a uniform measure on $\supp{\rho^*}$. The last equation holds by the fact that $\rho^* = \xii_{\#}\rho_0$ (see \ref{A2}) and since $f$ is part of an orthonormal basis, we must have $\mathbb{E}_{w^* \sim \xii_{\#}\rho_0} \|f(w^*)\|^2 = 1$. Likewise, for $y \in \cy$, using \eqref{eq:gu},
\begin{align}\label{eq:ymug}
\mathbb{E}_{g \sim \mu_{\cg}} |(g(y))(w^*)|^2g^{-1}(U)g^{-1}(U)^\top 
    &=  \mathbb{E}_{g \sim \mu_{\cg}} |y(g(w^*))|^2UU^\top \\
    &= \mathbb{E}_{{w^*}' \sim \rho^*}|y({w^*}')|^2 UU^\top \\
    &= UU^\top .
\end{align}
Combining Equations~\eqref{eq:fmug} and \eqref{eq:ymug} with \eqref{eq:mug} yields that
\begin{align}
    P_{\lambda}(w, w) \preceq I\left(\sum_{f \in \cf} \mathbf{1}(\lambda_f = \lambda) + \sum_{y \in \cy} \mathbf{1}(\lambda_y = \lambda)\right),
\end{align}
as desired.
\end{proof}
\end{proof}




\subsubsection{Properties of Potential}\label{sec:app_pot_properties}
To prove our key lemmas \ref{lemma:dec1body}, \ref{lemma:dec2body}, \ref{lemma:l1pertbody}, we will need several preliminary lemmas.

\begin{lemma}\label{lemma:ddtphiv}
Suppose the high probability event in Lemma~\ref{lemma:conctau} holds for $S = B_{\tau}$ and $v \in L^2(\sd,\rho_0 ,\mathbb{R}^d)$ which is an eigenfunction of $\hpit$. Suppose $(\hpit, \rho_0)$ has the \cri with respect to $B_{\tau}^t := \xi_t^{-1}(B_{\tau})$. Then with $\|v\|_{\infty} := \sup{w \in \sd}\|v(w)\|$, we have
\begin{align}
\langle{\nabla \phi_v(t), \Delta_t}\rangle_{\hpit}^{B_{\tau}^t} = \mathbb{P}_{\rtmf}[B_{\tau}]\lambda_v \phi_v(t) + \mathcal{E}\|v\|_{\infty},
\end{align}
where
\begin{align}
 \mathcal{E} \leq \eps_m^{\ref{lemma:conctau}} \mathbb{E}_i \|\dit\| + \mathbb{E}_{i} \|\dit\|\mathbf{1}(\bwti \notin B_{\tau}).
\end{align}
\end{lemma}
\begin{proof}
First observe that
\begin{align}
\nabla \phi_v = v \on{sign}(\langle{v, \Delta_t}\rangle),
\end{align}
and thus
\begin{align}
\langle{\nabla \phi_v(t), \Delta_t}\rangle_{\hpit}^{B_{\tau}^t} = \on{sign}(\langle{v, \Delta_t}\rangle)\langle{v, \Delta_t}\rangle_{\hpit}^{B_{\tau}^t}
\end{align}
Now by the conclusion of the concentration Lemma~\ref{lemma:conctau}, we have
\begin{align}
\langle{v, \Delta_t}\rangle_{\hpit}^{B_{\tau}^t} = \mathbb{E}_i X(i) \dit \mathbf{1}(\bwti \in B_{\tau}) \pm \|v\|_{\infty}\eps_m^{\ref{lemma:conctau}} \mathbb{E}_i \|\dit\|.
\end{align}
where $X(i) = \mathbb{E}_{w' \sim \rho_0}\hpit(w_i, w')v(w')\mathbf{1}(\xi_t(w') \in B_{\tau})$
Now since $v$ is an eigenfunction of $\hpit$, by the definition of consistent isometry, we have that 
\begin{align}
X(i) = \lambda_v v(w_i)\mathbb{P}_{\rtmf}[B_{\tau}].
\end{align}
Thus 
\begin{align}
\langle{v, \Delta_t}\rangle_{\hpit}^{B_{\tau}^t} = \lambda_v \langle{v, \Delta_t}\rangle^{B_{\tau}^t}\mathbb{P}_{\rtmf}[B_{\tau}] \pm \eps_m^{\ref{lemma:conctau}}\|v\|_{\infty} \mathbb{E}_i \|\dit\|.
\end{align}
Now 
\begin{align}
 \on{sign}(\langle{v, \Delta_t}\rangle)\langle{v, \Delta_t}\rangle^{B_{\tau}^t} &= \on{sign}(\langle{v, \Delta_t}\rangle)\langle{v, \Delta_t}\rangle \pm \mathbb{E}_{i} \|\dit\|\mathbf{1}(\bwti \notin B_{\tau})\\
 &= \phi_v(t) \pm \mathbb{E}_{i} \|\dit\|\mathbf{1}(\bwti \notin B_{\tau}).
\end{align}
Plugging this back in yields the lemma.
\end{proof}

Now we prove Lemma~\ref{lemma:dec1body}, which we restate here.
\dech*
\begin{proof}
Let $B^t_{\tau} := \xi_t^{-1}(B_{\tau})$, and let $\bar{B}^t_{\tau}$ be the complement in $\sd$ of $B^t_{\tau}$. We decompose 
\begin{align}\label{eq:decomp}
\langle{\nabla \Phi_{\mq}(t), \Delta_t}\rangle_{\hpt} &= \langle{\nabla \Phi_{\mq}(t), \Delta_t}\rangle_{\hpt}^{B_{\tau}^t, B_{\tau}^t} + \langle{\nabla \Phi_{\mq}(t), \Delta_t}\rangle_{\hpt}^{B_{\tau}^t, \bar{B}_{\tau}^t} + \langle{\nabla \Phi_{\mq}(t), \Delta_t}\rangle_{\hpt}^{\bar{B}_{\tau}^t, \sd}.
\end{align}
Lets start with the first term $\langle{\nabla \Phi_{\mq}(t), \Delta_t}\rangle_{\hpt}^{B_{\tau}^t, B_{\tau}^t} = \langle{\nabla \Phi_{\mq}(t), \Delta_t}\rangle_{\hpt}^{B_{\tau}^t}$. Bounding this term is the key part of the lemma.
\begin{claim}\label{claim:master}
    \begin{align}
        \langle{\nabla \Phi_{\mq}(t), \Delta_t}\rangle_{\hpt}^{B_{\tau}^t} \geq - (\creg + 1)\mathbb{E}_i \|\dit\|\mathbf{1}(\xi_t(w_i) \notin B_{\tau}) - \cbal \eps_m^{\ref{lemma:conctau}} \Omega(t) +  \left| \langle{\nabla \Phi(t), G}\rangle\right|,
    \end{align}
    where $\mathbb{E}_i \|G(i)\| \leq  \creg \tau \Omega(t)$.
\end{claim}
\begin{proof}
We have
\begin{align}\label{eq:firstdecomp}
\langle{\nabla \Phi_{\mq}(t), \Delta_t}\rangle_{\hpt}^{B_{\tau}^t} = \langle{\nabla \Phi_{\mq}(t), \Delta_t}\rangle_{\hpit}^{B_{\tau}^t} + \langle{\nabla \Phi_{\mq}(t), G}\rangle,
\end{align}
where $\|G(i)\| \leq \creg \tau \mathbb{E}_i \|\dit\|$, since $\|K'(\xii(w), \xii(w')) - K'(\xi_t(w), \xi_t(w'))\| \leq \creg \tau$. This relies on the fact that from the proof of \ref{A2}, almost surely $\|\xi_t(w) - \xii(w)\| \leq \tau$, because $\|\xi_t(w) - \xii(w)\| \leq \min_{w^* \in \supp{\rho^*}} \|\xi_t(w) - w^*\| \leq \tau$. Now we will break up $\Phi_{\mq}$ into the $\Psi_{\mq}$ and $\Omega$ parts. Starting with the $\Psi_{\mq}$ part, we have

\begin{align}\label{eq:psisummary}
 \langle{\nabla \Psi_{\mq}(t), \Delta_t}\rangle_{\hpit}^{B_{\tau}^t} &= \sum_{\lambda \in \Lambda} \eta_\lambda \frac{\sum_{v \in \mathcal{B}_\lambda} \phi_v(t)  \langle{\nabla \phi_{v}(t), \Delta_t}\rangle_{\hpit}^{B_{\tau}^t} }{\sqrt{\sum_{v \in \mathcal{B}_\lambda}(\phi_v(t))^2}}\\
 &= \sum_{\lambda \in \Lambda} \eta_\lambda \frac{\sum_{v \in \mathcal{B}_\lambda} \phi_v(t)\left(\lambda \phi_v(t) \mathbb{P}_{\rtmf}[B_{\tau}] + \bf{\mathcal{E}_v}\right)}{\sqrt{\sum_{v \in \mathcal{B}_\lambda}(\phi_v(t))^2}} \\
 &= \mathbb{P}_{\rtmf}[B_{\tau}] \sum_{\lambda \in \Lambda} \eta_\lambda \left(\lambda \sqrt{\sum_{v \in \mathcal{B}_\lambda}(\phi_v(t))^2} + \frac{\sum_{v \in \mathcal{B}_\lambda} \phi_v(t) \bf{\mathcal{E}_v} }{\sqrt{\sum_{v \in \mathcal{B}_\lambda}(\phi_v(t))^2}} \right)  \\
 &\geq \mathbb{P}_{\rtmf}[B_{\tau}]\sum_{\lambda \in \Lambda} \eta_\lambda \lambda \sqrt{\sum_{v \in \mathcal{B}_\lambda}(\phi_v(t))^2 } - \mathcal{E}~,
\end{align}
where we used Cauchy-Schwartz in the last inequality, the fact that $
\sum_\lambda \eta_\lambda = 1$, and 
$\|\bf{\mathcal{E}_v}\| \leq \mathcal{E}$, the error term appearing in Lemma~\ref{lemma:ddtphiv}.

Next consider the $\langle{\nabla \Omega(t), \Delta_t}\rangle_{\hpit}^{B^t_{\tau}}$ part. Recall from the definition of \wed  that $\hpit(w, w') = \sum_{v \in \mq} \lambda_v v(w)v(w')^\top$. 
 Let $u_i := \nabla_i \Omega(t) = \frac{\dit}{\|\dit\|}$. We can expand
 \begin{align}\label{eq:omega}
\left|\langle{\nabla \Omega(t), \Delta_t}\rangle_{\hpit}^{B^t_{\tau}, \sd}\right| &= \left|\mathbb{E}_{i, j} \sum_{v \in \mq} \lambda_v u_i^\top v(w_i)v(w_j)^\top \djt \mathbf{1}(w_i \in B^t_{\tau})\right|\\
&= \left|\mathbb{E}_{i} \sum_{v \in \mq} \lambda_v u_i^\top v(w_i) \mathbf{1}(i \in B^t_{\tau}) \left(\mathbb{E}_j v(w_j)^\top \djt\right)\right| \\
&\leq \sum_{v \in \mq} \lambda_v \phi_v(t) \mathbb{E}_{i} |u_i^\top v(w_i)|\mathbf{1}(i \in B^t_{\tau}).
\end{align}

Now fix $i$. For any vector $u \in \sd$, since $\mq = \{(\mathcal{B}_{\lambda}, \eta_\lambda)\}_{\lambda \in \lambda}$ is $\cbal$-balanced, we have
\begin{align}
\sum_{v \in \mq} \lambda_v \phi_v(t) |u^\top v(w_i)| &= \sum_{\lambda \in \Lambda} \lambda \sum_{v \in \mathcal{B}_\lambda} \phi_v(t) |u^\top v(w_i)|\\
&\leq \sum_{\lambda \in \Lambda} \lambda \sqrt{\sum_{v \in \mathcal{B}_\lambda} (\phi_v(t))^2} \sqrt{\sum_{v \in \mathcal{B}_\lambda} |u^\top v(w_i)|^2}\\
&= \sum_{\lambda \in \Lambda} \lambda \sqrt{\sum_{v \in \mathcal{B}_\lambda} (\phi_v(t))^2} \sqrt{u^\top \left(\sum_{v \in \mathcal{B}_\lambda} v(w_i)v(w_i)^\top \right)u}\\
&\leq  \sum_{\lambda \in \Lambda} \eta_\lambda \lambda \sqrt{\sum_{v \in \mathcal{B}_\lambda} (\phi_v(t))^2}.
\end{align}
Here the final inequality follows from the definition of a \wed, which states that for any $w \in \sd$, $\sum_{v \in \mathcal{B}_{\lambda}} v(w)v(w)^\top  \preceq \eta_{\lambda}^2 I$.
Thus plugging this back into to Equation~\eqref{eq:omega}, we have
\begin{align}
\left|\langle{\nabla \Omega(t), \Delta_t}\rangle_{\hpit}^{B^t_{\tau}, [m]}\right| \leq \mathbb{P}_i[B^t_{\tau}]\sum_{\lambda \in \Lambda} \eta_\lambda \lambda \sqrt{\sum_{v \in \mathcal{B}_\lambda} (\phi_v(t))^2}.
\end{align}
Now letting $H_i = \hpit(w_i, w_j)\mathbb{E}_j \djt \mathbf{1}(w_i \notin B_{\tau}^t)$, we have
\begin{align}
    \left|\langle{\nabla \Omega(t), \Delta_t}\rangle_{\hpit}^{B^t_{\tau}} - \langle{\nabla \Omega(t), \Delta_t}\rangle_{\hpit}^{B^t_{\tau}, [m]}\right| \leq \left|\langle{\nabla \Omega(t), H}\rangle\right| \leq \creg  \mathbb{E}_i \|\dit\|\mathbf{1}(\xi_t(w_i) \notin B_{\tau}),
\end{align}
and thus
\begin{align}\label{eq:omegasummary}
    \left|\langle{\nabla \Omega(t), \Delta_t}\rangle_{\hpit}^{B^t_{\tau}}\right| \leq \mathbb{P}_i[B^t_{\tau}]\sum_{\lambda \in \Lambda} \eta_\lambda \lambda \sqrt{\sum_{v \in \mathcal{B}_\lambda} (\phi_v(t))^2} +  \creg  \mathbb{E}_i \|\dit\|\mathbf{1}(\xi_t(w_i) \notin B_{\tau}).    
\end{align}
Now recall that $\Phi_{\mq}(t) := \Omega(t) + \Psi_{\mq}(t)$. Thus combining Equations~\eqref{eq:omegasummary} and \eqref{eq:psisummary}, and Equation~\eqref{eq:firstdecomp}, and plugging in the bound on $\mathcal{E}$ from Lemma~\ref{lemma:ddtphiv}, we have
\begin{align}
    \langle{\nabla \Phi_{\mq}(t), \Delta_t}\rangle_{\hpt}^{B_{\tau}^t} \geq - (\creg + 1)\mathbb{E}_i \|\dit\|\mathbf{1}(\xi_t(w_i) \notin B_{\tau}) - \cbal \eps_m^{\ref{lemma:conctau}} \Omega(t) +  \left| \langle{\nabla \Phi_{\mq}(t), G}\rangle\right|,
\end{align}
where $\mathbb{E}_i \|G_i\| \leq  \creg \tau \Omega(t)$. Here we have also used the fact that for all $v$ in the \wed $\mq$, we have that $\|v\|_{\infty} \leq \sqrt{\cbal} \leq \cbal$ (this is evident from the definition of \wed).
This proves the claim.
\end{proof}

Next consider the second term $\langle{\nabla \Phi_{\mq}(t), \Delta_t}\rangle_{\hpt}^{B_{\tau}^t, \bar{B}_{\tau}^t}$ in Equation~\eqref{eq:decomp}. We have 
\begin{align}\label{eq:secondterm}
\left|\langle{\nabla \Phi_{\mq}(t), \Delta_t}\rangle_{\hpt}^{B_{\tau}^t, \bar{B}_{\tau}^t} \right| = \langle{\nabla \Phi_{\mq}(t), H}\rangle,
\end{align}
where $\|H(i)\| \leq \creg \mathbb{E}_i \|\dit\|\mathbf{1}(\bwti \notin B_{\tau})$. 

Finally, for the third term  $\langle{\nabla \Phi_{\mq}(t), \Delta_t}\rangle_{\hpt}^{\bar{B}_{\tau}^t, \sd}$ in Equation~\eqref{eq:decomp}, we have just write 
\begin{align}\label{eq:thirdterm}
    \langle{\nabla \Phi_{\mq}(t), \Delta_t}\rangle_{\hpt}^{\bar{B}_{\tau}^t, \sd} = \langle{\nabla \Phi_{\mq}(t), m_t}\rangle^{\bar{B}_{\tau}^t},
\end{align}
where we recall that $m_t(i) = \mathbb{E}_j \hpt(i, j)\djt$.

Combining Equations~\eqref{eq:secondterm}, \eqref{eq:thirdterm} and Claim~\ref{claim:master} into Equation~\eqref{eq:decomp}, we obtain that

\begin{align}
    \langle{\nabla \Phi_{\mq}(t), \Delta_t}\rangle_{\hpt} \geq - (\creg + 1)\mathbb{E}_i \|\dit\|\mathbf{1}(\xi_t(w_i) \notin B_{\tau}) - \cbal \eps_m^{\ref{lemma:conctau}} \Omega(t) +  \left| \langle{\nabla \Phi_{\mq}(t), G + H + m_t}\rangle\right|,
\end{align}
where $\mathbb{E}_i \|G(i) + H(i)\| \leq  \creg \left(\tau \Omega(t) + \mathbb{E}_i \|\dit\|\mathbf{1}(\xi_t(w_i) \notin B_{\tau}) \right) $.

Now we use Lemma~\ref{lemma:l1pertbody} to bound 
\begin{align}
    \left| \langle{\nabla \Phi_{\mq}(t), G + H + m_t}\rangle\right| &\leq \mathbb{E}_i\|G(i) + H(i) + m_t(i)\|\left(1 + \cbal\right)\\
    &\leq \left(\creg \left(\tau \Omega(t) + \mathbb{E}_i \|\dit\|\mathbf{1}(\xi_t(w_i) \notin B_{\tau})\right) + \mathbb{E}_i \|m_t(i)\|\right)\left(1 + \cbal\right).
\end{align}
Plugging this back in to the equation above yields
\begin{align}
    \langle{\nabla \Phi_{\mq}(t), \Delta_t}\rangle_{\hpt} \geq - (1 + \cbal)\mathbb{E}_i \|m_t(i)\| - \mathcal{E}_{\ref{lemma:dec1body}},
\end{align}
where 
\begin{align}
    \mathcal{E}_{\ref{lemma:dec1body}} &= (\creg(2 + \cbal) + 1)\mathbb{E}_i \|\dit\|\mathbf{1}(\xi_t(w_i) \notin B_{\tau}) + (\cbal \eps_m^{\ref{lemma:conctau}} +  (1 + \cbal)\creg \tau)\Omega(t)\\
    &= O_{\creg, \cbal}(\mathbb{E}_i \|\dit\|\mathbf{1}(\xi_t(w_i) \notin B_{\tau}) + (\tau + \cbal \eps_m^{\ref{lemma:conctau}}) \Omega(t)).
\end{align}
This proves the lemma.
\end{proof}

Now we prove Lemma~\ref{lemma:dec2body}, which we restate here.
\decd*
\begin{proof}
Let $\er := \sqrt{L_{\md}(\rtmf)}$. We will show that 
\begin{align}\label{eq:Domega}
    \langle{\nabla \Omega(t), \dpt \odot \Delta_t}\rangle \leq - \left(\clsc\er\right) \Omega(t) + 2\creg\mathbb{E}_{i}\|\dit\|\mathbf{1}(\bwti \notin B_{\tau}),
\end{align}
and that 
\begin{align}\label{eq:Dpsi}
    \langle{\nabla \Psi_{\mq}(t), \dpt \odot \Delta_t}\rangle &\leq - \left(\clsc\er\right) \Psi_{\mq}(t) + \frac{\clsc \er + 2\cbal \creg \tau}{2} \Omega(t)
    \\
    &\qquad 2\cbal \creg\mathbb{E}_{i}\|\dit\|\mathbf{1}(\bwti \notin B_{\tau}) + \cbal \mathbb{E}_i \|\dit\|^2.
\end{align}
The first statement is straightforward. Since $\nabla_i \Omega(t) = \frac{\dit}{\|\dit\|}$, we have
\begin{align}
    \langle{\nabla \Omega(t), \dpt \odot \Delta_t}\rangle &\leq \mathbb{E}_i \frac{\dit^\top \dpt(i)\dit}{\|\dit\|}\\
    &= \mathbb{E}_i \frac{\dit^\top \dpt(i)\dit}{\|\dit\|}\mathbf{1}(\xi_t(w_i) \in B_{\tau}) + \mathbb{E}_i \frac{\dit^\top \dpt(i)\dit}{\|\dit\|}\mathbf{1}(\xi_t(w_i) \notin B_{\tau})\\
    &\leq - \clsc\er \mathbb{E}_i \|\dit\|\mathbf{1}(\xi_t(w_i) \in B_{\tau}) + \mathbb{E}_i \|\dpt(i)\dit\|\mathbf{1}(\xi_t(w_i) \notin B_{\tau})\\
    &\leq - \clsc\er \mathbb{E}_i \|\dit\| + 2\creg \mathbb{E}_i \|\dit\|\mathbf{1}(\xi_t(w_i) \notin B_{\tau}),
\end{align}
as desired.

For the second statement, write 
\begin{align}
    \dpt(i) = D^{\text{good}}_t(i) + D^{\text{bad}}_t(i),
\end{align}
where 
\begin{align}
    D^{\text{good}}_t(i) = -c_1 P^{\perp}_{\xii(w_i)}(VV^\top )P^{\perp}_{\xii(w_i)} - c_2 (UU^\top ).
\end{align}
By the structured condition in Assumption~\ref{def:lsc}, we can write such a decomposition where $c_1, c_2 \geq \clsc\er$, and for any $i$ such that $\xi_t(w_i) \in B_{\tau}$, we have $\|D^{\text{bad}}_t(i)\| \leq \frac{\clsc \er}{2\sqrt{\cbal}} + \creg \tau$. Note that this decomposition still holds for $i$ where $\xi_t(w_i) \notin B_{\tau}$, but $\|D^{\text{bad}}_t(i)\|$ can be as large as $2\creg$.

\begin{claim}
\begin{align}
    \langle{\nabla \phi_v(t), D^{\text{good}}_t \odot \Delta_t}\rangle \leq  -\clsc\er \phi_v(t) +  \langle{\nabla \phi_v(t), G}\rangle,
\end{align}
where $\|G(i)\| \leq \tau \|\dit\| + 0.5\|\dit\|^2 + \|\dit\| \mathbf{1}(\xi_t(w_i) \notin B_{\tau})$.;  
\end{claim}
\begin{proof}

Now recall that in the construction for $\mq$ given in Lemma~\ref{lemma:balancedbody}, for any $v \in \supp{\mq}$, it holds that either $v(w) \in \on{span}(U)$ for all $w \in \sd$, or $v(w) \in \on{span}(V)$ for all $w \in \sd$.
We consider the two cases separately. First suppose $v(w) \in \on{span}(U)$ for all $w \in \sd$. Fix $w_i$ with $\xi_t(w_i) \in B_{\tau}$. For any $w$, we have
\begin{align}
    v(w)^\top D^{\text{good}}_t(i)\dit &=  -c_2 v(w)^\top \dit,
\end{align}
and thus the desired conclusion holds. Now suppose $v(w) \in \on{span}(V)$. Note that $V$ commutes with $P^{\perp}_{\xii(w_i)}$. Thus any $w$, we have
\begin{align}
    v(w)^\top D^{\text{good}}_t(i)\dit &= -c_1 v(w)P^{\perp}_{\xii(w_i)}\dit.
\end{align}
Now for $i$ with $\xi_t(w) \in B_{\tau}$, we have $\|\xi_t(w) - \xii(w)\| \leq \tau$ (see the proof of \ref{A2}), and thus, since additionally $|\dit \xi_t(w)| \leq \frac{\|\dit\|2}{2}$ (see \eqref{eq:perp} in the proof of Lemma~\ref{lemma:errdynamicsbody}), we have that 
\begin{align}
    v(w)^\top D^{\text{good}}_t(i)\dit &= -c_1 v(w)P^{\perp}_{\xii(w_i)}\dit\\
    &= -c_1 v(w)\dit + O(\tau \|v(w)\| + \|\dit\|^2).
\end{align}
Thus in conclusion, we have that
\begin{align}
    \langle{\nabla \phi_v(t), D^{\text{good}}_t \odot \Delta_t}\rangle \leq -c_2\er \phi_v(t) +  \langle{\nabla \phi_v(t), G}\rangle,
\end{align}
where $\|G(i)\| \leq \tau \|\dit\| + 0.5\|\dit\|^2 + \|\dit\| \mathbb{1}(\xi_t(w_i) \notin B_{\tau})$. This proves the claim.
\end{proof}
Thus with $G$ as in the claim,
\begin{align}
    \langle{\nabla \Psi_{\mq}(t), D^{\text{good}}_t \odot \Delta_t - G}\rangle &\leq \sum_{\lambda \in \Lambda} \eta_\lambda \frac{\sum_{v \in \mathcal{B}_\lambda} \phi_v(t) \langle{\nabla \phi_v(t), D^{\text{good}}_t \odot \Delta_t}\rangle}{\sqrt{\sum_{v \in \mathcal{B}_\lambda}(\phi_v(t))^2}}\\
     &\leq \sum_{\lambda \in \Lambda} \eta_\lambda \frac{\sum_{v \in \mathcal{B}_\lambda} -\clsc \er (\phi_v(t))^2 }{\sqrt{\sum_{v \in \mathcal{B}_\lambda}(\phi_v(t))^2}}\\
    &= -\clsc\er \sum_{\lambda \in \Lambda} \eta_\lambda \sqrt{\sum_{v \in \mathcal{B}_\lambda}(\phi_v(t))^2}\\
    &= -\clsc\delta \Phi_{\mq}(t).
\end{align}
It follows that from the proof of Lemma~\ref{lemma:l1pertbody} (see Equation~\eqref{eq:l1psi}) we have
\begin{align}
    |\langle{\nabla \Psi_{\mq}(t), D^{\text{good}}_t \odot \Delta_t - G}\rangle| \leq  \cbal \left(\tau\Omega(t) + 0.5\mathbb{E}_i \|\dit\|^2 + \mathbb{E}_{i}\|\dit\|\mathbf{1}(\bwti \notin B_{\tau})\right)
\end{align}
Similarly, we have that 
\begin{align}
    \langle{\nabla \Psi_{\mq}(t), D^{\text{bad}}_t \odot \Delta_t}\rangle &=  \langle{\nabla \Psi_{\mq}(t), D^{\text{bad}}_t \odot \Delta_t}\rangle
    _{B_{\tau}^t} +  \langle{\nabla \Psi_{\mq}(t), D^{\text{bad}}_t \odot \Delta_t}\rangle_{\bar{B}_{\tau}^t}\\
    &\leq \cbal \left(\frac{\clsc \er}{2\sqrt{\cbal}} + \creg \tau \right)\mathbb{E}_i \|\dit\| + \cbal (2\creg)\mathbb{E}_i \|\dit\| \mathbf{1}(\bwti \notin B_{\tau}),
\end{align}
and so
\begin{align}
    \langle{\nabla \Psi_{\mq}(t), \dpt \odot \Delta_t}\rangle \leq - \clsc \er \Psi_{\mq}(t) + \left(\frac{\clsc \er + 2\cbal \creg \tau}{2} \Omega(t)\right) + 3\cbal\creg \mathbb{E}_{i}\|\dit\|\mathbf{1}(\bwti \notin B_{\tau}).
\end{align}
This yields \eqref{eq:Dpsi}, which proves the lemma.

\end{proof}
We now prove Lemma~\ref{lemma:l1pertbody}, which we restate here.
\pert*
\begin{proof}[Proof of Lemma~\ref{lemma:l1pertbody}]
    First observe that $\langle{\nabla \Phi_{\mq}(t), G}\rangle \leq \mathbb{E}_i \|G(i)\|$, since $\nabla_i \Omega(t) = \frac{\dit}{\|\dit\|}$, which has norm $1$. Now for any $v \in \supp{\mq}$, we have
    \begin{align}
        |\langle{\nabla \phi_{v}(t), G}\rangle| \leq \mathbb{E}_i  |G(i)^\top v(w_i)| ,
    \end{align}
    and so
    \begin{align}\label{eq:l1psi}
    |\langle{\nabla \Psi_{\mq}(t), G}\rangle| &\leq \sum_{\lambda \in \Lambda} \eta_\lambda \frac{\sum_{v \in \mathcal{B}_\lambda} \phi_v(t) |\langle{\nabla \phi_v(t), G }\rangle}{\sqrt{\sum_{v \in \mathcal{B}_\lambda}(\phi_v(t))^2}}\\
    &\leq \mathbb{E}_i \left[\sum_{\lambda \in \Lambda} \eta_\lambda \frac{\sum_{v \in \mathcal{B}_\lambda} \phi_v(t) |G(i)^\top v(w_i)| }{\sqrt{\sum_{v \in \mathcal{B}_\lambda }(\phi_v(t))^2}}\right]\\
    &\leq \mathbb{E}_i \left[\sum_{\lambda \in \Lambda} \eta_\lambda \frac{\sqrt{\sum_{v \in \mathcal{B}_\lambda} (\phi_v(t))^2}\sqrt{\sum_{v \in \mathcal{B}_\lambda} |G(i)^\top v(w_i)|^2} }{\sqrt{\sum_{v \mathcal{B}_\lambda}(\phi_v(t))^2}}\right]\\
    &= \mathbb{E}_i \left[\sum_{\lambda \in \Lambda} \eta_\lambda \sqrt{G(i)^\top \left(\sum_{v \in \mathcal{B}_\lambda} v(w_i)v(w_i)^\top \right) G(i)}\right]\\
    &\leq \mathbb{E}_i \sum_{\lambda \in \Lambda} \eta_\lambda^2 \|G(i)\| = \cbal \mathbb{E}_i\|G(i)\|.
    \end{align}
    Here in the third inequality, we used Cauchy-Schwartz.
    It follows that 
    \begin{align}
        |\langle{\nabla \Phi_{\mq}(t), G}\rangle| \leq |\langle{\nabla \Omega(t), G}\rangle| + |\langle{\nabla \Psi_{\mq}(t), G}\rangle| \leq (1 + \cbal)\mathbb{E}_i\|G(i)\|,
    \end{align}
    as desired.

    
    
    
\end{proof}

\subsection{Dynamics of the Potential}
Before proving our main theorem on the dynamics of the potential, we need the following lemma, which gathers all the required concentration events.
\begin{lemma}\label{lemma:alltheconc}
Fix some $\er$. With high probability as $d, m, n \rightarrow \infty$, the events in all concentration lemmas (Lemma~\ref{lemma:concentration},Lemma~\ref{lemma:nconcentration}, Lemma~\ref{lemma:concJ} and Lemma~\ref{lemma:conctau}) hold, where we apply Lemma~\ref{lemma:concJ} and Lemma~\ref{lemma:conctau} for 
$S = B_{\tau}$ for all
\begin{align}
\tau \in \left\{\frac{\clsc \cdot \text{rd}(e)}{8 (\chdec + \cdecd)}\right\}_{e \in [\er, 1]},
\end{align}
where $\text{rd}(z)$ is a rounding of $z$ to its first non-zero decimal, in binary (so $\text{rd}(z) \in [z/2, z]$). We also apply Lemma~\ref{lemma:conctau} for all eigenfunctions $v$ in the \wed $\mq$.
\end{lemma}
\begin{proof}
The set $\left\{\frac{c\cdot \text{rd}(e)}{8 (\chdec + \cdecd)}\right\}_{e \in [\er, 1]}$ has size at most $O_{\chdec, \cdecd}(\log_2(1/\er))$, so we can take a union bound over the result in Lemma~\ref{lemma:concJ} for all $B_{\tau}$. Similarly, since there are $O(d \cstarr)$ eigenfunctions in $\mq$ (see the proof of Lemma~\ref{lemma:balancedbody}), we take a union bound of Lemma~\ref{lemma:conctau} over all these eigenfunctions. (Note that the ``with high probability'' is explicitly $o(1/d)$ there).
The rest follows immediately from the three concentration lemmas.
\end{proof}

For the remainder of the text, we assume the following assumptions hold up to time $T$ (if relevant): Assumptions~\ref{assm:reg},\ref{assm:growth},\ref{def:lsc},\ref{assm:symmetries}. Let $(\clsc, \tlsc)$ denote the parameters of the local strong convexity (we will use the parameter $\tau$ differently later). We also assume that  $\mq$ is a $\cbal$-balanced \wed, where by Lemma~\ref{lemma:balancedsimple}, we have that $\cbal = \cstarr$.

\begin{theorem}[Main Potential Dynamics Theorem]\label{lemma:ddtphi}
Let $\er := \sqrt{L_{\md}(\rTmf)}$, and condition on the event that the high probability event in Lemma~\ref{lemma:alltheconc} holds for $\delta$. Let $\eps_{n, m} := \eps_n + \epsmv + \epsmj + \epsmh$ from the concentration lemmas. Suppose $n$ and $m$ are large enough such that $\jmax^2t^2\eps_{n, m} \leq \frac{1}{64}$.
Suppose that
\begin{align}\label{eq:ind}
    \jmax^2 \left(\int_{s = 0}^{t} \Phi_{\mq}(s)^2ds\right) \leq \eps_{n, m}.
\end{align}
Then for some $C = O_{\creg, \cbal}(1)$ and $\tau = \Omega_{\creg, \cbal}(\er)$, for all $t \leq T$, we have
\begin{align}
\frac{d}{dt} \Phi_{\mq}(t) &\leq - \frac{\clsc\er}{C} \Phi_{\mq}(t) + C\javg(\tau) \int_{s = 0}^t \Phi_{\mq}(s)ds + C\jmax t\eps_{n, m}.
\end{align}
\end{theorem}

\begin{corollary}[Solution to Potential Dynamics]\label{cor:main}
Suppose that for some $\tau = \Omega_{\creg, \cbal}(\er)$,
$$4\jmax^4C^2T^3 \exp(2C\javg (\tau) T/(\clsc\er)) \eps_{n, m} \leq 1.$$
Condition on the event that the hypothesis of Theorem~\ref{lemma:ddtphi}  holds. Then for any $t \leq T$, we have 
\begin{align}
\mathbb{E}_i \|\dit\| \leq \Phi_{\mq}(t) \leq \exp(Ct\javg(\tau)/(\clsc \er)) C\jmax t\eps_{n, m}.
\end{align}
\end{corollary}
\begin{proof}[Proof of Corollary~\ref{cor:main}]
We will use real induction (see eg. ~\cite[Theorem 2]{clark2012instructor}). Our inductive hypothesis will be that for some $t$,
\begin{align}\label{eq:ind2}
    \jmax^2 \left(\int_{s = 0}^{t} \Phi_{\mq}(s)^2ds\right) \leq \frac{1}{2}\eps_{n, m}.
\end{align}
Note that is implies the assumption in Equation~\eqref{eq:ind}. Clearly this holds for $t = 0$. Since $\Phi_{\mq}(s)$ is continuous, if Equation~\eqref{eq:ind2} holds for all $s < t$, it also holds for $t$. This is the continuity assumption. Finally, for the inductive step, we will show that if Equation~\eqref{eq:ind2} holds for some $s$, then for some $\iota$ small enough, it holds at $s + \iota$. To show this, first we use Lemma~\ref{lemma:pdesoln} (which bounds the solution of the ODE given in Theorem~\ref{lemma:ddtphi}), to show that for all $s' \leq s$, 
\begin{align}
    \Phi_{\mq}(s') \leq \exp(Cs'\javg(\tau)/(\clsc\er))Cs\jmax \eps_{n, m} + \eps_{n, m} \leq \left(\exp(Cs\javg(\tau)/(\clsc\er))Cs\jmax\right)\eps_{n, m}.
\end{align}
Note that $\Phi_{\mq}(t)$ is continuous. Thus for $\iota$ small enough, we have $\Phi_{\mq}(t) \leq \Phi_{\mq}(s) + \eps_{n, m}$ for all $t \in [s, s+ \iota]$.
It follows that for $\iota$ small enough, for $t \in [s, s + \iota]$,
\begin{align}
    \int_{s' = 0}^{t} (\Phi_{\mq}(s'))^2ds' &\leq (Cs\jmax \eps_{n, m})^2\int_{s' = 0}^{s} \exp(2Cs\javg(\tau)/(\clsc\er)) ds' + \int_{s' = s}^{t} (\Phi_{\mq}(s) + \eps_{n, m})^2 ds'\\
    &\leq 2(Cs\jmax \eps_{n, m})^2 s\exp(2C\javg(\tau) s/(\clsc\er)).
\end{align}

Now using the assumption in the corollary that 
\begin{align}
    4\jmax^4C^2T^3 \exp(2C\javg(\tau) T/(\clsc\er)) \eps_{n, m} \leq 1,
\end{align}
it follows that $\int_{s = 0}^{s'} (\Phi_{\mq}(s))^2ds \leq \frac{\eps_{n, m}}{2\jmax^2}$

This proves the inductive step. Thus by real induction, the hypothesis in Eq~\eqref{eq:ind2} holds up to time $T$. The result of the lemma then holds by applying Lemma~\ref{lemma:pdesoln} to the result of Theorem~\ref{lemma:ddtphi} at any time $t \leq T$.
\end{proof}

\begin{proof}[Proof of Theorem~\ref{lemma:ddtphi}]
Recall from Lemma~\ref{lemma:errdynamicsbody} that 
\begin{align}
    \frac{d}{dt}\dit &= D^{\perp}_t(i) \dit - \mathbb{E}_{j}H^{\perp}_t(i, j) \djt + {\bm{\epsilon}_{t, i}},
\end{align}
where 
\begin{align}
    \|{\bm{\eps}_{t, i}}\| \leq 2\eps_{n, m} + 2\creg\left(\|\dit\|^2 + \mathbb{E}_j\|\djt\|^2\right).
\end{align}

Now we have
\begin{align}\label{eq:change}
\frac{d}{dt}\Phi_{\mq}(t) &\leq \langle{\nabla \Phi_{\mq}(t), \frac{d}{dt}\Delta_t}\rangle\\
&= -\langle{\nabla \Phi_{\mq}(t), \hpt\Delta_t}\rangle + \langle{\nabla \Phi_{\mq}(t), \dpt \odot \Delta_t}\rangle + \langle{\nabla \Phi_{\mq}(t), \mathcal{E}}\rangle,
\end{align}
where $\me(i) = {\bm{\eps}_{t, i}}$. We will consider the terms in order. Let 
\begin{align}\label{eq:tauchoices}
    \tau := \frac{\clsc \cdot \text{rd}(\er)}{8 (\chdec + \cdecd )},
\end{align} 
where $\text{rd}(z)$ is a rounding of $z$ to its first non-zero decimal, in binary (so $\text{rd}(z) \in [z/2, 2z]$).

Now by Lemma~\ref{lemma:dec1body}, we have
\begin{align}
    -\langle{\nabla \Phi_{\mq}(t), \hpt\Delta_t}\rangle &= -\langle{\nabla \Phi_{\mq}(t), \Delta_t}\rangle_{\hpt} \leq (1 + \cbal)\mathbb{E}_i \|m_t(i)\|\mathbf{1}(\xi_t(w_i) \notin B_{\tau}) + \mathcal{E}_{\ref{lemma:dec1body}},
\end{align}
where $m_t(i) = \mathbb{E}_j \hpt(i, j)\djt$, and
\begin{align}
    \mathcal{E}_{\ref{lemma:dec1body}} &= \chdec(\mathbb{E}_i \|\dit\|\mathbf{1}(\xi_t(w_i) \notin B_{\tau}) + (\tau + \cbal\eps_m^{\ref{lemma:conctau}}) \Omega(t)).
\end{align}

Next by Lemma~\ref{lemma:dec2body}, we have
\begin{align}
\langle{\nabla \Phi_{\mq}(t), \dpt \odot \Delta_t}\rangle \leq -\left(\frac{\clsc\er}{2} - \tau \cdecd \right)\Phi_{\mq}(t) + \cdecd \mathbb{E}_i \|\dit\|\mathbf{1}(\xi_t(w_i) \notin B_{\tau}) + \cbal \mathbb{E}_i \|\dit\|^2.
\end{align}
Here we have used the fact that since the loss is decreasing, the loss in Lemma~\ref{lemma:dec2body} is less than the loss $\delta^2$ at time $T$.

Putting these together, and employing Lemma~\ref{lemma:l1pertbody}, yields
\begin{align}\label{eq:masterdec}
\frac{d}{dt}\Phi_{\mq}(t) &\leq  \left(-\frac{\clsc \er}{4}\right)\Phi_{\mq}(t)\\
&\qquad + (\chdec + \cdecd) \mathbb{E}_i \|\dit\|\mathbf{1}(\xi_t(w_i) \notin B_{\tau})\\
&\qquad + (1 + \cbal)\mathbb{E}_i \|m_t(i)\|\mathbf{1}(\xi_t(w_i) \notin B_{\tau})\\
&\qquad + (1 + 2\cbal)\mathbb{E}_i \|\mathbf{\eps}_{t, i}\|,
\end{align}
where here we used that $\tau$ was chosen such that $(\cdecd + \chdec)(\tau + \cbal \eps_m^{\ref{lemma:conctau}}) \leq \frac{\clsc\er}{8}$, and trivially, $\Omega(t) \leq \Phi_{\mq}(t)$. We also bounded $\mathbb{E}_i\|\dit\|^2$ by $\mathbb{E}_i \|\mathbf{\eps}_{t, i}\|$.


Now let us consider the term $\mathbb{E}_i \|m_t(i)\| \mathbf{1}(\xi_t(w_i) \notin B_{\tau})$. Using Lemma~\ref{lemma:averaging}, we have
\begin{align}
\mathbb{E}_i \|m_t(i)\|\mathbf{1}(\xi_t(w_i) \notin B_{\tau}) &\leq (1 + \cbal)\left(\eps_{n, m} + \javg(\tau)\right)\Phi_{\mq}(t).
\end{align}

Now let use consider the term $\mathbb{E}_i \|\dit\|\mathbf{1}(\xi_t(w_i) \notin B_{\tau})$. Recall from 
Equation~\eqref{eq:duhamel} that 
\begin{align}
\dit = -\int_{s = 0}^t J_{t, s}(i)m_s(i)ds + \int_{s = 0}^t J_{t, s}(i) \bm{\eps_{s, i}}ds.
\end{align}   

Thus by Lemma~\ref{lemma:averaging}, we have
\begin{align*}
\mathbb{E}_i \|\dit\|\mathbf{1}(\xi_t(w_i) \notin B_{\tau}) &\leq (1 + \cbal)\left(\eps_{n, m} + \javg(\tau)\right)\int_{s = 0}^t \Phi_{\mq}(s)ds \\
&\qquad + \int_{s = 0}^t \mathbb{E}_i \| J_{t, s}(i) \bm{\eps_{s, i}}\|\mathbf{1}(\xi_t(w_i) \notin B_{\tau}) ds.
\end{align*}

Plugging this back into Equation~\eqref{eq:masterdec} yields
\begin{align}
\frac{d}{dt} \Phi_{\mq}(t) &\leq - \frac{\clsc\er}{5}\Phi_{\mq}(t) + (\chdec + 4\sqrt{\cbal}\creg) (1 + \cbal)(\eps_{n, m} + \javg(\tau))\int_{s = 0}^t \Phi_{\mq}(s)ds\\
&\qquad + (1 + \cbal)\mathbb{E}_i \|\bm{\eps}_{t, i}\| + (1 + \cbal)(\chdec + 4\sqrt{\cbal}\creg)\int_{s = 0}^t \mathbb{E}_i \| J_{t, s}(i) \bm{\eps_{s, i}}\| ds\\
&\leq - \frac{\clsc\er}{5}\Phi_{\mq}(t) + C \javg(\tau) \int_{s = 0}^t \Phi_{\mq}(s)ds\\
&\qquad + (1 + \cbal)\mathbb{E}_i \|\bm{\eps}_{t, i}\| + C\int_{s = 0}^t \mathbb{E}_i \| J_{t, s}(i) \bm{\eps_{s, i}}\| ds,
\end{align}
where $C = O_{\creg, \cbal}(1)$. Let us simplify the error terms. Appealing to Lemma~\ref{lemma:inductquad}, we have for all $i$, $\|\dit\|^2 \leq 4\eps_{n, m}$ and $E_{t, i} := \int_{s = 0}^t \| J_{t, s}(i) \bm{\eps_{s, i}}\|ds \leq 8\jmax t\eps_{n,m}$.

Thus 
\begin{align}
    \mathbb{E}_i \|\bm{\eps}_{t, i}\| \leq 2\eps_{n, m} + 4\creg\mathbb{E}_i \|\dit\|^2 \leq 18\creg\eps_{n, m},
\end{align}
and 
\begin{align}
    \int_{s = 0}^t \mathbb{E}_i \| J_{t, s}(i) \bm{\eps_{s, i}}\| ds = \mathbb{E}_i E_{t, i} \leq 8\jmax t\eps_{n,m}.
\end{align}
Thus plugging this back into the bound on the dynamics, we have
\begin{align}
\frac{d}{dt} \Phi_{\mq}(t) &\leq - \frac{\clsc\er}{5}\Phi_{\mq}(t) + C \javg \int_{s = 0}^t \Phi_{\mq}(s)ds + C\jmax t\eps_{n,m} ds,
\end{align}
where $C = O_{\creg, \cbal}(1)$.
\end{proof}

\begin{lemma}[Inductive Squared Error Bound.]\label{lemma:inductquad}
Suppose Assumption~\ref{assm:growth} hold with value $\jmax$. Suppose for all $t' \leq t$, we have
\begin{align}
    \jmax^2 \left(\int_{s = 0}^{t'} \Phi_{\mq}(s)^2ds\right) \leq \eps_{n, m}.
\end{align}
and $\jmax^2t^2\eps_{n, m} \leq \frac{1}{64}$.
Then for all $i$ and $t' \leq t$, we have
\begin{align}
    \|\Delta_{t'}(i)\|^2 &\leq 4\eps_{n, m} \\
    E_{t, i} &:= \int_{s = 0}^t \| J_{t, s}(i) \bm{\eps_{s, i}}\|ds \leq 8\jmax t \eps_{n, m},
\end{align}
where $\bm{\eps_{s, i}}$ is defined in Lemma~\ref{lemma:errdynamicsbody}.
\end{lemma}
\begin{proof}
It suffices to prove the statement just for the final time $t$, because we could always apply the lemma with a smaller value of $t$. Recall that
\begin{align}
\bm{\eps_{s, i}} \leq 2\eps_{n, m} + 2\creg\left(\|\dit\|^2 + \mathbb{E}_j\|\djt\|^2\right).
\end{align}

Since
\begin{align}
\mathbb{E}_i \|\bm{\eps}_{t, i}\| \leq 2\eps_{n, m} + 4\creg\mathbb{E}_i \|\dit\|^2,
\end{align}
by Equation~\eqref{eq:duhamel}, we have
\begin{align}
\|\dit\| &\leq \int_{s = 0}^t J_{t, s}(i)(m_s(i) +  \bm{\eps}_{s, i})ds \\
&\leq \int_{s = 0}^t \|J_{t, s}(i)m_s(i)ds\| +  \int_{s = 0}^t \|J_{t, s}(i)\bm{\eps}_{s, i}ds\|\\
&= \int_{s = 0}^t \|J_{t, s}(i)m_s(i)ds\| +  E_{t, i}\\
&\leq \sqrt{\int_{s = 0}^t \|J_{t, s}(i)\|^2ds} \sqrt{\int_{s = 0}^t \|m_s(i)\|^2ds} +  E_{t, i}\\
&\leq \jmax\sqrt{\int_{s = 0}^t \|m_s(i)\|^2ds} +  E_{t, i}\\
&\leq \jmax\sqrt{\int_{s = 0}^t \Phi_{\mq}(s)^2ds} +  E_{t, i}\\
&\leq \sqrt{\eps_{n,m}} +  E_{t, i},
\end{align}
Here in the second last inequality, we used the fact that $\|m_s(i)\| \leq \Phi_{\mq}(s)$ for any $i$, \mg{Might want to have a formal lemma for this? I think I used it in the potential proof too.} and in the last line, we used assumption of the lemma. Note that this same calculation holds for all $s \leq t$, so we have
\begin{align}
   \|\Delta_s(i)\| \leq \sqrt{\eps_{n,m}} +  E_{t, i}.
\end{align}

Now lets bound $E_{t, i}$:
\begin{align}
E_{t, i} &:= \int_{s = 0}^t \| J_{t, s}(i) \bm{\eps_{s, i}}\| ds \leq \int_{s = 0}^t \| J_{t, s}(i)\|\left(2\eps_{n, m} + 4\creg \max_j\|\djs\|^2\right) ds\\
&\leq \jmax \int_{s = 0}^t \left(2\eps_{n, m} + \max_j \left(2\eps_{n, m} + 2 E_{t, j}^2\right) \right)ds,
\end{align}
where in the second line, we plugged in the bound on $\dis$.

Thus letting $E_t := \max_j E_{t, j}$, we have
\begin{align}
    E_t &\leq 2\jmax t\left(2 \eps_{n, m} + E_{t}^2\right)
\end{align}
Now assuming the discriminant $1 - 32\jmax^2t^2\eps_{n, m} > 0$, this equation has two sets of disjoint solutions, one small (including $0$) and one large:
\begin{align}
    E_t \in \left[-\infty, \frac{1 - \sqrt{1 - 32\jmax^2t^2\eps_{n, m}}}{4\jmax t}\right] \cup \left[\frac{1 + \sqrt{1 - 32\jmax^2t^2\eps_{n, m}}}{4\jmax t} , \infty\right]
\end{align}

Note that since at time $t = 0$, we have $E_t = 0$, and $E_t$ is continuous, it must be that if the discriminant is positive up to time $t$, the solution is always in the first set. Indeed, since an assumption of the lemma is that $\jmax^2t^2\eps_{n, m} \leq \frac{1}{64}$. Thus we have
\begin{align}
    E_t \leq \frac{1 - \sqrt{1 - 32\jmax^2t^2\eps_{n, m}}}{4\jmax t} \leq 8\jmax t \eps_{n, m}.
\end{align}
Plugging this back above into our bound on $\dit$ yields that for all $i$,
\begin{align}
    \|\dit\|^2 \leq 4 \eps_{n, m}.
\end{align}
\end{proof}

\begin{lemma}[ODE Analysis]\label{lemma:pdesoln}
Suppose we have a differential equation of the form 
\begin{align}
\frac{d}{dt} X_t &\leq - a X_t + b\int_{s = 0}^t X_s ds + \eps.
\end{align}
with initial condition $X_0 = 0$ and $a, b \geq 0$. Then
\begin{align}
    X_t \leq \exp(bt/a) \frac{\eps}{\sqrt{a^2 + 4b}}.
\end{align}
\end{lemma}
\begin{proof}
Let $Y_t$ solve the ODE 
\begin{align}
   \frac{d}{dt} Y_t = - a Y_t + b\int_{s = 0}^t Y_s ds + 2\eps,
\end{align}
with initial condition $Y_0 = 0$, and let $Z_t = X_t - Y_t$. We will show that $Z_t$ never goes above $0$.

Observe that $Z_t$ solves the differential equation
\begin{align}
   \frac{d}{dt} Z_t &\leq - a Z_t + b\int_{s = 0}^t Z_s ds - \eps,
\end{align}
with initial condition $Z_t = 0$. One can check by the \em real induction \em that $Z_t \leq 0$. Indeed, if $Z_s \leq 0$ for all $s < t$, then we have $Z_t \leq 0$. Further, since $Z_t$ is continuous, if the hypothesis $Z_{t} \leq 0$ holds up to time $s$, we can show that it holds at time $s + \iota$ for some $\iota > 0$. Indeed, for $\iota$ small enough (in terms of $b$ and $\eps$), for all $r \in [s, s + \iota]$, we have $Z_r \leq \frac{\eps}{b}$. Thus for $r \in [s, s + \iota]$, we have $\frac{d}{dr}Z_r \leq -a Z_r + b \iota \left(\frac{\eps}{b}\right) - \eps \leq -aZ_r$ for $\iota \leq 1$. Then Gronwall's inequality gives that $Z_{s + \iota} \leq Z_s \leq 0$, which is the inductive step. This yields the claim that $Z_t \leq 0$ for all $t > 0$.

Now we just need to solve the differential equation for $Y_t$. Taking a second derivative, we have
\begin{align}
   Y''_t = - a Y'_t + b Y_t.
\end{align}
A standard second order ODE analysis yields that 
\begin{align}
    Y_t = C_1 \exp(r_1 t) + C_2\exp(r_2 t), 
\end{align}
where $r_1$ and $r_2$ are the roots of $x^2 + ax - b = 0$, that is,
\begin{align}
   (r_1, r_2) = \frac{-a \pm \sqrt{a^2 + 4b}}{2}
\end{align}
Checking the initial condtitions of $Y_0$ and $Y'_0$ yields 
\begin{align}
    Y_t = \left(\frac{\eps}{\sqrt{a^2 + 4b}}\right)\left(\exp(r_1 t) - \exp(r_2 t)\right), 
\end{align}
where $r_1$ is the larger root. Since $r_1 \leq \frac{b}{a}$, we obtain the lemma.
\end{proof}

%% file: Appendices/examples_proofs.tex
\section{Applications to Learning a Single-index Model}
\subsection{Setting}\label{sec:SIMsetting}
We will study the setting of learning a well-specified even single index function $f^*(x) = \sigma(x^\top w^*)$, where $w^* \in \sd$, and $\sigma(z) = \sum_{k = k^*}^K c_k \hek(z)$, where:
\begin{enumerate}
    \item $k^* \geq 4$, and $\frac{1}{\csim} \leq c_{k^*} \leq \csim \max_k {c_k}$.
    \item All $k$ with $c_k \neq 0$ are even. (That is, $\sigma$ is an even function).
\end{enumerate}
We assume the initial distribution $\rho_0$ of the neurons is uniform on $\sd$, and the data is drawn i.i.d from the distribution $\md$, which has Gaussian covariates, and subGaussian label noise: that is, 
\begin{align}
    x &\sim \mathcal{N}(0, I_d) \sim \md_x\\
    y &= f^*(x) + \zeta(x),
\end{align}
where $\zeta(x)$ has mean $0$ and is $1$-subGaussian.

We will prove the following theorem, which we restate from Theorem~\ref{theorem:SIM} in the main body.
\thmsim*

We will prove Theorem~\ref{theorem:SIM} by (1) analyzing the MF dynamics to show the convergence of $\rtmf$, and then (2) checking the assumptions of Theorem~\ref{thm:main} hold, and applying it to show the convergence of $\rtm$.

\paragraph{Notation}
Define $\alpha(w) := |w^\top w^*|$. Let $v(\alpha, t)$ denote the velocity of a particle $w$ with $\alpha(w) = \alpha$ in the $w^* \on{sign}(w^\top w^*)$ direction. Formally, we have
\begin{align}
    v(\alpha, t) := \langle w^*, \nu(w, \rtmf) \rangle \on{sign}(w^\top w^*),
\end{align}
for any $w$ with $\alpha(w) = \alpha$. We will often use the notation $\alpha \sim \rho$ or $\alpha' \sim \rho$ to denote the distribution of $\alpha(w)$ with $w \sim \rho$. We use $\alpha_t(w) := \alpha(\xi_t(w))$. We use $\xi_{t, s}(w)$ denote the location of the particle at time $t$ which is initialized at $w$ at time $s$. In this language, we have that $\xi_{t}(w) = \xi_{t, 0}(w)$. We similarly define $\alpha_{t, s}(\beta)$ to be $\alpha(\xi_{t, s}(w))$ for any $w$ with $\alpha(w) = \beta$.

We will use $q_{\sigma}$ to denote the polynomial with $k$th coefficient $k! c_{k}^2$, where $\sum c_k \hek(z)$ is the Hermite decomposition of $\sigma$. Similarly, we denote $q_{\sigma'}(z) = \sum_{k = k^* - 1}^{K-1} c_{k + 1}^2(k + 1)(k + 1)!z^k$. From the Hermite polynomial identity that $\mathbb{E}_x \hek(w^\top x)\hej(v^\top x) = k! \delta_{jk}(w^\top v)^k$, we have
\begin{align}
    \mathbb{E}_x \sigma(w^\top x)\sigma(v^\top x) &= q_{\sigma}(w^\top v).\\
    \mathbb{E}_x \sigma'(w^\top x)\sigma'(v^\top x) &= q_{\sigma'}(w^\top v).
\end{align}

\subsection{Bounds on the Velocity and its Derivative}
\begin{figure}[htbp]\label{fig:SIM}
    \centering
    \begin{minipage}{1\textwidth}
        \centering
        \includegraphics[width=0.5\textwidth]{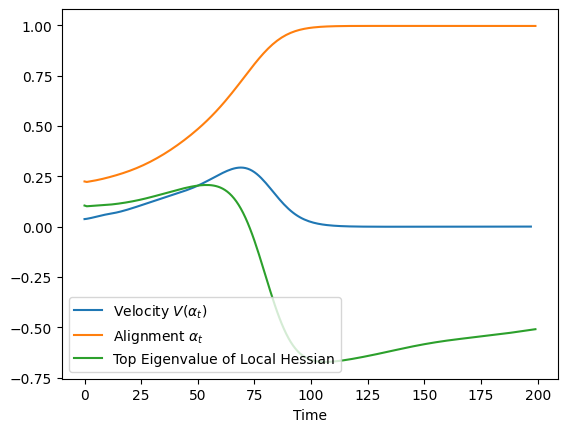}
        \vspace{-2mm}
        \caption{\small Self-Concordance Property: the top eigenvalue of the Local Hessian is Bounded by $\frac{k - 1}{\alpha_t}\nu(\alpha_t)$}
    \end{minipage}
\end{figure}

The key ingredients in both the MF convergence analysis, the perturbation analysis (bounding $\jmax$ and $\javg$), and in showing local strong convexity, is obtaining a lower bound on the particle velocity, and bounds on the local Hessian, $D^{\perp}_t(w)$. 
It turns out, it is much easier to bound these quantities under a certain inductive assumption (which in our MF analysis we will prove holds). We define the inductive property with parameter $\iota$ to hold at time $t$ if
\begin{align}
    \mathbb{P}_{w \sim \rtmf}\left[\alpha(w) \in [\iota, 1 - \iota] \right] \leq \iota.  \tag{$\star$} \label{eq:starred}
\end{align}
Eventually, we will choose $\iota$ to be some small constant dependent on the desired final loss $\delta$.

\begin{lemma}[Lower Bound of Velocity]\label{lemma:SIMv}
Let $\er := \sqrt{\mathbb{E}_x (f_{\rtmf}(x) - f^*(x))^2}.$ Suppose \eqref{eq:starred} holds at time $t$ for $\iota \leq \min\left(\Theta_K(1), \delta^{6K^2}\right)$. Then
\begin{align*}
    v(\alpha, t) &\geq q_{\sigma'}(\alpha)(1 - \alpha^2)(1 - r_t)- O_K\left((1 - \alpha^2)\mathcal{R}_{\alpha}\right),
\end{align*}
where $\mathcal{R}_{\alpha} = O_K(\iota(\alpha \sqrt{d}^{-\max(2, k^*-2)} + \alpha^{\max(1, k^*-3)}\sqrt{d}^{-2}) + \alpha\sqrt{d}^{-k^*})$, and $r_t = \mathbb{E}_{\alpha' \sim \rtmf} (\alpha')^{k^*} = \Omega_K(\delta)$. In particular, if $\alpha \geq \frac{\delta^{3K}}{\sqrt{d}}$, for $d$ large enough (in terms of $\delta, K$), we have that 
\begin{align}
    v(\alpha, t) \geq q_{\sigma'}(\alpha)(1 - \alpha^2)(1 - r_t)(1 - \sqrt{\iota}).
\end{align}
\end{lemma}
\begin{proof}
Let us expand the velocity by expressing $v(\alpha, t)$ as a polynomial in terms of $\alpha$. 
Fix $w$ with $\alpha(w) = \alpha$ and without loss of generality assume $w^\top w^* > 0$. For $w' \in \mathbb{S}^{d-1}$, we denote $w' = \alpha' w^* + y$, where $y' \in \sqrt{1 - \alpha'^2}\mathbb{S}^{d-2}$, which we will use to denote the sphere perpendicular to $w^*$ of radius $\sqrt{1 - \alpha'^2}$. We expand
\begin{align}\label{eq:velocity1}
    \nu(w, \rtmf)^\top  w^* &= \mathbb{E}_x (f^*(x) - f_{\rtmf}(x))\sigma'(w^\top x)x^\top  P^{\perp}_w w^* \\
    &= \mathbb{E}_x \sigma({w^*}^\top x)\sigma'(w^\top x)x^\top  P^{\perp}_w w^* - \mathbb{E}_{w' \sim \rtmf}\mathbb{E}_x \sigma(w'^\top x)\sigma'(w^\top x)x^\top  P^{\perp}_w w^*\\
    &= q_{\sigma'}(w^\top w^*) {w^*}^\top P^{\perp}_w w^*  - \mathbb{E}_{w' \sim \rtmf} q_{\sigma'}(w^\top w') (w')^\top P^{\perp}_w w^*\\
    &= q_{\sigma'}(\alpha)(1 - \alpha^2) - \mathbb{E}_{w' \sim \rtmf} q_{\sigma'}(w^\top w') (w')^\top P^{\perp}_w w^*\\
    &= q_{\sigma'}(\alpha)(1 - \alpha^2) - \mathbb{E}_{\alpha' \sim \rtmf} \mathbb{E}_{y' \sim \sqrt{1 - \alpha'^2}\mathbb{S}^{d-2}}q_{\sigma'}(\alpha \alpha' + y'^\top w)\left(\alpha'(1 - \alpha^2) - \alpha  y'^\top w \right).
\end{align}
Here in the fifth equality, we used the rotational symmetry of $\rtmf$ about the $w^*$ axis.

Lets break down this expression. Let 
\begin{align}
    r_{t, k} : = \mathbb{E}_{\alpha \sim \rtmf}\alpha^k.
\end{align}

Fix a (necessarily odd) coefficient $k^* - 1 \leq k \leq K - 1$ of the polynomial $q_{\sigma'}(z) := \sum q_k z^k$, and consider all terms in the above equation arising from that order term:
\begin{align}
    &q_k\alpha^{k}(1 - \alpha^2) - q_k\sum_{j = 0}^{k} \binom{k}{j}(\alpha \alpha')^{j} \mathbb{E}_{y' \sim \sqrt{1 - \alpha'^2}\mathbb{S}^{d-2}} (y'^\top w)^{k-j}(\alpha'(1 - \alpha^2) - \alpha  y'^\top w) \\
    &\qquad = q_k\alpha^{k}(1 - \alpha^2)\left(1 - r_{t, k+1}\right) + \mathbb{E}_{\alpha' \sim \rtmf}\mathcal{E}_{\alpha, \alpha', k},
\end{align}
where 
\begin{align}
    \mathcal{E}_{\alpha, \alpha', k} &= \begin{cases}
        O_{k}\left((1 - \alpha^2)(\alpha')^2(1 - \alpha'^2)(\alpha \sqrt{d}^{-(k-1)} + \alpha^{k-2}\sqrt{d}^{-2}) + \alpha(1 - \alpha^2)\sqrt{d}^{-(k+1)}\right) & k \geq 3\\
        0 & k = 1
    \end{cases}.
\end{align}
Note here that we have used the fact that $k$ is even and $\mathbb{E}_{y'} (y'^\top w)^j = O_j(\left((1 - \alpha'^2)(1 - \alpha^2)d^{-1}\right)^{j/2})$, and is $0$ for odd $j$. The final error terms arises from the fact that we have only counted the terms in the binomial expansion which could be most significant --- depending on the relative size of $\alpha \alpha'$ and $\sqrt{(1 - \alpha^2)(1 - \alpha'^2)}/\sqrt{d}$. Now plugging in the hypothesis \eqref{eq:starred}, we have that $\mathbb{E}_{\alpha' \sim \rtmf}(\alpha')^2(1 - \alpha'^2) \leq 2\iota$, so for all $k$,
\begin{align}
    \mathcal{E}_{\alpha, \alpha', k} &= 
        O_{k}\left((1 - \alpha^2)\iota(\alpha \sqrt{d}^{-(k-1)} + \alpha^{k-2}\sqrt{d}^{-2}) + \alpha(1 - \alpha^2)\sqrt{d}^{-(k+1)}\right) \leq (1 - \alpha^2)\mathcal{R}_{\alpha}
\end{align}
Summing over all odd $k^* - 1 \leq k \leq K - 1$ yields that 
\begin{align}\label{eq:vfull}
    v(\alpha, t) &= \sum_{k = k^* - 1}^{K - 1}q_k \alpha^k(1 - \alpha^2)(1 - r_{t, k+1}) +  (1 - \alpha^2)\mathcal{R}_{\alpha}\\
    &\geq q_{\sigma'}(\alpha)(1 - \alpha^2)(1 - r_t) - (1 - \alpha^2)\mathcal{R}_{\alpha},
\end{align}
where here in the inequality, we used the fact that $r_t = \mathbb{E}_{\alpha' \sim \rtmf}(\alpha')^{k^*} \geq \mathbb{E}_{\alpha' \sim \rtmf}(\alpha')^{k} = r_{t, k}$ for all $k \geq k^*$.
Now for $\alpha \geq \frac{\delta^{3K}}{\sqrt{d}}$, we have
\begin{align}
    v(\alpha, t) &\geq q_{\sigma'}(\alpha)(1 - \alpha^2)(1 - r_t)\\
    &\qquad - O_K\left(\iota(1 - \alpha^2)(\alpha^{k^*-1}\delta^{-3K(k^* - 2)} + \alpha^{k^*-1}\delta^{-6K}) + (1 - \alpha^2)\alpha^{k^* - 1}\delta^{-3K(k^* - 2)}/d\right),
\end{align}
Since by Lemma~\ref{lemma:rt}, we have $(1 - r_t) = \Omega(\delta)$, it follows that 
\begin{align}
    v(\alpha, t) &\geq q_{\sigma'}(\alpha)(1 - \alpha^2)(1 - r_t)(1 - \sqrt{\iota}).
\end{align}
\end{proof}

In the following lemma, we analyze $\frac{d}{d \alpha} v(\alpha, t) $. As will be shown in Section~\ref{sec:SIMpert}, bounding $\frac{d}{d \alpha} v(\alpha, t)$ is useful in bounding $D^{\perp}_t(w)$. The second part of this lemma will also be instrumental in proving local strong convexity (Definition~\ref{def:lsc}). 
\begin{lemma}\label{lemma:selfconcordance}
Let $\er := \sqrt{\mathbb{E}_x (f_{\rtmf}(x) - f^*(x))^2}.$ Suppose \eqref{eq:starred} holds at time $t$ for $\iota \leq \min\left(\Theta_K(1), \delta^{6K^2}\right)$. Then
\begin{align} 
    \frac{d}{d \alpha} v(\alpha, t)  \begin{cases}
        = \frac{k^* - 1}{\alpha} v(\alpha, t) + \mathcal{E}_{\alpha} & \alpha \leq 1;\\
         \leq - \frac{\alpha}{1- \alpha^2}v(\alpha, t) - \Omega_K(\delta) & \alpha \geq 1 - \frac{1}{5K},
    \end{cases}
\end{align}
where $\mathcal{E}_{\alpha} := \Theta_{K}\left(\alpha^{k^*} + \iota(\sqrt{d}^{-(k^* - 2)} + \alpha^{k^* - 4
    }\sqrt{d}^{-2}) + \sqrt{d}^{-(k^* - 2)}\right)$.
\end{lemma}
\begin{proof}
First we compute $\frac{d}{d \alpha} v(\alpha, t)$. Fix a coefficient $k^* - 1 \leq k \leq K - 1$ of the polynomial $q_{\sigma'}$, and consider all terms in the the derivative of Equation~\eqref{eq:velocity1} arising from that order term:
\begin{align}\label{eq:dvda}
    &q_k k \alpha^{k-1}\left(1 - \frac{k+2}{k}\alpha^2\right)\\
    &\qquad - q_k\sum_{j = 0}^{k} \binom{k}{j} j(\alpha \alpha')^{j-1} \mathbb{E}_{y' \sim \mathbb{S}^{d-2}_{\sqrt{1 - \alpha'^2}}} (y'^\top w)^{k-j}\left(\alpha'\left(1 - \frac{j+2}{j}\alpha^2\right) + \frac{j+1}{j}\alpha  y'^\top w\right) \\
    &= k q_k\alpha^{k-1}\left(1 - \frac{k+2}{k}\alpha^2\right)(1 - r_{t, k+1}) + \mathcal{E}_{\alpha, k}, 
\end{align}
where $\mathcal{E}_{\alpha, k} = \Theta_{k}(\iota(\sqrt{d}^{-(k - 1)} + \alpha^{k - 3
    }\sqrt{d}^{-2}) + \sqrt{d}^{-(k + 1)})$, and $r_{t, k} = \mathbb{E}_{\alpha' \sim \rtmf}(\alpha')^k$.
    
Here we have used the same computations as in the proof of Lemma~\ref{lemma:SIMv}. Summing over all odd $k^* - 1 \leq k \leq K - 1$ yields
\begin{align}\label{eq:dfull}
    \frac{d}{d \alpha} v(\alpha, t) &= \sum_{k = k^* - 1}^K q_k k \alpha^{k-1}\left(1 - \frac{k+2}{k}\alpha^2\right)(1 - r_{t, k+1}) + \mathcal{E}_{\alpha, k}\\
    &= (k^* - 1)q_{k^* - 1} (1 - r_t) \alpha^{k^*-2} + \Theta_{K}\left(\alpha^{k^*} + \iota(\sqrt{d}^{-(k^* - 2)} + \alpha^{k^* - 4
    }\sqrt{d}^{-2}) + \sqrt{d}^{-(k^* - 2)}\right)
\end{align}

Combining Lemma~\ref{lemma:SIMv} with the previous equation, we obtain
\begin{align}
    \frac{d}{d \alpha} v(\alpha, t) = \frac{k^* - 1}{\alpha}v(\alpha, t) + \Theta_{K}\left(\alpha^{k^*} + \iota(\sqrt{d}^{-(k^* - 2)} + \alpha^{k^* - 4
    }\sqrt{d}^{-2}) + \sqrt{d}^{-(k^* - 2)}\right).
\end{align}
This yields the first case in the conclusion of the lemma.

For the case that $\alpha \geq 1 - \frac{1}{5K} \geq \sqrt{\frac{k}{k + 0.5}}$ for all $k \leq K$, we have
\begin{align}
    k\left(1 - \frac{k + 2}{k} \alpha^2\right) \leq - 1.5\alpha^2
\end{align}
We will compare the terms with coefficient $q_k$ in the first line of Equation~\eqref{eq:dfull} and the first line of Equation~\eqref{eq:vfull}. Let 
\begin{align}
    v_k(\alpha, t) := q_k \alpha^k(1 - \alpha^2)(1 - r_{t, k+1}),
\end{align}
such that Equation~\eqref{eq:vfull} gives
\begin{align}
    v(\alpha, t) = \sum_{k = k^* - 1}^{K - 1}v_{k}(\alpha, t) + (1-\alpha^2)\mathcal{R}_{\alpha},
\end{align}
where $\mathcal{R}_{\alpha}$ is as in Lemma~\ref{lemma:SIMv}. 
Thus the first line of Equation~\eqref{eq:dfull} gives
\begin{align}
    \frac{d}{d\alpha}v(\alpha, t) &= \sum_k  \left(v_k(\alpha, t)\right) \frac{1}{
    \alpha(1 - \alpha^2)} k\left(1 - \frac{k + 2}{k\alpha^2}\right) + \mathcal{E}_{\alpha, k}\\
    &\leq \sum_k \left(v_k(\alpha, t)\right) \frac{-1.5\alpha}{
    (1 - \alpha^2)} + \mathcal{E}_{\alpha, k}\\
    &= \frac{-1.5\alpha}{
    (1 - \alpha^2)}\left(v(\alpha, t) - (1 - \alpha^2)\mathcal{R}_{\alpha}\right) + \sum_k \mathcal{E}_{\alpha, k}\\
    &\leq -\frac{\alpha}{1- \alpha^2}v(\alpha, t) - \Omega_K(\delta).
\end{align}
Here in the first inequality, we used the fact that all the $q_k$ (and hence all the $v_k(\alpha, t)$) are non-negative. Indeed, recall that $q_k$ are the coefficients of the polynomial $q_{\sigma'}(z) := \sum_{k=k^* -1}^{K-1}c^2_{k+1}(k + 1)(k+1)!z^k$, where $\sum_k c_k \hek(z)$ is the Hermite decomposition of $\sigma$. In the last inequality, we have used the bounds on $\mathcal{R}_{\alpha}$ and $\mathcal{E}_{\alpha, k}$, along with the fact from Lemma~\ref{lemma:SIMv} that $v(\alpha, t) = \Omega_K((1 - \alpha^2)\delta)$. This yields the desired conclusion.

\end{proof}

A key part of both our MF convergence analysis, and the perturbation analysis is understanding the stability of the $\alpha_t(w)$ with respect to small changes in $\alpha_s(w)$. The following lemma controls this derivative. Define 
\begin{align}\label{eq:lts}
    \ell_{t, s}(w) := \frac{d \alpha_{t, s}(\beta)}{d \beta} \bigg|_{\beta = \alpha_s(w)}
\end{align} 
\begin{lemma}\label{claim:lst2}
Suppose that for all $s \leq t$, we have $\sqrt{\mathbb{E}_x(f_{\rsmf}(x) - f^*(x))^2} \geq \delta$. Suppose $\iota \leq \min\left(\Theta_K(1), \delta^{6K^2}\right)$, and $t \leq \frac{\sqrt{d}^{k^* - 2}}{\iota}$. Finally suppose \eqref{eq:starred} holds for all $s \leq t$.  Then for and $\tau \leq 1/2$ and any $w$ for which $\alpha_t(w) \leq 1 - \tau$, we have
\begin{align}
    \ell_{t, s}(w)  := \frac{d \alpha_{t, s}(\beta)}{d \beta} \bigg|_{\beta = \alpha_s(w)} = \left(\frac{\alpha_t(w)}{\alpha_s(w)}\right)^{k^*-1}\exp\left(O_K\left(\frac{\log(1/\tau)}{\delta}\right)\right).
\end{align}
\end{lemma}
\begin{proof}
Observe that $\ell_{t, s}(w)$ satisfies the differential equation 
\begin{align}\label{eq:l}
    \frac{d}{dt}\ell_{t, s}(w) &= \left(\frac{d}{d \alpha_t(w)} 
    v(\alpha_t(w), t) \right)\ell_{t, s}(w);\\
    \ell_{s, s}(w) &= 1.
\end{align}

From Lemma~\ref{lemma:selfconcordance}, we have that
\begin{align}
    \frac{d}{dt}\ell_{t, s}(w) &= \left((k^*-1)\frac{v(\alpha_t(w), t)}{\alpha_t(w)} + \mathcal{E}_{\alpha}\right)\ell_{t, s}(w);\\
    \frac{d}{dt}\alpha_t(w) &= \frac{v(\alpha_t(w), t)}{\alpha_t(w)} \alpha_t(w),
\end{align}
where we recall that 
\begin{align}
   \mathcal{E}_{\alpha} =  O_K\left(\alpha^{k^*} + \sqrt{d}^{-k^*} + \iota\left(\sqrt{d}^{-(k^* - 2)} + \alpha_t^{(k^* - 4)}\sqrt{d}^{-2}\right)\right)
\end{align}
Equivalently, taking logs, we have
\begin{align}
     \frac{d}{dt}\frac{\log(\ell_{t, s}(w))}{k^* - 1} &= \frac{v(\alpha_t(w), t)}{\alpha_t(w)} + \mathcal{E}_{\alpha};\\
    \frac{d}{dt}\log(\alpha_t(w)) &= \frac{v(\alpha_t(w), t)}{\alpha_t(w)}.   
\end{align}

Let us split up the time interval into (at most 3) intervals: $[s, t_1], [t_1, t_2], [t_2, t]$, where $t_1$ is first moment at which $\alpha_{t_1} \geq \frac{1}{\sqrt{d}}$, and $\alpha_{t_2}$ is the first moment at which $\alpha_{t_2} = 0.5$. In the first interval, we have $ \mathcal{E}_{\alpha} \leq O_K(\iota\sqrt{d}^{(k^* - 2)})$. In the second interval, by Lemma~\ref{lemma:SIMv}, we have $\mathcal{E}_{\alpha} \leq O_K\left(\sqrt{\iota}\frac{v(\alpha, t)}{\alpha_t^3} \sqrt{d}^{-2} + v(\alpha, t)\alpha/\delta\right)$. 

For the first interval, since $t \leq \frac{\sqrt{d}^{k^* - 2}}{\iota}$, we have
\begin{align}
    \int_{r = s}^{t_1} \mathcal{E}_{{\alpha_r}} dr \leq O_K(\iota\sqrt{d}^{-(k^* - 2)})(t_1 - s) \leq O_K(1).
\end{align}
For the second interval, using $u$-substitution, we have
\begin{align}
    \int_{r = t_1}^{t_2} \mathcal{E}_{{\alpha_r}} dr &\leq \frac{O_K(\sqrt{\iota})}{d}\int_{r = t_1}^{t_2} \frac{v(\alpha_r, r)}{(\alpha_r)^3} dr + \int_{r = t_1}^{t_2} O_K(v(\alpha_r, r)\alpha_r/\delta)  dr\\
    &= \frac{O_K(\sqrt{\iota})}{d}\int_{\alpha = \alpha_{t_1}}^{\alpha_{t_2}} \frac{1}{\alpha^3} d\alpha + \int_{\alpha = \alpha_{t_1}}^{\alpha_{t_2}} O_K(\alpha^2/\delta) d\alpha + O_K(\alpha_{t_2}^2)\\
    &= \frac{O_K(\sqrt{\iota})}{d}\left(\frac{1}{2\alpha_{t_1}^2} - \frac{1}{2\alpha_{t_2}^2}\right) \leq O_K(1/\delta).
\end{align}
For the third interval, observe from Lemma~\ref{lemma:SIMv} that during the duration of this interval, $1 - \alpha_r(w)$ decays exponentially with rate $O_K(\delta)$. Thus, the length of this interval is at most $O_K\left(\frac{\log(1/\tau)}{\delta^2}\right)$, so 
\begin{align}
    \int_{r = t_2}^{t} \mathcal{E}_{{\alpha_r}} dr \leq O_K\left(\frac{\log(1/\tau)}{\delta}\right). 
\end{align}


Thus integrating, we obtain
\begin{align}
    \frac{\log(\ell_{t, s}(w)) - \log(\ell_{s, s}(w))}{k^* - 1} = \int_{r = s}^t \frac{v(\alpha_r(w), t)}{\alpha_r(w)}dr + O_K(\log(1/\tau)/\delta).
\end{align}
Plugging in the integration of the differential equation for $\log(\alpha_t(w))$ yields
\begin{align}
    \frac{\log(\ell_{t, s}(w))}{k^* - 1} = \log\left(\frac{\alpha_t(w)}{\alpha_s(w)}\right) + O_K(\log(1/\tau)/\delta).
\end{align}
Multiplying both sides by $k^* - 1$ and exponentiating yields
\begin{align}
    \ell_{t, s}(w) = \left(\frac{\alpha_t(w)}{\alpha_s(w)}\right)^{k^*-1}\exp\left(O_K\left(\frac{\log(1/\tau)}{\delta}\right)\right)
\end{align}
as desired.
\end{proof}

\begin{lemma}\label{lemma:rt}
For $d$ large enough in terms of $\delta = \sqrt{\mathbb{E}_x(f_{\rtmf}(x) - f^*(x))^2}$, we have 
\begin{align}
    1 - \mathbb{E}_{\alpha' \sim \rtmf} (\alpha')^{k^*} \geq \Omega_{K, \creg}\left(\delta\right).
\end{align}
\end{lemma}
\begin{proof}
Observe that 
\begin{align}
    \mathbb{E}_x(f^*(x))^2 = \mathbb{E}_x \sigma({w^*}^\top x)\sigma({w^*}^\top x) = q_{\sigma}(1).
\end{align}
\begin{align}
    \mathbb{E}_xf_{\rtmf}(x)f^*(x) = \mathbb{E}_x \mathbb{E}_{w' \sim \rtmf}\sigma(w'^\top x)\sigma({w^*}^\top x) = \mathbb{E}_{\alpha' \sim \rtmf}q_{\sigma}(\alpha').
\end{align}
Further 
\begin{align}
    \mathbb{E}_x(f_{\rtmf}(x))^2 = \mathbb{E}_x \mathbb{E}_{w, w' \sim \rtmf}\sigma(w^\top x)\sigma({w'}^\top x) = \mathbb{E}_{w, w' \sim \rtmf}q_{\sigma}(w^\top w').
\end{align}
Now for even $k$, we have
\begin{align}
    \mathbb{E}_{w, w' \sim \rtmf}(w^\top w')^k =  \mathbb{E}_{\alpha, \alpha' \sim \rtmf}\mathbb{E}_{\zeta}(\alpha \alpha' + \sqrt{(1 - \alpha^2)(1 - \alpha')^2}\zeta)^k,
\end{align}
where $\zeta$ is $\frac{1}{\sqrt{d}}$-subGaussian. Thus by Minowski's inequality, we have 
\begin{align}
    \mathbb{E}_{w, w' \sim \rtmf}(w^\top w')^k &\leq \left(\left(\mathbb{E}_{\alpha, \alpha' \sim \rtmf}(\alpha \alpha')^k\right)^{1/k} + 
    \frac{O_K(1)}{\sqrt{d}}\right)^k \\
    &\leq \mathbb{E}_{\alpha, \alpha' \sim \rtmf}(\alpha \alpha')^k + \frac{O_K(1)}{\sqrt{d}}\\
    &= \left(\mathbb{E}_{\alpha \sim  \rtmf}\alpha^k\right)^2 + \frac{O_K(1)}{\sqrt{d}}.
\end{align}

It follows that with $q_{\sigma}(z) = \sum_k q_k z^k$, we have
\begin{align}
    \mathbb{E}_x(f_{\rtmf}(x) - f^*(x))^2 &= \mathbb{E}_x (f^*(x))^2 + \mathbb{E}_x (f_{\rtmf}(x))^2 -2 \mathbb{E}f^*(x)f_{\rtmf}(x)\\
    &= \sum_{k = k^*}^K q_k \left(1^k + \left(\mathbb{E}_{\alpha \sim  \rtmf}\alpha^k\right)^2 - 2\mathbb{E}_{\alpha \sim  \rtmf}\alpha^k\right) + \frac{O_K(1)}{\sqrt{d}}\\
    &= \sum_{k = k^*}^K q_k \left(1 - \mathbb{E}_{\alpha \sim  \rtmf}\alpha^k\right)^2  + \frac{O_K(1)}{\sqrt{d}}
\end{align}

Now for all $k > k^*$, with $1 - s := r := \mathbb{E}_{\alpha \sim \rtmf} (\alpha)^{k^*}$, using \eqref{eq:starred}, we have
\begin{align}
    r^{\frac{k}{k^*}} \leq \mathbb{E}_{\alpha' \sim \rtmf} (\alpha')^k, 
\end{align}
so 
\begin{align}
    \left(1 - \mathbb{E}_{\alpha \sim  \rtmf}\alpha^k\right)^2 &\leq \left(1 - r^{k/k^*}\right)^2 = \left(1 - (1 - s)^{k/k^*}\right)^2\\
    &\leq \left(1 - (1 - sk/k^*)\right) = O_K(s^2).
\end{align}
So  
\begin{align}
    \mathbb{E}_x(f_{\rtmf}(x) - f^*(x))^2 &= \mathbb{E}_x (f^*(x))^2 = O_{K, \creg}(1 - \mathbb{E}_{\alpha \sim \rtmf} (\alpha)^{k^*})  + \frac{O_K(1)}{\sqrt{d}},
\end{align}
and thus for $d$ large enough in terms of $\delta = \sqrt{\mathbb{E}_x(f_{\rtmf}(x) - f^*(x))^2}$, we have $1 - \mathbb{E}_{\alpha \sim \rtmf} (\alpha)^{k^*} = O_{K, \creg}(\delta)$ as desired.




\end{proof}

\subsection{MF Convergence Analysis}

\begin{proposition}[Convergence of $f_{\rtmf}$ to $f^*$]\label{prop:SIMmf}
Fix any $\delta$ small enough, and let $\iota = \delta^{6K^2}$. For $d$ large enough, we have
\begin{align}
    T(\delta) := \arg\min\{t : \mathbb{E}_x (f_{\rtmf}(x) - f^*(x))^2 \leq \delta^2\} = O_K(\sqrt{d}^{k^* - 2}\delta^{-(k^* - 1)}).
\end{align}
We also have the following implication (which we will use for the analysis of $\jmax$ and $\javg$) for any $t \leq T(\delta)$ and for any $\tau > 0$:
\begin{align}\mathbb{E}_{w \sim \rtmf}[(\alpha(w))^{k^* - 1}\mathbf{1}(\alpha(w) \leq 1 - \tau)] \leq \sqrt{d}^{-(k^* - 2)}O_{K, \delta}\left(\frac{1}{\tau^{O_K(1)}}\right).
\end{align}
\end{proposition}
\begin{proof}
First we need to prove by induction on $t$ that for all $t \leq T(\delta)$, the hypothesis \eqref{eq:starred} holds. 
First observe that it holds at time $0$, because 
\begin{align}
    \mathbb{P}_{w \sim \sd}[\alpha(w) \geq \iota] \leq \exp(\Theta(d/\iota^2)) \leq \iota
\end{align}
for $d$ large enough. Suppose the hypothesis holds up to some time $s$. We need to show that it holds at time $s + \eps$ for some $\eps$. First note that for $\eps$ small enough, by the continuity of $v(\alpha, t)$ and $\frac{d}{d\alpha} v(\alpha, t)$, the conclusion of Lemma~\ref{lemma:SIMv} and Lemma~\ref{lemma:selfconcordance} still hold up to time $t$. To prove the hypothesis holds at time $t$, our approach will be to non-constructively bound the interval of $I \subset [0, 1]$ for which $\alpha_0(w) \notin I$ implies $\alpha_t(w) \notin [\iota, 1 - \iota]$. We will use the following claim. 
\begin{claim}\label{claim:backwards}
Suppose $\eqref{eq:starred}$ holds up to time $t$.
For any $\tau \leq 1/2$ and $\gamma \leq \frac{1 - \tau}{2}$, we have 
\begin{align}
    \mathbb{P}_{w \sim \rtmf}\left[\alpha(w) \in [\gamma, 1 - \tau] \right] \leq \frac{2}{\gamma^{k^* - 2}}\sqrt{d}^{-(k^* - 2)}\exp\left(O_K\left(\frac{\log(1/\tau)}{\delta}\right)\right)
\end{align} 
\end{claim}
\begin{proof}
We will show that 
\begin{align}
    \mathbb{P}_{w \sim \rtmf}\left[\alpha(w) \in [\gamma, 2\gamma] \right] \leq \frac{1}{\gamma^{k^* - 2}}\sqrt{d}^{-(k^* - 2)}\exp\left(O_K\left(\frac{\log(1/\tau)}{\delta}\right)\right)
\end{align}   
The claim will then follow by summing this bound over $\log_2((1-\tau)/\gamma)$ intervals.

Suppose we have some $w$ and $w'$ with $\alpha_t(w), \alpha_t(w') \in [\gamma, 2\gamma]$. Since the conditions of Lemma~\ref{claim:lst2} hold up to time $t$ for any particle $\tilde{w}$ with $\alpha_0(\tilde{w})$ initialized between $\alpha_0(w)$ and $\alpha_0(w')$, by the mean value theorem, we have that 
\begin{align}
    \alpha_t(w) - \alpha_t(w') &\geq \left|\alpha_0(w) - \alpha_0(w')\right| \min_{\tilde{w}: \alpha_0(\tilde{w}) \in [\alpha_0(w'), \alpha_0(w')] }\left(\frac{\alpha_t(\tilde{w})}{\alpha_0(\tilde{w})}\right)^{k^*-1}\exp\left(O_K\left(\frac{\log\left(\tau/(k^* - 1)\right)}{\delta}\right)\right)\\
    &\geq \left|\alpha_0(w) - \alpha_0(w')\right| \left(\frac{\gamma}{\alpha_0(w')}\right)^{k^*-1}\exp\left(O_K\left(\frac{\log\left(\tau\right)}{\delta}\right)\right),
\end{align}
Thus since $|\alpha_t(w) - \alpha_t(w')| \leq \gamma$, we have that
\begin{align}
    \left|\alpha_0(w) - \alpha_0(w')\right| \leq \frac{1}{\gamma^{k^* - 2}}\left(\alpha_0(w')\right)^{k^* - 1}\exp\left(O_K\left(\frac{\log\left(\tau\right)}{\delta}\right)\right).
\end{align}

We need to upper bound the probability over $\rho_0$ of the set in which $\alpha_0(w')$ and $\alpha_0(w)$ can lie. By the above calculation, the set which $\alpha_0(w')$ and $\alpha_0(w)$ lies in is contained in 

\begin{align}
    I_{\lambda} := \left[\frac{\lambda}{\sqrt{d}}, \frac{\lambda}{\sqrt{d}} + \frac{1}{\gamma^{k^* - 2}}\left(\frac{\lambda}{\sqrt{d}}\right)^{k^* - 1}\exp\left(O_K\left(\frac{\log\left(\tau\right)}{\delta}\right)\right)\right]
\end{align}
for some $\lambda$. Recall that the distribution of $\alpha_0(w)$ under $w \sim \rho_0$ is $\frac{1}{\sqrt{d}}$-subGaussian. Thus
\begin{align}
    \mathbb{P}_{w \sim \rho_0}[\alpha_0(w) \in I_{\lambda}] &\leq \frac{\lambda^{k^* - 1}}{\gamma^{k^* - 2}}\sqrt{d}^{-(k^* - 2)}\exp\left(O_K\left(\frac{\log(1/\tau)}{\delta^2}\right)\right) \left(\exp(-\lambda^2)\right)\\
    &\leq \frac{1}{\gamma^{k^* - 2}}\sqrt{d}^{-(k^* - 2)}\exp\left(O_K\left(\frac{\log(1/\tau)}{\delta}\right)\right).
\end{align}
This proves the claim.
\end{proof}
Plugging $\gamma = \iota$ and $\tau = \iota$ into this claim yields that 
\begin{align}
    \mathbb{P}_{w \sim \rtmf}[\alpha(w) \in [\iota, 1 - \iota]] \leq \frac{2}{\iota^{k^* - 2}}\sqrt{d}^{-(k^* - 2)} \exp\left(O_K\left(\frac{\log(1/\tau)}{\delta}\right)\right) \leq \iota,
\end{align}
where the second inequality holds for $d$ large enough in terms of $\delta$. This proves the inductive step.

Now to prove the convergence guarantee, a standard analysis of the ODE for $\alpha$ (see eg. \cite{damian2023smoothing}) now yields that, for any $w$ with $\alpha_0(w) \geq \frac{\delta^2}{\sqrt{d}}$, we have that 
\begin{align}
    \alpha_t(w) \geq 1 - \frac{1}{2K}
\end{align}
for $t \geq \frac{\Theta(1)}{\delta^2(\alpha_0(w))^{k^* - 2}}$. This arises directly from the fact that Lemma~\ref{lemma:SIMv} guarantees that for $\alpha \geq \frac{\delta^2}{\sqrt{d}}$,
\begin{align}
    v(\alpha, t) \geq \Theta_K(\delta \alpha^{k^* - 1}(1 - \alpha^2)).
\end{align}
After that, it is clear that $1 - \alpha_t(w)$ decays exponentially fast (with rate $\Omega(\delta)$), so for $t \geq \frac{\Theta(1)}{\delta(\alpha_0(w))^{k^* - 2}} + O_K(\log(1/\delta)) = \frac{\Theta(1)}{\delta(\alpha_0(w))^{k^* - 2}}$, we have $1 - \alpha_t(w) \leq \delta/4$.

Now using the initial distribution of $\alpha_0(w)$ with $w \sim \rho_0$, we have that an at least $1 - \delta/4$ fraction of particles have initialization $\alpha_0(w) \geq O_K(\frac{\delta}{\sqrt{d}})$. Clearly once all these particles achieve $1 - \alpha_t(w) \leq 1 - \delta/4$, we will have loss at most $\delta$. Thus occurs at some time at most
\begin{align}
     \frac{\Theta_K(1)}{\delta(\delta\sqrt{d}^{-1})^{(k^* - 2)}} = O_K(\sqrt{d}^{k^* - 2}\delta^{-(k^* - 1)}).
\end{align}
This proves the main statement of the proposition. To prove the additional clause, fix $\tau$. We have
\begin{align}
    \mathbb{E}_{w \sim \rtmf}[(\alpha(w))^{k^* - 1}\mathbf{1}(\alpha(w) \leq 1 - \tau)] &= \int_{\beta = 0}^{1 - \tau} \mathbb{P}_{w \sim \rtmf}[(\alpha(w))^{k^* - 1} \in [\beta, (1 - \tau)] d\beta.\\
    &= \int_{\gamma = 0}^{1 - \tau} \mathbb{P}_{w \sim \rtmf}[\alpha(w) \in [\gamma^{\frac{1}{k^* - 1}}, (1 - \tau)^{\frac{1}{k^* - 1}}] d\gamma.\\
    &\leq \int_{\gamma = 0}^{1 - \tau} \frac{2}{\gamma^{\frac{k^* - 2}{k^* - 1}}}\sqrt{d}^{-(k^* - 2)}\exp\left(O_K\left(\frac{\log(1/\tau)}{\delta}\right)\right) d\gamma \\
    &= \sqrt{d}^{-(k^* - 2)}\exp\left(O_K\left(\frac{\log(1/\tau)}{\delta}\right)\right) \int_{\gamma = 0}^{1 - \tau}  \frac{2}{\gamma^{\frac{k^* - 2}{k^* - 1}}} d\gamma \\
    &= \sqrt{d}^{-(k^* - 2)}\exp\left(O_K\left(\frac{\log(1/\tau)}{\delta}\right)\right) 2 (k^* - 1) \gamma^{\frac{1}{k^* - 1}}\Big|^{1 - \tau}_{0}\\
    &= \sqrt{d}^{-(k^* - 2)}\exp\left(O_K\left(\frac{\log(1/\tau)}{\delta}\right)\right).
\end{align}
Here the inequality follows from Claim~\ref{claim:backwards} and the fact that $(1 - \tau)^{\frac{1}{k^* - 1}} \geq 1 - \frac{\tau}{k^* - 1}$. This proves the additional clause.

\end{proof}

\subsection{Proving Assumptions in Theorem~\ref{thm:main} for Single-index Model} \label{sec:SIMpert}

We need to check that the problem $(f^*, {\md_x}, \rho_0)$ introduced in Section~\ref{sec:SIMsetting} satisfies the Assumptions of Theorem~\ref{thm:main}. 
Fix a desired loss $\delta$, and let $T(\delta)$ be as in Proposition~\ref{prop:SIMmf}.

\paragraph{Local Strong Convexity.}
\begin{lemma}[Local Strong Convexity for SIM]\label{lemma:lsc}
If $d$ is large enough, then for any $t \leq T(\delta)$, we have for any $w$ with $|\xi_t(w) - w^*\on{sign}(\xi_t(w)^\top w^*)| \leq \frac{1}{5K}$,
\begin{align}
    \dpt(w) \preceq - \Omega_{K, \creg}\left(\sqrt{L(\rtmf)}\right).
\end{align}
\end{lemma}
\begin{proof}
For simplicity, let $w_t := \xi_t(w)$, let $\alpha := \alpha(w_t)$. Assume that $\alpha \neq 1$; if $\alpha = 1$, we can take the limit of the calculations below.

Recall that 
\begin{align}
    \dpt(w) &= \nabla_w \nu(w_t, \rtmf)
\end{align}
It is evident that $\nu(w_t, \rtmf)$ is in the direction $\tw : = \sqrt{1 - \alpha} w^* - \alpha \wpe$, where $\wpe = \frac{P^{\perp}_{w^*}w_t}{\|P^{\perp}_{w^*}w_t\|}$, and thus
\begin{align}
    \nu(w_t, \rtmf) = v(\alpha, t)\frac{\tw}{\sqrt{1 - \alpha^2}}.
\end{align}

We will consider the quadratic form $y^\top \dpt(w) y$ for $y \in \on{span}\tw$  and for $y \perp \on{span}(\xi_t(w), w^*)$. It suffices to show that for both such vectors we have $y^\top \dpt(w) y \leq - \Omega_{K, \creg}\left(\sqrt{L(\rtmf)}\right)\|y\|^2$.

Lets start with the first, letting $y = \tw$. We have
\begin{align}
    \dpt(w)y &= \frac{d \nu(w, \rtmf)}{d (y^\top w) }\\ 
    &= \frac{v(\alpha, t)}{\sqrt{1 - \alpha^2}} \frac{d \tw }{d (y^\top w_t)} + v(\alpha, t)\tw \frac{d (1 - \alpha^2)^{-1/2}}{d (y^\top w_t)} + \left(\frac{\tw}{\sqrt{1 - \alpha^2}}\right)\frac{d v(\alpha, t)}{d (y^\top w_t)}
\end{align}
Now 
\begin{align}
        \left(\frac{\tw}{\sqrt{1 - \alpha^2}}\right)\frac{d v(\alpha, t)}{d (y^\top w_t)} &= \left(\frac{\tw}{\sqrt{1 - \alpha^2}}\right)\frac{d v(\alpha, t)}{d \alpha} \frac{d \alpha}{d (y^\top w_t)} = \tw \frac{d v(\alpha, t)}{d \alpha} .
\end{align}
Next,
\begin{align}
    \frac{d (1 - \alpha^2)^{-1/2}}{d (y^\top w_t)} &= \frac{d (1 - \alpha^2)^{-1/2}}{d \alpha}  \frac{d \alpha}{d (y^\top w_t)}\\
    &= \frac{-\alpha}{(1 - \alpha^2)^{3/2}}  \frac{1}{\sqrt{1 - \alpha^2}}\\
    &= \frac{\alpha}{(1 - \alpha^2)}.
\end{align}
Finally, 
\begin{align}
\frac{d \tw }{d (y^\top w_t)} = 0
\end{align}
Thus in summary, putting these three terms together we have
\begin{align}
    y^\top \dpt(w)y = v(\alpha, t)\frac{\alpha}{(1 - \alpha^2)} + \frac{d v(\alpha, t)}{d \alpha}.
\end{align}
By Lemma~\ref{lemma:selfconcordance}, we have for $y = \tw$,
\begin{align}
    y^\top \dpt(w)y \leq -\Omega_{K, \creg}\left(\sqrt{L(\rtmf)}\right).
\end{align}
Now we consider $y \perp \tw, w_t$. We have
\begin{align}
    y^\top \frac{d \nu(w_t , \rtmf)}{d(y^\top w_t)} &= y^\top \tw \frac{d\left(\frac{v(\alpha, t)}{\sqrt{1 - \alpha^2}}\right)}{d y^\top w} + \frac{v(\alpha, t)}{\sqrt{1 - \alpha^2}} y^\top \frac{d \tw}{d (y^\top w_t)}\\
    &= 0 + \frac{v(\alpha, t)}{\sqrt{1 - \alpha^2}} y^\top \frac{d \tw}{d (y^\top w_t)}\\
    &= -\alpha \frac{v(\alpha, t)}{\sqrt{1 - \alpha^2}} y^\top \frac{d \wpe}{d (y^\top w_t)}\\
    &= -\alpha \frac{v(\alpha, t)}{\sqrt{1 - \alpha^2}} y^\top  \frac{y}{\sqrt{1 - \alpha^2}}\\
    &= -\frac{\alpha v(\alpha, t)}{1 - \alpha^2} \|y\|\\
    & \leq -\Omega_{K, \creg}\left(\sqrt{L(\rtmf)}\right).
\end{align}
Here the final inequality follows from Lemma~\ref{lemma:SIMv}.
\end{proof}

\paragraph{Proving Assumption~\ref{assm:growth} for SIM.}

First we will need the following lemma. Recall that $\mathbf{J}_h$ denotes the Jacobian of a multivariate function $h$.
\begin{lemma}\label{prop:simjst}
For any $w$ and $s \leq t \leq T(\delta)$, we have
\begin{align}
    \left\|\mathbf{J}_{\xi_{t, s}}(w_s)\right\| \leq O_K\left(\left(\frac{\alpha_t(w)}{\alpha_s(w)}\right)^{k^*-1}\right)\exp\left(O_K\left(\frac{1}{\delta}\right)\right).
\end{align}
\end{lemma}
\begin{proof}
It suffices to check that this holds for times where $\alpha_t(w) \leq \frac{1}{5K}$, because after that, by Lemma~\ref{lemma:selfconcordance}, $\dpt(w)$ is negative definite, and so $\left\|\mathbf{J}_{\xi_{t, s}}(\xi_s(w))\right\|$ can only decrease.

\begin{claim}\label{claim:lst}
In the setting of of the lemma, for any $w$ with $\alpha_t(w) \leq \frac{1}{5K}$, we have
\begin{align}
   \left\|\mathbf{J}_{\xi_{t, s}}(\xi_s(w))\right\| \leq O_K\left(\frac{d \alpha_{t, s}(z)}{d z}\bigg|_{z = \alpha_s(w)}\right) + 1.
\end{align}
\end{claim}
\begin{proof}
Let $w_s = \xi_s(w)$. Without loss of generality assume $w_s^\top w^* > 0$ such that $\alpha_s(w) = \xi_s(w)^\top w^*$. Let $w_{\perp} := \frac{P^{\perp}_{w^*}w_s}{\|P^{\perp}_{w^*}w_s\|}$. We have 
\begin{align}
    \xi_{t, s}(w_s) = \alpha_{t, s}(w_s) w^* + \sqrt{1 - \alpha_{t, s}(w_s)^2} \wpe.
\end{align}
Thus 
\begin{align}
    \mathbf{J}_{\xi_{t, s}}(w_s) = \mathbf{J}_{\alpha_{t, s}}(w_s) (w^*)^\top  + \frac{-\alpha_{t, s}(w_s)}{\sqrt{1 - \alpha_{t, s}(w_s)^2}} \mathbf{J}_{\alpha_{t, s}}(w_s) (\wpe)^\top  + \frac{\sqrt{1 - \alpha_{t, s}(w_s)^2}}{\sqrt{1 - \alpha_s(w_s)^2}} P^{\perp}_{w^*},
\end{align}
and so, since $\alpha_r(w)$ is increasing for $s \leq r \leq t$ if $\alpha_s(w) \geq \frac{1}{\sqrt{d}}$ (see Lemma~\ref{lemma:SIMv}) and $\alpha_t(w) \leq 1 - \frac{1}{5K}$, we have
\begin{align}
    \left\|\mathbf{J}_{\xi_{t, s}}(w_s)\right\| \leq O_K\left(\|\mathbf{J}_{\alpha_{t, s}}(w_s)\|\right) + 1.
\end{align}
\end{proof}
The conclusion now follows from combining this claim and Lemma~\ref{claim:lst2}.
\end{proof}
We are now ready to bound $\jmax$ and $\javg$.

\begin{lemma}\label{lemma:assmg}
For any $t \leq T(\delta)$, we have 
\begin{align}
    \jmax &\leq O_{K, \delta}(\sqrt{d}^{2(k^* - 1)}) \\
    \javg(\tau) & \leq O_{K, \tau, \delta}(1/T(\delta)).
\end{align}
\end{lemma}
\begin{proof}
By Lemma~\ref{prop:simjst}, for all $w$, we have
\begin{align}\label{eq:J}
    \|J^{\perp}_{t, s}(w)\| = O_{K, \delta}\left(\left(\frac{\alpha_t(w)}{\alpha_s(w)}\right)^{k^*-1}\right)
\end{align}
We bound this in two cases. Let $\iota = \delta^{6K^2}$. In the first case, if $\alpha_s(t) \geq \frac{\iota}{\sqrt{d}}$, then this is at most $O_{K, \delta}(\sqrt{d}^{k^* - 1})$ as desired. In the second case, if $\alpha_s(w) \leq \frac{\iota}{\sqrt{d}}$, then we can show that $\alpha_t(w)$ never exceeds $2 \alpha_s(w)$. Indeed, one can inductively show by Equation~\eqref{eq:vfull} that for $s \leq r \leq t$, we have $v(\alpha_r, r) \leq \iota^2 \sqrt{d}^{-(k^* - 1)}$. Since $T(\delta) \leq \frac{1}{\iota}\sqrt{d}^{k^* - 2}$, we have $\alpha_t(w) \leq 2 \alpha_s(w)$. Thus in either case, we have $\|J^{\perp}_{t, s}(w)\| =  O_{K, \delta}(\sqrt{d}^{k^* - 1})$. The desired bound on $\jmax$ is immediate.

To bound $\javg$ we have to be more careful, and we will use an additional averaging lemma (Lemma~\ref{lemma:Havg}) which allows us to show that when a set of neurons $w$ are well-dispersed on the sphere at some time $s$, then on average over $w$, $H^{\perp}(w, w')$ is small for any $w'$.

\begin{align}
    \mathbb{E}_{w \sim \rho_0} &\|J_{t, s}(w) H_s^{\perp}(w, w')v\|\mathbf{1}(\xi_t(w) \notin B_{\tau}) \\
    &= \mathbb{E}_{\alpha \sim \rsmf}\mathbb{E}_{w \sim \rho_0 | \alpha_s(w) = \alpha} \|J_{t, s}(w) H_s^{\perp}(w, w')v\|\mathbf{1}(\xi_t(w) \notin B_{\tau})\\
    &\leq \mathbb{E}_{\alpha \sim \rsmf} \mathbf{1}(\alpha_{t, s}(\alpha) \leq 1 - \tau) \sup_{w | \alpha_s(w) = \alpha} \|J_{t, s}(w)\|\mathbb{E}_{w \sim \rho_0 | \alpha_s(w) = \alpha} \|H_s^{\perp}(w, w')v\|\\
    &\leq \mathbb{E}_{\alpha \sim \rsmf} \mathbf{1}(\alpha_{t, s}(\alpha) \leq 1 - \tau) O_{K, \delta}\left(\frac{\alpha_t(w)}{\alpha_s(w)}\right)^{k^* - 1}\left(\frac{\alpha_t(w)}{\alpha_s(w)}\right)^{k^* - 1}\left(\alpha_s(w)^{k^* - 1} + \sqrt{d}^{-(k^* - 1)}\right)
 \end{align}
 Here the first inequality follows from the fact that the event $\xi_t(w) \notin B_{\tau}$ is equivalent to the even $\alpha_{t, s}(\alpha_s(w)) \leq 1 - \tau$. The second inequality is derived from \eqref{eq:J} and Lemma~\ref{lemma:Havg}.

 Now to bound this expectation, recall the two cases from earlier in the lemma: $\alpha_s(w) \leq \frac{\iota}{\sqrt{d}}$, and $\alpha_s(w) \geq \frac{\iota}{\sqrt{d}}$. Recall that in the first case, $\alpha_t(w) \leq 2 \alpha_s(w)$. Thus we have 
 \begin{align}
     &\mathbb{E}_{\alpha \sim \rsmf} \mathbf{1}(\alpha_{t, s}(\alpha) \leq 1 - \tau) O_{K, \delta}\left(\frac{\alpha_t(w)}{\alpha_s(w)}\right)^{k^* - 1}\left(\frac{\alpha_t(w)}{\alpha_s(w)}\right)^{k^* - 1}\left(\alpha_s(w)^{k^* - 1} + \sqrt{d}^{-(k^* - 1)}\right) \\
     &\qquad \leq O_{K, \delta}\left(\sqrt{d}^{(k^* - 1)}\right) + \mathbb{E}_{w \sim \rtmf} O_{K, \delta}\left(\alpha(w)^{k^* - 1}\right)\mathbf{1}(\alpha(w) \leq 1 - \tau).
 \end{align}
The additional implication in Proposition~\ref{prop:SIMmf} bounds this second term, yielding 
\begin{align}
    \mathbb{E}_{w \sim \rho_0} &\|J_{t, s}(w) H_s^{\perp}(w, w')v\|\mathbf{1}(\xi_t(w) \notin B_{\tau}) &\leq O_{K, \delta}\left(\sqrt{d}^{(k^* - 1)}\right) + \sqrt{d}^{-(k^* - 2)}O_{K, \delta}\left(\frac{1}{\tau^{O_K(1)}}\right)\\
    &= O_{K, \delta, \tau}\left(1/T(\delta)\right).
\end{align}
This proves the lemma.
\end{proof}

\begin{lemma}\label{lemma:Havg}
For any distribution $\mu$ over $w$, for and $w', v \in \mathbb{S}^{d-1}$, with $w_s := \xi_s(w)$, we have
\begin{align}
    \sup_{w', v}\mathbb{E}_{w \sim \mu}\|H^{\perp}_s(w, w')v\| &\lessapprox \sup_{\|u\| = 1} \sqrt{\mathbb{E}_{w \sim \mu} ({w_s}^\top u)^{2(k^*-1)}}\|v\|\\
    &\qquad + \sup_{\|u\| = 1} \sqrt{\mathbb{E}_{w \sim \mu} ({w_s}^\top u)^{2(k^*-2)}({w_s}^\top v)^{2}}.
\end{align}
In particular, if the distribution of $w_s$ is rotationally symmetric in some set of dimensions, and has norm at most $\alpha$ if the remaining dimensions, then 
\begin{align}
    \sup_{w', v}\mathbb{E}_{w \sim \mu}\|H^{\perp}_s(w, w')v\| &\leq O_K\left(\alpha^{k-1} + \sqrt{d}^{-(k^*-1)}\right).
\end{align}
\end{lemma}
\begin{proof}[Proof of Lemma~\ref{lemma:Havg}]
By Cauchy-Schwartz, 
\begin{align}
    \mathbb{E}_{w \sim \mu}\|H^{\perp}_s(w, w')v\| &\leq \sqrt{\mathbb{E}_{w \sim \mu}v(H^{\perp}_s(w, w'))^\top H^{\perp}_s(w, w')v}.
\end{align}
Let us expand $H^{\perp}_s(w, w')$. With $w_s := \xi_s(w)$ and $u := \xi_s(w')$, we have
\begin{align}
H_s^{\perp}(w, w') = \sum_{k = k^* - 1}^{K - 1} P^{\perp}_{{w_s}}\left(c(w, w')({w_s}^\top {u})^{k}I + c'(w, w')({w_s}^\top {u})^{k - 1}{u}{w_s}^\top \right)P^{\perp}_{{u}},
\end{align}
where $c(w, w'), c'(w, w') \leq \creg$. Thus we have
\begin{align}
&H_s^{\perp}(w, w')^\top  H_s^{\perp}(w, w')\\
&\qquad\preceq \sum_k 2\creg({w_s}^\top u)^{2k} P^{\perp}_{u} \\
&\qquad\qquad +  2\creg({w_s}^\top u)^{2(k - 1)} P^{\perp}_{u}{w_s}u^\top P^{\perp}_{{w_s}}u{w_s}^\top P^{\perp}_{u}\\
&\qquad\preceq 2\creg({w_s}^\top u)^{2(k^* - 1)}I \\
&\qquad\qquad + 2\creg({w_s}^\top u)^{2(k^* - 2)} {w_s}{w_s}^\top ,
\end{align}
and thus 
\begin{align}
    \mathbb{E}_{w \sim \mu}v(H^{\perp}_s(w, w'))^\top H^{\perp}_s(w, w')v &\leq 2\creg({w_s}^\top u)^{2(k^* - 1)}\|v\|^2 + 2\creg({w_s}^\top u)^{2(k^* - 1)} (v^\top {w_s})^2.
\end{align}
Taking a square root yields the desired result. The second statement follows observing that $\mathbb{E}_{w}[(u^\top w_s)^k] = O_k(\sqrt{d}^{-k})$ if $u$ is in the span of the rotationally invariant directions, because $u^\top w_s$ $\frac{1}{\sqrt{d}}$- subGaussian.
\end{proof}

\begin{proof}[Proof of Theorem~\ref{theorem:SIM}]
Fix a desired loss $\delta$, and let $T(\delta) = O_{K}(\sqrt{d}^{k^* - 2}\delta^{-(k^* - 1)})$ be as in Proposition~\ref{prop:SIMmf}, such that
\begin{align}\label{eq:MF}
        \mathbb{E}_x (f_{\rtmf}(x) - f^*(x))^2 &\leq \delta^2.
\end{align}

Let us check the conditions of Theorem~\ref{thm:main}. First, the regularity conditions in Assumption~\ref{assm:reg} trivially hold for $\creg = O_{\csim}(1)$ by our choice of Gaussian data and $\sigma$. 

By Lemma~\ref{lemma:assmg}, up to time $T(\delta)$, $(f^*, \rho_0, {\md_x})$ satisfies Assumption~\ref{assm:growth} with $\jmax = O_{K, \delta}(d^{2(k^* - 1)})$ and $\javg(\tau) = O_{K, \delta, \tau}(1/T(\delta))$. 

Observe that by Lemma~\ref{lemma:lsc}, $(f^*, \rho_0, {\md_x})$ is $(c, \tau)$ local strongly convex up to time $T(\delta)$ for $c = \Omega_{K, \creg}(1)$, $\tau = \frac{1}{5K}$. Further, since the problem has rotational symmetry in all directions orthogonal to the $w^*$ axis, the \em structured \em condition holds because by the smoothness of $\nabla_w \nu(w, \rtmf) P^{\perp}_w$ in $w$, and the fact that at $\nabla_{\xii(w_i)} \nu(\xii(w_i), \rtmf)P^{\perp}_{\xii(w_i)}$ (which approximates $\dpt(i)$ to $\creg \tau$  error) must be completely in the space orthogonal to $w^*$, and is rotationally symmetric in that space. Thus Assumption~\ref{def:lsc} holds.

Finally, the symmetry conditions in Assumption~\ref{assm:symmetries} trivially hold because the data is Guassian, and there is a reflection symmetry between $w^*$ and $-w^*$.

Now suppose $n \geq d^{11k^*} \geq \jmax^8 (T(\delta))^6 d^4$ and $m \geq d^{13k^*} \geq \jmax^{10} (T(\delta))^6 d^4$ such that 
\begin{align}
    \eps_n + \eps_m = \frac{\log(n)d^{3/2}}{\sqrt{n}} + \frac{\log(mT)\max(d^{1/2}\jmax, d^{3/2})}{\sqrt{m}} \leq \frac{1}{d\jmax^4 T^3}.
\end{align}
Thus for $d$ large enough, the condition on $\eps_n + \eps_m$ in Theorem~\ref{thm:main} holds.
Thus all the assumptions of Theorem~\ref{thm:main} hold, and the result guarantees that for $t \leq T(\delta)$, with high probability over the draw of the data and of the neural network initialization, we have with $\lambda = \min(\tau, \delta)$,
\begin{align}
    \mathbb{E}_x (f_{\rtmf}(x) - f_{\rtm}(x))^2 &\leq t\jmax(\eps_m + \eps_n)\exp\left(\frac{O(t \javg(\lambda))}{c\lambda - \Omega(\javg/\lambda)}\right)\\
    &\leq t d^{2(k^* - 1)} (\eps_n + \eps_m)O_{K, \delta}(1).
\end{align}
Combining this with Equation~\eqref{eq:MF}, we have that
\begin{align}
    \mathbb{E}_x (f^*(x) - f_{\rtm}(x))^2 &\leq 2\mathbb{E}_x (f_{\rtmf}(x) - f_{\rtm}(x))^2 + 2\mathbb{E}_x (f_{\rtmf}(x) - f^*(x))^2\\
    &\leq 2\delta^2 + 2t d^{2(k^* - 1)} (\eps_n + \eps_m)O_{K, \delta}(1) \leq 3\delta^2.
\end{align}
This proves the theorem.
\end{proof}

%% file: Appendices/apx_sims.tex
\newpage

\section{Full Details of Simulations}\label{apx:sims}

\begin{table}[!htb]
    {\small 
    \centering
     \scalebox{0.9}{
    \begin{tabular}{|c|c|c|c|c|c|c|}
    \hline
    Name & Target Function & Activation/Network Design & LSC? & Symmetric? & $\javg$ assm? & $\cstarr$ \\
    \hline
    \hefour & $\He_{4}(x^\top e_1)$ &  $\sigma = \He_{4}$ & Yes & Yes & Yes & 1\\
    \hline
    \hline
    \mantwo & $\mathbb{E}_{w \sim \mathbb{S}^{1}}\He_4(x^\top w)$ &  $\sigma = \He_{4}$ & No & Yes & Yes & $\approx 2^4$\\
    \hline
    \mis & $0.8\He_{4}(x^\top e_1) + 0.6\He_{6}(x^\top e_1)$ &  $\sigma = \He_{4} + \He_6$ & No & No  & Yes & $\approx d^4$\\
    \hline 
    \random{6}{6} & $\He_4$ link, $6$ random teachers in $\R^6$ &  $\sigma = \He_4$ & Yes & No & Yes? & $6$\\
    \hline 
    \hoponethree & $0.25x_1 + 0.75\mathsf{XOR}_4(x_{[4]})$ &  $\sigma = \text{SoftPlus}$, 2nd layer $\pm 8$ & Yes & No  & No & $\approx 2^8$\\
    \hline 
    \xorfour & $\mathsf{XOR}_4(x_{[4]})$ &  $\sigma = \text{SoftPlus}$, 2nd layer $\pm 8$ & Yes & No  & ? & $\approx 2^4$\\
    \hline
    \end{tabular}
    }
    \caption{\small List of problem settings we empirically investigated.
    } \label{table:sims}
    } 
\end{table}

\subsection{Experimental Design }\label{apx:sims:design}
 For each problem of interest, we simulated the training dynamics for several different widths $m \in [2^{12}, 2^{15}]$). We let $M$ be twice the largest value of $m$.  Crucially, \textit{we initialized all the networks to be a subnetwork of the largest width network.} Further, we used the same training data and training procedure (hyperparameters, batch size, batch selection, etc.) for all values of $m$ and $M$. 
We used the width $M$ network as a proxy for the mean-field limit, and studied how the neurons in the smaller networks differed in their trajectories from their counterparts in the largest network. All experiments are repeated for 3 times. 
Source code is available at \href{https://github.com/margalitglasgow/prop-chaos}{https://github.com/margalitglasgow/prop-chaos}.

\paragraph{Training procedure.}
We optimized the neural network as follows.
\begin{enumerate}
    \item We trained the models via mini-batch SGD with $n = 2^{16}$ total data points, and a batch size of $8196$. 
    \item We used a step size of $0.01$ (or occasionally smaller) for the problems with Gaussian data, and $0.05$ for the problems with Boolean data. This was mainly because the Gaussian data had higher moments, and hence the loss occasionally exploded under large step size. 
    \item For the Gaussian single-/multi-index problems we used a Hermite activation function and all-1 second-layer weights, whereas in the Boolean experiments we used the SoftPlus activation with temperature 16 (which is a smooth approximation of ReLU), and we fixed the 2nd layer weights to $\pm C_k$ with equal probability, where $C_k = 2^k / \sqrt{k}$ for the $k$-parity problem. 
\end{enumerate}

\paragraph{Analysis procedure.}
We made the following measurements along the training dynamics. 
\begin{enumerate}
    \item At each epoch, we computed the function error between the networks of width $m$ and $M$, using a randomly sampled dataset of size $n$.
    \item For each neuron $i$ in the width-$m$ network, we computed $\|\hdit\|$ as the norm of the difference between the neuron in the width-$m$ network and the corresponding neuron in the width-$M$ network. 
    \item We plot (a) the prediction risk curves, (b) the function error over time, and (c) $\mathbb{E}_i\|\hdit\|$ over time.
    In all the plots of the function and parameter error, we scaled up the error by the width $m$ for better visualization. 
\end{enumerate}


\newpage

\subsection{Additional Experimental Results}

\begin{figure}[!htb]
\begin{minipage}{0.32\linewidth}
\centering
{\includegraphics[width=0.96\linewidth]{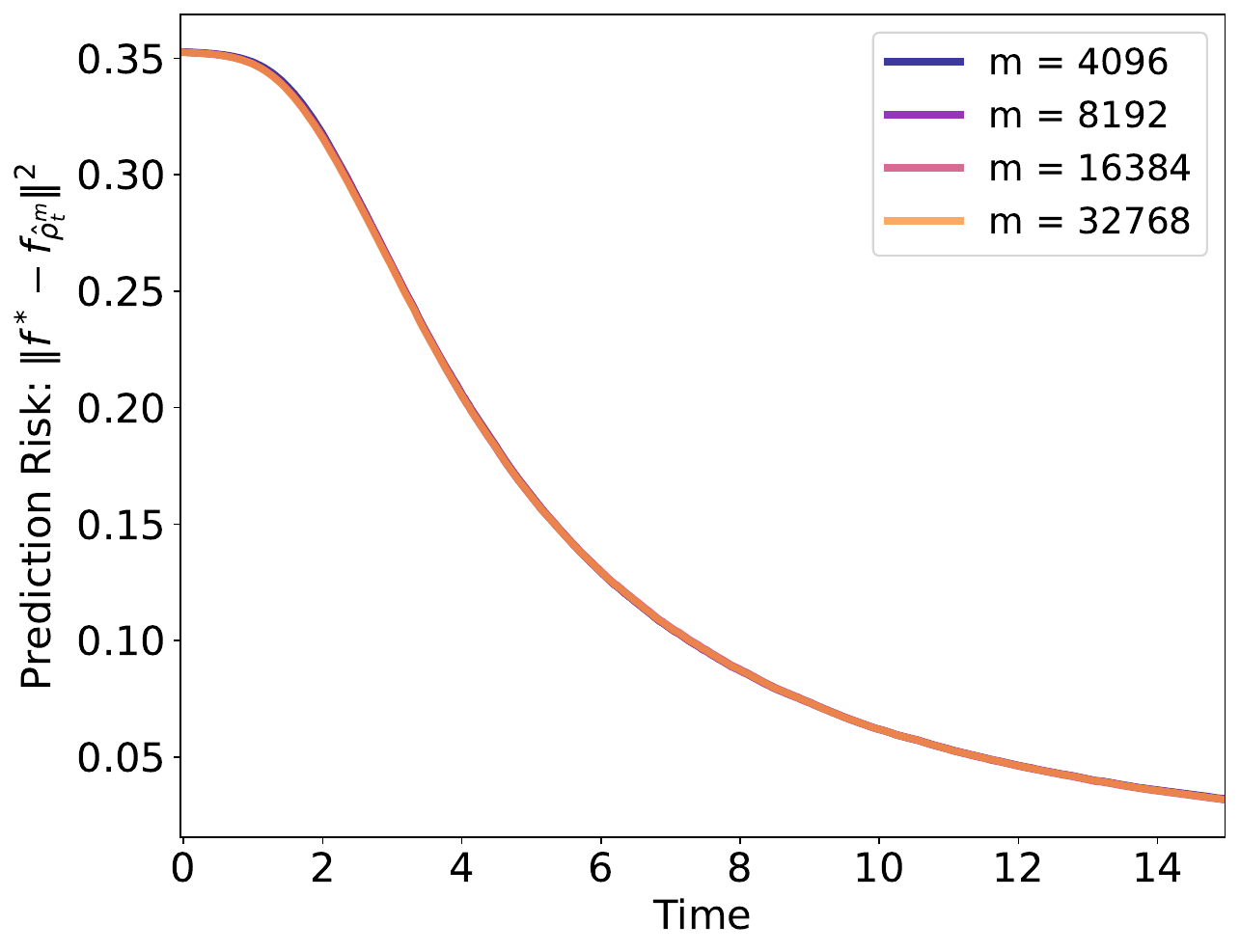}}  \\ \vspace{-1mm}
\small (a) prediction risk \\ $\|f_{\rtm} - f^*\|^2$.  
\end{minipage}%
\begin{minipage}{0.32\linewidth}
\centering
{\includegraphics[width=0.96\linewidth]{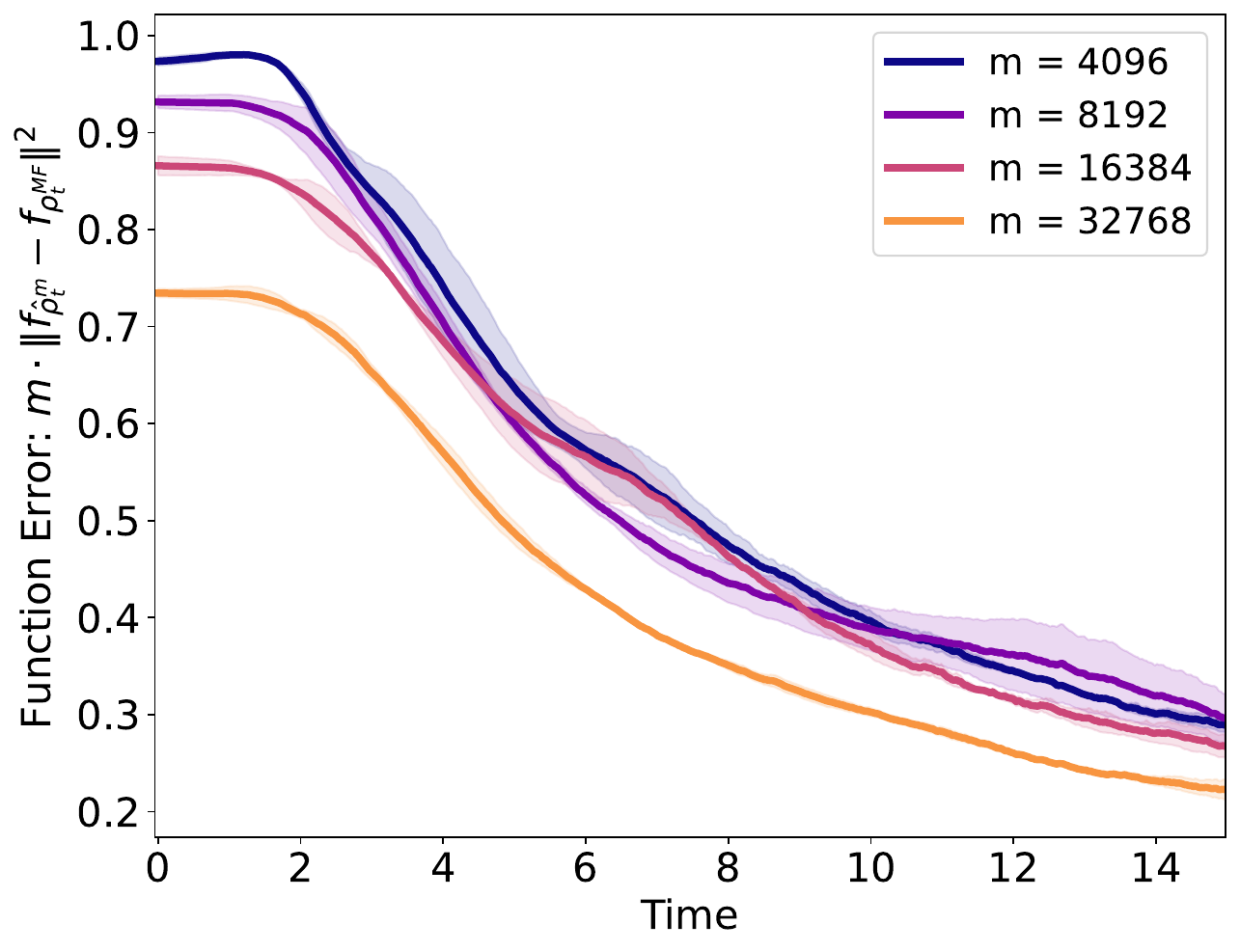}}  \\ \vspace{-1mm}
\small (b) scaled function error $m \|f_{\rtm} - f_{\rtmf}\|^2$.  
\end{minipage}%
\begin{minipage}{0.32\linewidth}
\centering
{\includegraphics[width=0.96\linewidth]{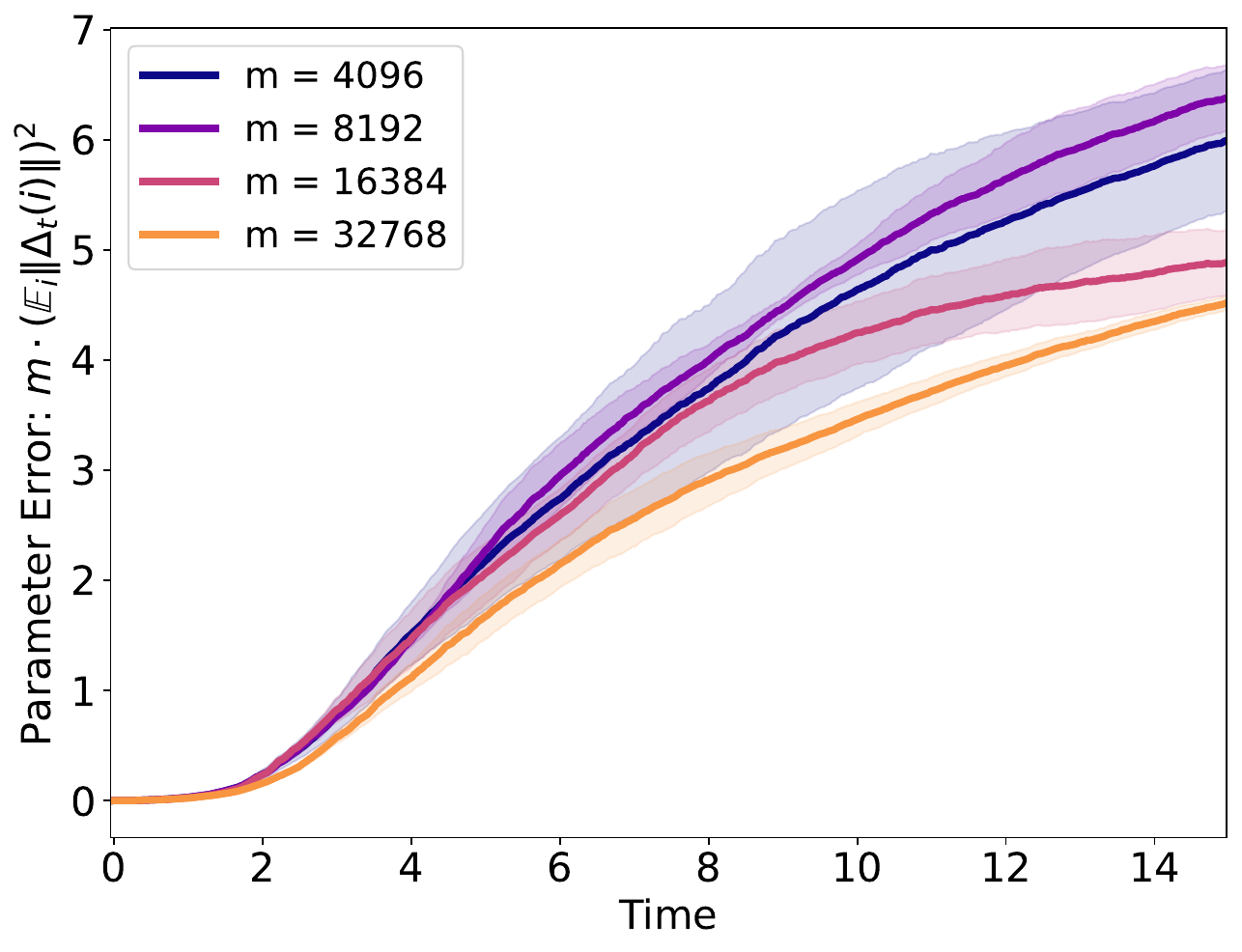}}  \\ \vspace{-1mm}
\small (c) scaled parameter coupling error $m (\mathbb{E}_i \|\hdit\|)^2$.  
\end{minipage}%
\vspace{-1mm}
\caption{Manifold (\mantwo) target function $f^*(x) = \mathbb{E}_{w \sim \mathbb{S}^{1}}\He_4(x^\top w), x\sim\mathcal{N}(0,I_d)$, and $\sigma = \text{He}_4$ ($\rho^*$ is distributed on a circle in 2 dimensions). We set $d=64$ and learning rate $\eta=0.01$. 
    }
    \label{fig:circle}
\end{figure}

\begin{figure}[!htb]
\begin{minipage}{0.32\linewidth}
\centering
{\includegraphics[width=0.96\linewidth]{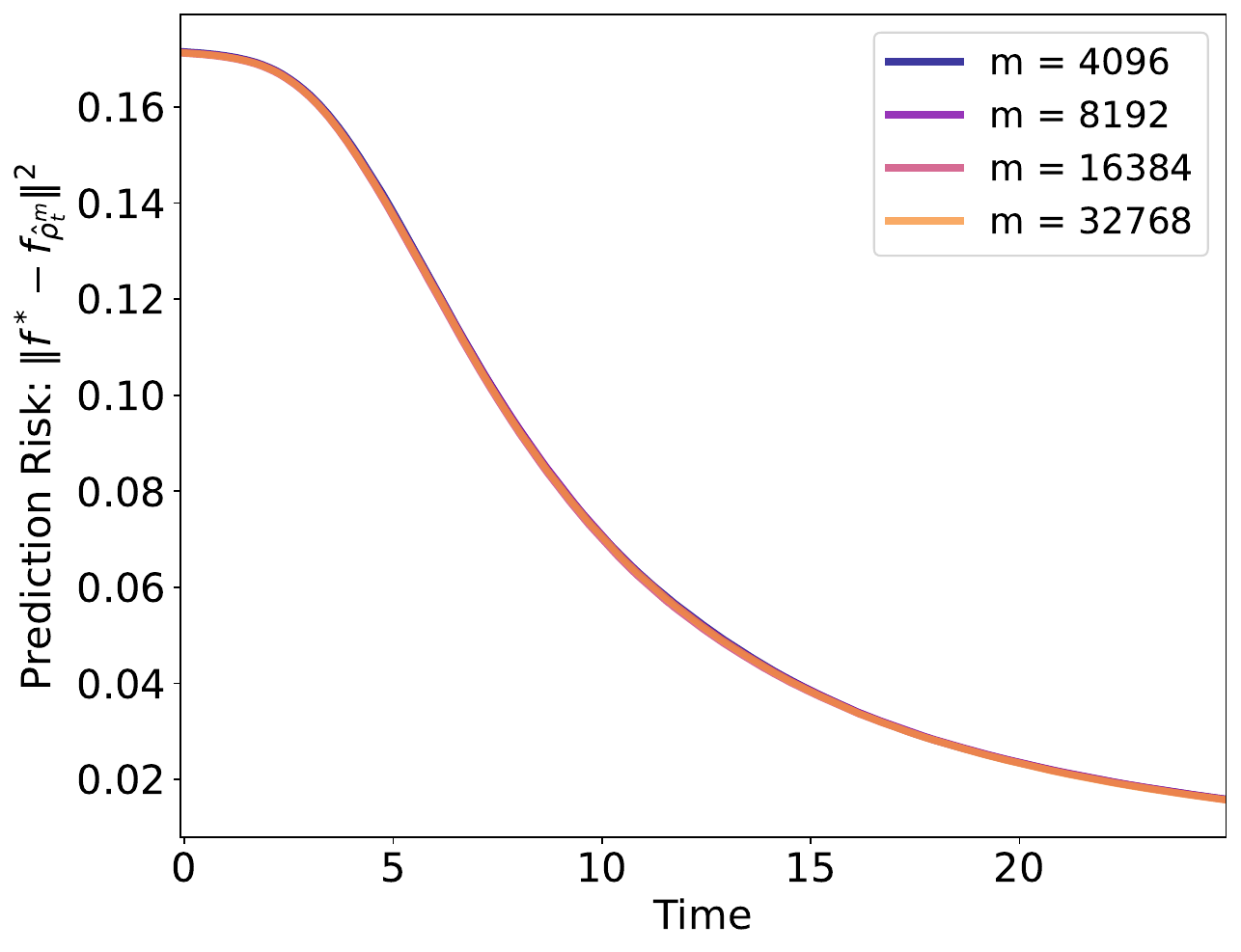}}  \\ \vspace{-1mm}
\small (a) prediction risk \\ $\|f_{\rtm} - f^*\|^2$.  
\end{minipage}%
\begin{minipage}{0.32\linewidth}
\centering
{\includegraphics[width=0.96\linewidth]{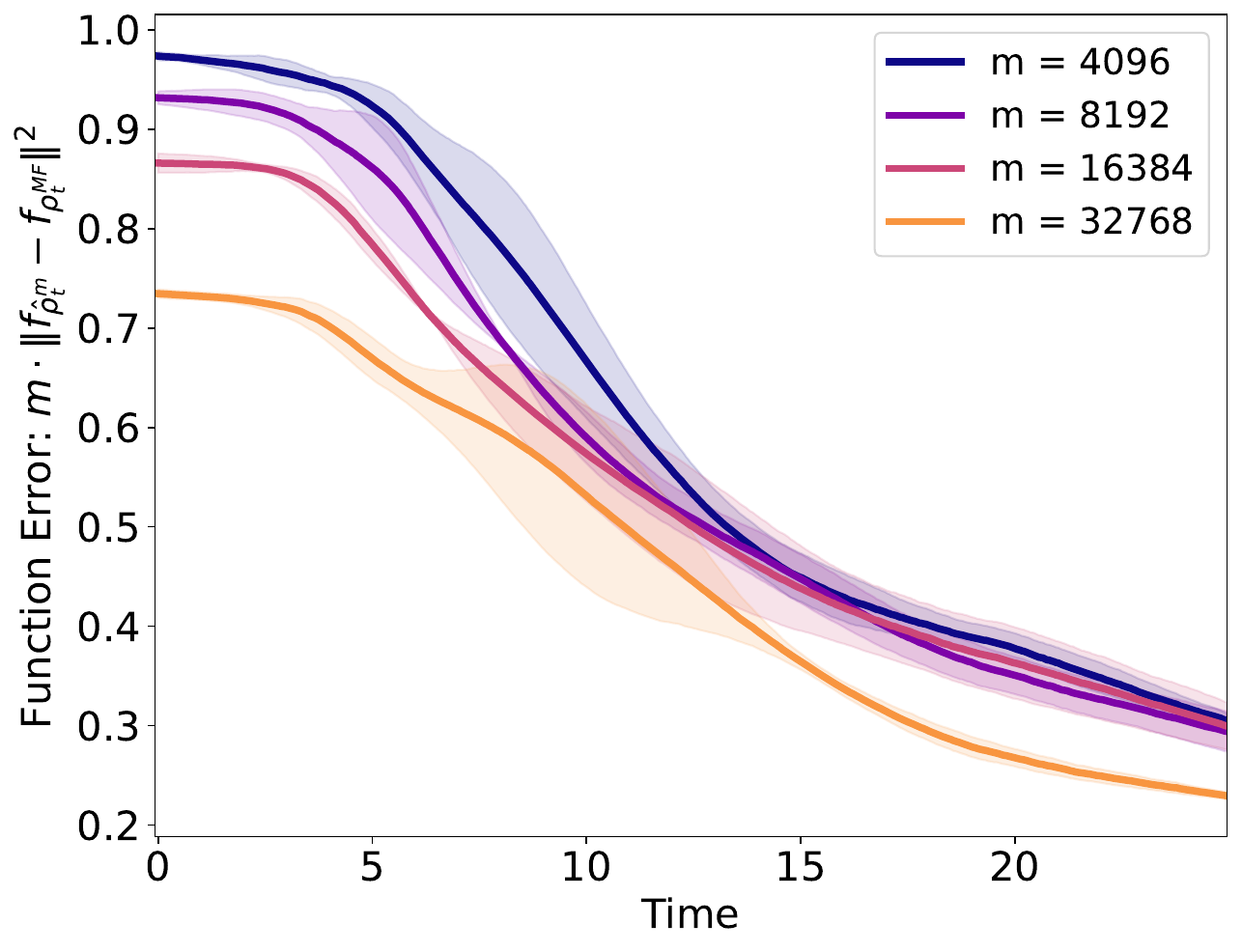}}  \\ \vspace{-1mm}
\small (b) scaled function error $m \|f_{\rtm} - f_{\rtmf}\|^2$.  
\end{minipage}%
\begin{minipage}{0.32\linewidth}
\centering
{\includegraphics[width=0.96\linewidth]{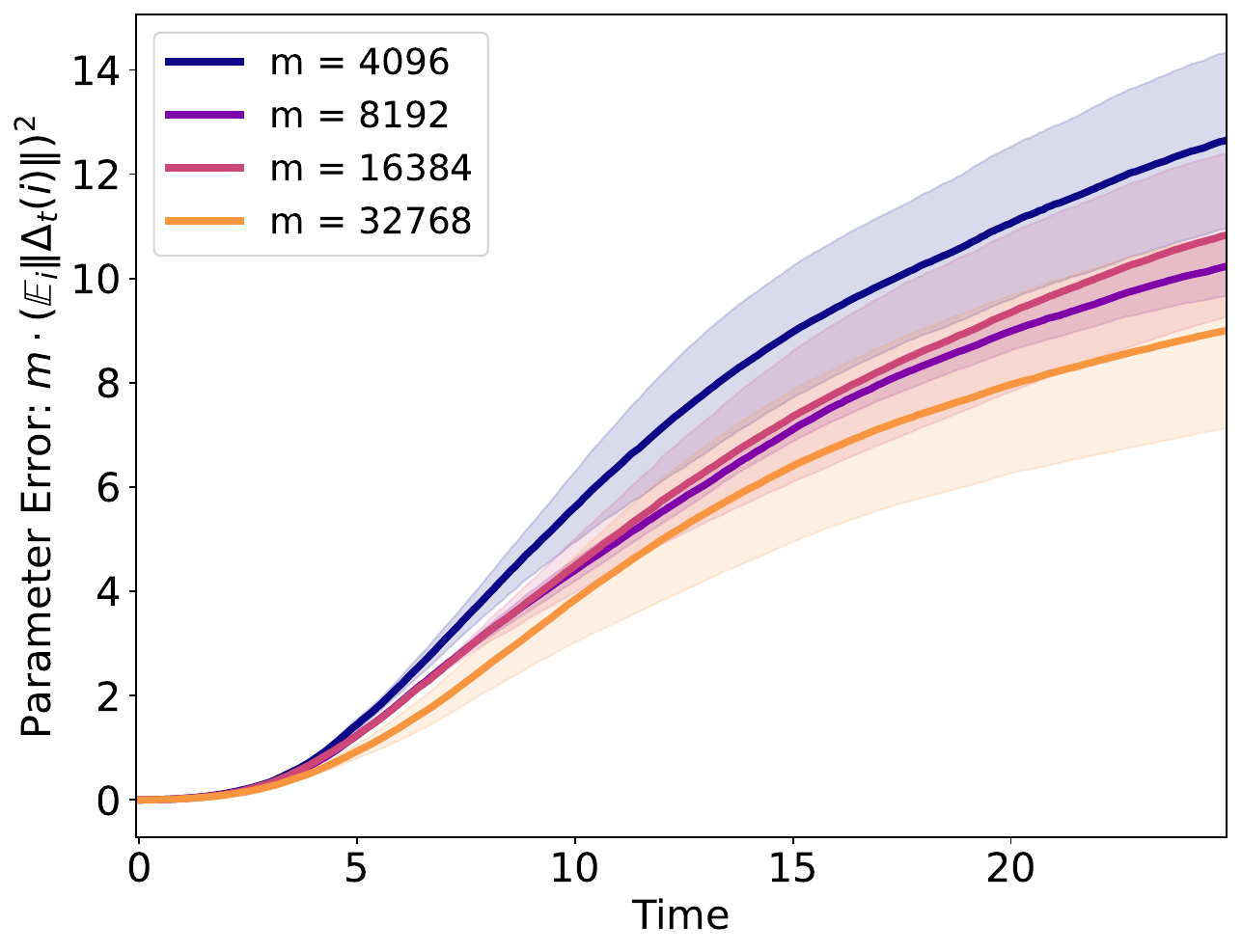}}  \\ \vspace{-1mm}
\small (c) scaled parameter coupling error $m (\mathbb{E}_i \|\hdit\|)^2$.  
\end{minipage}%
\vspace{-1mm}
\caption{Additive (\random{6}{6}) target function $f^*(x) = \frac{1}{6} \sum_{i=1}^6 \text{He}_4(x^\top w_i), x\sim\mathcal{N}(0,I_d)$, and $\sigma = \text{He}_4$. We set $d=64$ and learning rate $\eta=0.01$. 
    }
    \label{fig:additive}
\end{figure}

\begin{figure}[!htb]
\begin{minipage}{0.32\linewidth}
\centering
{\includegraphics[width=0.96\linewidth]{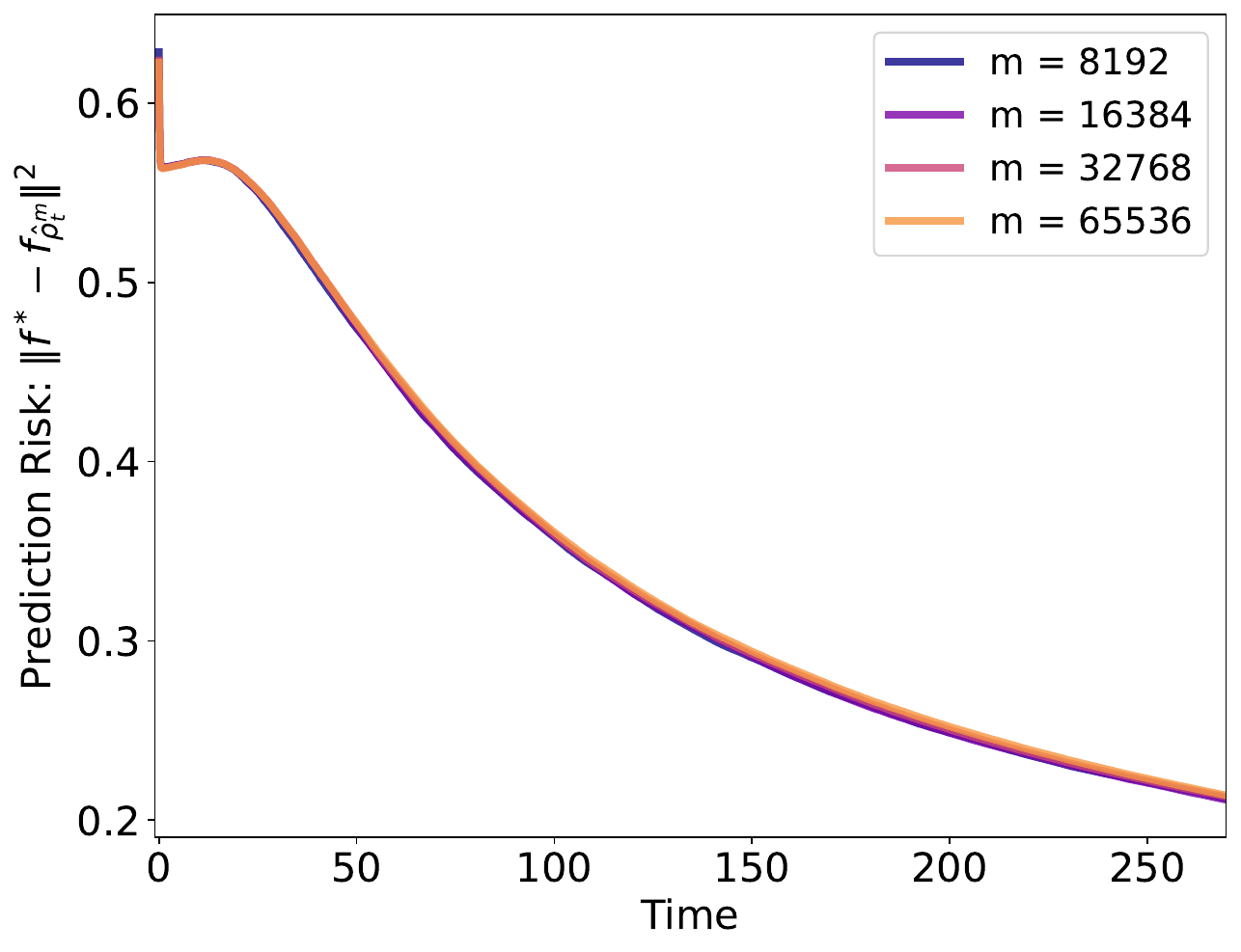}}  \\ \vspace{-1mm}
\small (a) prediction risk \\ $\|f_{\rtm} - f^*\|^2$.  
\end{minipage}%
\begin{minipage}{0.32\linewidth}
\centering
{\includegraphics[width=0.96\linewidth]{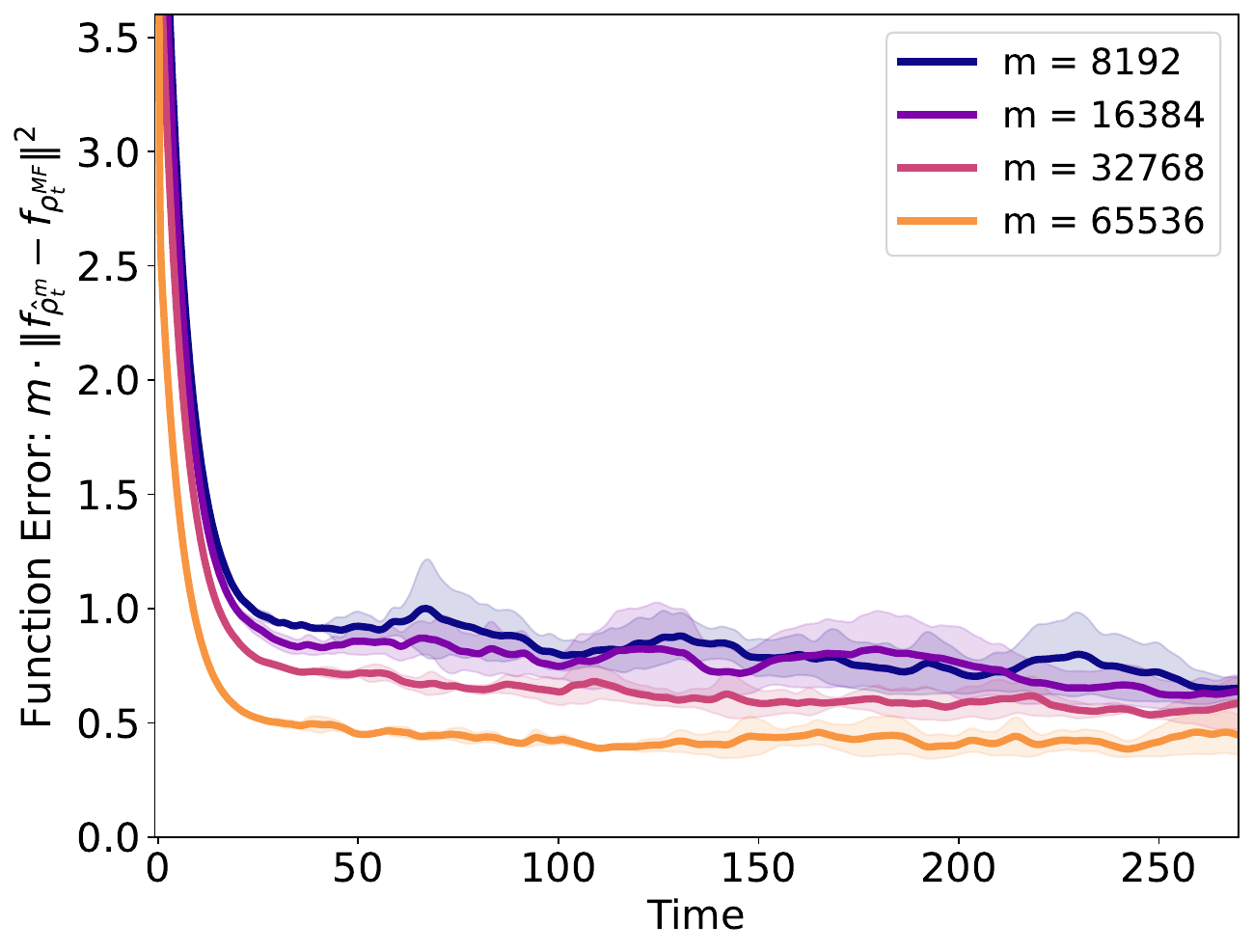}}  \\ \vspace{-1mm}
\small (b) scaled function error $m \|f_{\rtm} - f_{\rtmf}\|^2$.  
\end{minipage}%
\begin{minipage}{0.32\linewidth}
\centering
{\includegraphics[width=0.96\linewidth]{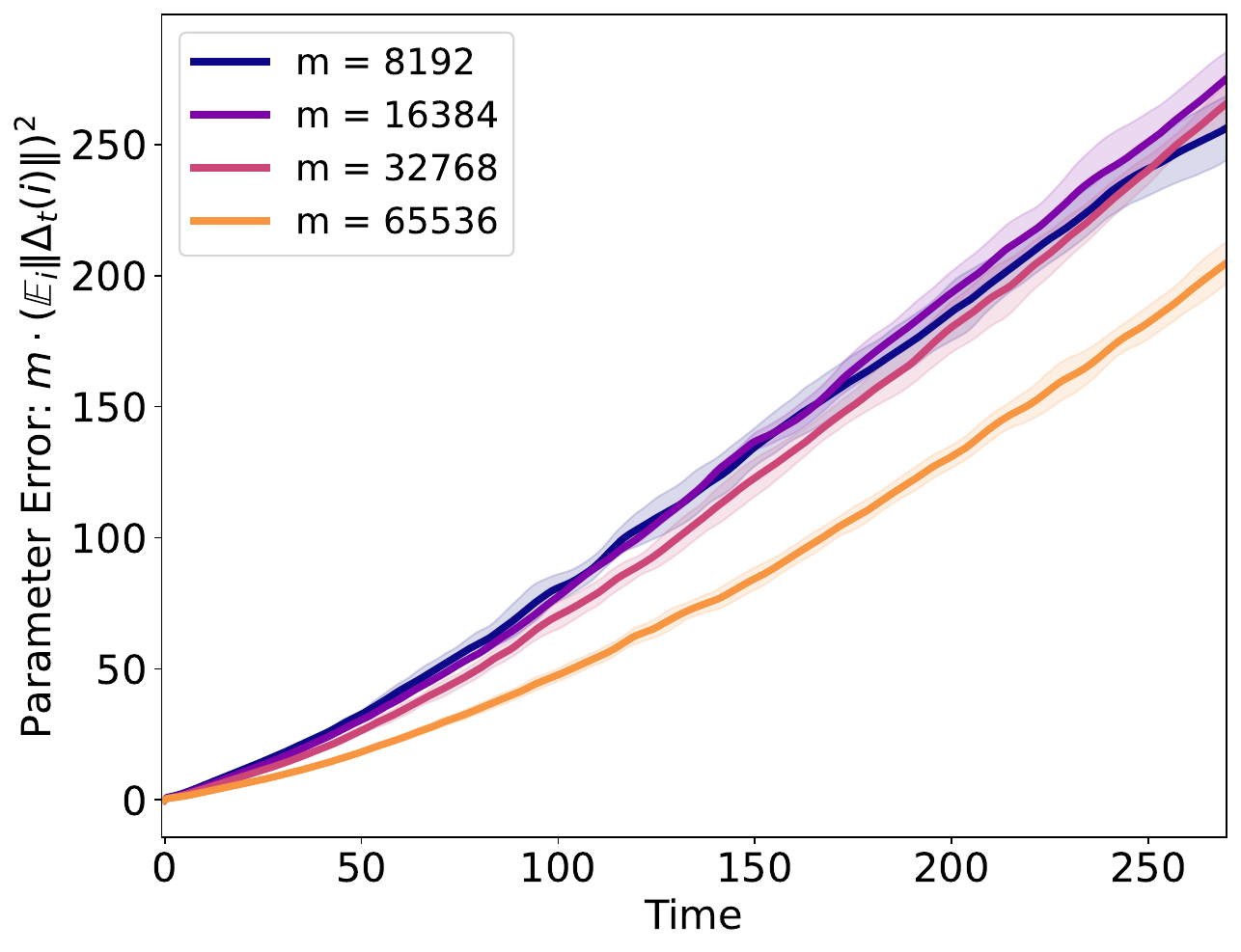}}  \\ \vspace{-1mm}
\small (c) scaled parameter coupling error $m (\mathbb{E}_i \|\hdit\|)^2$.  
\end{minipage}%
\vspace{-1mm}
\caption{\hoponethree~target function $f^*(x) = 0.25 [x]_1 + 0.75 \prod_{j\le 4} [x]_j, [x]_i\sim\text{Unif}\{1,-1\}$, and $\sigma = \text{SoftPlus}$ with temperature 16. We set $d=64$ and learning rate $\eta=0.025$. 
    }
    \label{fig:staircase}
\end{figure}

\vspace{-10mm}